\documentclass[preprint,12pt]{elsarticle}

\usepackage{amssymb}

\usepackage{amsthm}
\usepackage{thmtools}
\usepackage{bbm}
\usepackage{complexity}
\usepackage{xspace}
\usepackage{algorithm}
\usepackage{algpseudocode}
\usepackage{amsfonts}
\usepackage{amsmath}
\usepackage[numbers]{natbib}
\usepackage{xcolor}
\usepackage{graphicx}
\usepackage{caption}
\usepackage{subcaption}
\usepackage{amssymb}
\usepackage{hyperref}
\usepackage{booktabs}
\usepackage{thm-restate}
\usepackage{tikz}
\DeclareRobustCommand\full  {\tikz[baseline=-0.6ex]\draw[thick] (0,0)--(0.5,0);}

\DeclareRobustCommand\dashed{\tikz[baseline=-0.6ex]\draw[thick,dashed] (0,0)--(0.54,0);}

\usepackage{fontawesome}
\usepackage[flushleft]{threeparttable}

\newcommand{\code}[1]{\textsf{#1}}
\newcommand{\mln}{\Phi}
\newcommand{\sentence}{\Gamma}
\newcommand{\generalsentence}{{\widehat{\sentence}}}
\newcommand{\recursivesentence}{{\widetilde{\sentence}}}
\newcommand{\fotwoformula}{\psi}
\newcommand{\formula}{\alpha}

\newcommand{\weight}{w}
\newcommand{\vecweight}{\mathbf{\weight}}
\newcommand{\negweight}{\bar{w}}

\newcommand{\world}{\omega}
\newcommand{\wfomc}{WFOMC}

\newcommand{\symwfomc}{\ensuremath{\mathsf{WFOMC}}}

\newcommand{\fotwo}{\ensuremath{\mathbf{FO}^2}\xspace}
\newcommand{\ctwo}{\ensuremath{\mathbf{C}^2}\xspace}
\newcommand{\sctwo}{\ensuremath{\mathbf{SC}^2}}

\newcommand{\indicator}{\mathbbm{1}}

\newcommand{\domain}{\Delta}

\newcommand{\vecg}{\mathbf{g}}

\newcommand{\vecy}{\mathbf{y}}
\newcommand{\vecx}{\mathbf{x}}

\newcommand{\nat}{\mathbb{N}}
\newcommand{\real}{\mathbb{R}}

\newcommand{\pro}{\mathbb{P}}
\newcommand{\extformula}{\varphi}

\newcommand{\pair}[1]{\{#1\}}

\newcommand{\fomodels}[2]{\mathcal{M}_{#1, #2}}

\newcommand{\ufotwo}{\ensuremath{\mathbf{UFO}^2}\xspace}

\newcommand{\structure}{\mathcal{A}}
\newcommand{\typeweight}[1]{\langle \weight, \negweight\rangle(#1)}

\newcommand{\relaxcelltype}[2]{{#1}\downarrow{{#2}}}
\newcommand{\relaxcelltypeb}[2]{({#1}\downarrow{{#2}})}
\newcommand{\proj}[2]{\langle#1\rangle_{#2}}
\newcommand{\boldparagraph}[1]{\paragraph{{\normalfont \textbf{#1}}}}

\usepackage{forloop}
\newcounter{ct}

\newtheorem{definition}{Definition}
\newtheorem{example}{Example}
\newtheorem{remark}{Remark}

\newtheorem{theorem}{Theorem}
\newtheorem{lemma}{Lemma}

\algrenewcommand\textproc{\code}

\journal{Artificial Intelligence}

\begin{document}

\begin{frontmatter}



\title{Lifted Algorithms for Symmetric Weighted First-Order Model Sampling}


\author[inst1,inst2]{Yuanhong Wang\corref{cor1}}
\ead{lucienwang@buaa.edu.cn}
\address[inst1]{State Key Laboratory of Software Development Environment, Beihang University, Beijing, China}
\address[inst2]{Zhongfa Aviation Institute of Beihang University, Hangzhou, China China}


\author[inst1,inst2]{Juhua Pu}
            
\author[inst3,inst4]{Yuyi Wang}
\address[inst3]{CRRC Zhuzhou Insitute, Zhuzhou, China}
\address[inst4]{ETH Zurich, Zurich, Switzerland}


\author[inst5]{Ond\v{r}ej Ku\v{z}elka\corref{cor1}}
\ead{ondrej.kuzelka@fel.cvut.cz}
\address[inst5]{Czech Technical University in Prague, Prague, Czech Republic}


\cortext[cor1]{Corresponding author}

\begin{abstract}
Weighted model counting (WMC) is the task of computing the weighted sum of all satisfying assignments (i.e., models) of a propositional formula. 
Similarly, weighted model sampling (WMS) aims to randomly generate models with probability proportional to their respective weights. 
Both WMC and WMS are hard to solve exactly, falling under the \#\P-hard complexity class.
However, it is known that the counting problem may sometimes be tractable, if the propositional formula can be compactly represented and expressed in first-order logic. 
In such cases, model counting problems can be solved in time polynomial in the domain size, and are known as \textit{domain-liftable}.
The following question then arises: Is it also the case for WMS?
This paper addresses this question and answers it affirmatively.
Specifically, we prove the \textit{domain-liftability under sampling} for the two-variables fragment of first-order logic with counting quantifiers in this paper, by devising an efficient sampling algorithm for this fragment that runs in time polynomial in the domain size.
We then further show that this result continues to hold even in the presence of cardinality constraints.
To empirically validate our approach, we conduct experiments over various first-order formulas designed for the uniform generation of combinatorial structures and sampling in statistical-relational models. 
The results demonstrate that our algorithm outperforms a state-of-the-art WMS sampler by a substantial margin, confirming the theoretical results.

\end{abstract}

\begin{graphicalabstract}
\includegraphics{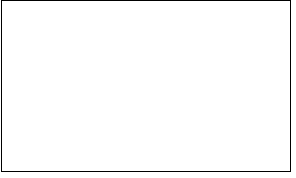}
\end{graphicalabstract}



\end{frontmatter}





\section{Introduction}

Given a propositional formula and a weight for each truth assignment, \textit{weighted model counting} (WMC) aims to compute the cumulative weight of all satisfying assignments (i.e., models) of the input formula.
A closely related problem to WMC is \textit{weighted model sampling} (WMS), which samples models of the propositional formula, where the probability of choosing a model is proportional to its weight.
These problems find applications in various domains, including machine learning~\cite{deraedtProbLogProbabilisticProlog2007,sang2005solving}, probabilistic inference~\cite{chavira2008probabilistic}, statistics, planning and combinatorics~\cite{domshlak2007probabilistic,van2021symmetric}, and constrained random verification~\cite{chakrabortyScalableNearlyUniform2013,chakrabortyParallelScalableUniform2015,soosTintedDetachedLazy2020}. 
Unfortunately, both WMC and WMS are computationally challenging, falling within the \#\P-hard complexity class~\cite{valiant1979SIAMJ.Comput.,dyer2002counting}.
Nevertheless, a glimmer of hope emerges when the input propositional formula can be naturally and compactly represented using first-order logic. 
In such cases, the WMC problem may become tractable by exploiting the symmetries present in the problem.

When the input formula $\sentence$ is expressed in first-order logic, the problem of calculating the Weighted Model Count (WMC) is known as Weighted First-Order Model Counting (\wfomc{}). 
In \wfomc{}, a domain $\domain$ containing a finite number of objects is provided as part of the input. 
The models of $\sentence$ over $\domain$ are usually defined as the models of the grounding of $\sentence$ over $\domain$.
The weights are assigned to atomic facts and their negations in the models, and the overall weight of a model is determined by the product of the weights of its constituent facts. 
The objective of a WFOMC problem is to compute the weighted summation over all models of the formula $\sentence$ over $\domain$.
Many real-world problems, including probabilistic inference and weight learning in various \textit{statistical-relational learning} (SRL) models, can be directly reduced to \wfomc{}~\cite{vandenbroeckLiftedProbabilisticInference2011,van2016lifted}.

Compared with WMC, an important advantage of \wfomc{} is the existence of \textit{domain-lifted} (or simply \textit{lifted}) algorithms for certain fragments of first-order logic, which are algorithms that have a polynomial runtime with respect to the domain size \cite{broeckCompletenessFirstOrderKnowledge2011,beameSymmetricWeightedFirstOrder2015,kuusistoWeightedModelCounting2018,kuzelkaWeightedFirstorderModel2021}. 
The intuition behind lifted algorithms is to lift the computation from the propositional (ground) level to the higher level of first-order logic, avoiding the need to ground the input formula~\cite{10.7551/mitpress/10548.001.0001}.
For instance, the model count of $\Phi = \forall x\exists y: R(x,y)$ over a domain of size $n$ can be easily computed as $(2^n-1)^n$ using the fact that all objects in the domain are indistinguishable.
However, if we ground $\Phi$ to a propositional logic formula, the indistinguishability of the objects will be lost.

The existence of lifted algorithms for \wfomc{} raises an intriguing question: Are there likewise lifted algorithms for WMS variations in first-order logic?
In this work, we answer this question by studying the problem of \textit{weighted first-order model sampling} (WFOMS).
This problem, which is a sampling counterpart of \wfomc{}, aims to generate a model of the input first-order sentence $\sentence$ based on a probability that is proportional to its weight. 
The WFOMS problem offers a natural reduction for a wide range of sampling problems without the necessity of converting them into WMS by grounding the input first-order sentence. 
Many problems relating to the generation of combinatorial structures can be readily formulated as WFOMS.
For instance, suppose we are interested in uniformly sampling labeled 2-colored graphs with $n$ nodes, it is equivalent to solving the WFOMS problem on the sentence:
\begin{equation}
  \begin{aligned}
      \forall x :\ &\neg E(x,x)\ \land \\
      \forall x \forall y :\ &E(x,y) \Rightarrow E(y,x)\ \land \\
      \forall x :\ &Red(x) \lor Black(x)\ \land \\
      \forall x :\ &\neg Red(x) \lor \neg Black(x) \land \\
      \forall x \forall y :\ &E(x,y) \Rightarrow \neg (Red(x) \land Red(y)) \land \neg (Black(x) \land Black(y))
  \end{aligned}
  \label{eq:2-color-graph}
\end{equation}
over a domain of size $n$, where $Red$ and $Black$ are two unary predicates representing the two colors of vertices.
Moreover, sampling problems in SRL can be also reduced to WFOMS as well.
An illustrative example is the sampling of possible worlds from a Markov logic network (MLN)~\cite{richardsonMarkovLogicNetworks2006}.
This can be easily transformed into a WFOMS problem using the same reduction technique employed for converting probabilistic inference to \wfomc{} (please see \ref{sub:expsettings}).

The answer to the question of whether there exist lifted algorithms for WFOMS depends on the input sentence of the problem.
Indeed, by the direct reduction from the spectrum membership problem~\cite{jaegerLowerComplexityBounds2015}, one can easily show that it is unlikely for every first-order sentence to have a lifted sampling algorithm\footnote{Recall that the spectrum, $\textsf{Spec}(\sentence)$, of a formula $\sentence$ is the set of numbers $n$ for which $\sentence$ has a model over a domain of size $n$. 
The spectrum membership problem, ``is $n\in \textsf{Spec}(\sentence)?$'', can be reduced to WFOMS by checking whether the sampling algorithm fails to sample a model.
Even if the sampling problem is guaranteed to have models, the spectrum decision problem can still be reduced to WFOMS.
Consider a WFOMS on $\sentence\lor A()$ over a domain of size $n$, where $A$ is a nullary predicate not in $\sentence$.
If the sampler can generate models with $\textsf{False}$ $A()$, we can confirm that $n\in \textsf{Spec}(\sentence)$, or if the sampler always generates models with $\textsf{True}$ $A()$, we can conclude that $n\notin \textsf{Spec}(\sentence)$ with high probability.}.
In the context of model counting, a first-order sentence or a class of sentences that allows for a lifted counting algorithm is referred to as \textit{domain-liftable} (or simply \textit{liftable}).
Similarly, for WFOMS, we use the terminology of \textit{domain-liftable under sampling} (or simply \textit{sampling liftable}) to describe the fragments that allow for a lifted sampling algorithm.
The objective of this paper is to investigate the sampling liftability for certain fragments of first-order logic.

The main contribution of this work is establishing the sampling liftability for the two-variables fragment of first-order logic with counting quantifiers $\exists_{=k}$, $\exists_{\le k}$ and $\exists_{\ge k}$ (see, e.g., \cite{gradelTwovariableLogicCounting1997}), which stand for \textit{exist exactly $k$}, \textit{exist at most $k$}, and \textit{exist at least $k$} respectively.
We note that this fragment, denoted by \ctwo{}, is expressive enough to encode various interesting sampling problems.
For instance, the sentence for sampling 2-colored graphs mentioned above contains only two variables and is in \ctwo{}.
The uniform generation of $k$-regular graphs, a problem that has been widely studied in the combinatorics community \cite{cooperSamplingRegularGraphs2007, gaoUniformGenerationRandom2015}, can be also solved by a lifted sampling algorithm for \ctwo{}.
This problem can be formulated as a WFOMS on the following \ctwo{} sentence:
\begin{equation}
  \label{eq:k-regular-graph}
  \forall x\forall y: (E(x,y)\Rightarrow E(y,x))\land \forall x: \neg E(x,x)\land \forall x\exists_{=k}y: E(x,y),
\end{equation}
where $\forall x\exists_{=k} y: E(x,y)$ expresses that every vertex $x$ has exactly $k$ incident edges. 

Our proof for the sampling liftability of \ctwo{} proceeds by progressively demonstrating the sampling liftability for the following fragments:
\begin{itemize}
    \item \ufotwo{}: The first-order fragment comprising of universally quantified sentences of the form $\forall x\forall y: \fotwoformula(x,y)$ with some quantifier-free formula $\fotwoformula(x,y)$;
    \item \fotwo{}: The two-variable fragment of first-order logic obtained by restricting the variable vocabulary to $\pair{x,y}$;
    \item \ctwo{}: The two-variable fragment of first-order logic with counting quantifiers $\exists_{=k}$, $\exists_{\le k}$ and $\exists_{\ge k}$.
\end{itemize}
Note that \ufotwo{} is a sub-fragment of \fotwo{}, which in turn is a sub-fragment of \ctwo{}.
The analysis of sampling liftability for smaller fragments can serve as a basis for that of the larger ones.
This is analogous to the case of \wfomc{}, where the liftability for \ctwo{} is established by first proving the liftability for \ufotwo{}~\cite{broeckCompletenessFirstOrderKnowledge2011} and \fotwo{}~\cite{vandenbroeck2014Proc.FourteenthInt.Conf.Princ.Knowl.Represent.Reason.}, and then extending the result to \ctwo{}~\cite{kuzelkaWeightedFirstorderModel2021}.


This manuscript is an expanded version of a conference paper~\cite{DBLP:conf/lics/WangP0K23} that appeared in LICS 2023. 
It includes all the technical preliminaries and provides comprehensive details of the proofs. 
Additionally, this journal version extends the positive result therein to the \ctwo{} fragment.
The rest of the paper is organized as follows.
Section~\ref{sec:preliminaries} presents the essential concepts used throughout this paper.
In Section~\ref{sec:wfoms}, we formally define the problem of weighted first-order model sampling.
The domain-liftability under sampling for \ufotwo{}, \fotwo{} and \ctwo{} is then established in Sections~\ref{sec:ufo2_liftable}, \ref{sec:fo2_liftable} and \ref{sec:c2_liftable}, respectively. 
In Section~\ref{sec:ccliftability}, we extend the result to the case with additional cardinality constraints, which in turn provides a more practically efficient sampling algorithm for a particular subfragment of \ctwo{}.
Section~\ref{sec:exp} presents the experimental results, which demonstrate the efficiency of our proposed algorithm.
Section~\ref{sec:related_work} discusses the relevant literature concerning the WFOMS problems. 
Finally, Section~\ref{sec:conclusion} provides the concluding remarks and outlines potential directions for future research.

\section{Preliminaries}
\label{sec:preliminaries}

In this section, we briefly review the main necessary technical concepts that we will use in the paper.

\subsection{First-order Logic}

We consider the function-free fragment of first-order logic.
An \emph{atom} of arity $k$ takes the form $P(x_1,\dots,x_k)$ where $P/k$ is from a vocabulary of predicates (also called relations), and $x_1,\dots,x_k$ are logical variables from a vocabulary of variables.
A \emph{literal} is an atom or its negation. 
A \emph{formula} is defined inductively as an atom, the negation of a single formula, or the conjunction or disjunction of two formulas.
A formula may optionally be surrounded by one or more quantifiers of the form $\forall x$ or $\exists x$, where $x$ is a logical variable. 
A logical variable in a formula is said to be \emph{free} if it is not bound by any quantifier.
A formula with no free variables is called a \emph{sentence}. 
The vocabulary of a formula $\formula$ is taken to be $\mathcal{P}_\formula$.

Given a predicate vocabulary $\mathcal{P}$, a $\mathcal{P}$-structure $\structure$ is a tuple $(\domain, \mathcal{I})$, where $\domain$, called \textit{domain}, is an arbitrary finite set, and $\mathcal{I}$ interprets each predicate in $\mathcal{P}$ over $\domain$.
In the context of this paper, the domain is usually predefined, and thus we can leave out the domain from a structure and instead treat a structure as either a set of ground literals in $\mathcal{I}$ or their conjunction.
Given a $\mathcal{P}$-structure $\structure$ and $\mathcal{P}'\subseteq\mathcal{P}$, we write $\proj{\structure}{\mathcal{P}'}$ for the $\mathcal{P}'$-reduct of $\structure$ that only contains the literals for predicates in $\mathcal{P}'$.
We follow the standard semantics of first-order logic for determining whether a structure is a model of a formula.
We denote the set of all models of a formula $\alpha$ over the domain $\domain$ by $\fomodels{\formula}{\domain}$. 
A set $L$ of ground literals is said to be \textit{valid} w.r.t. $(\formula, \domain)$, if there exists a model $\mu$ in $\fomodels{\alpha}{\domain}$ such that $L\subseteq \mu$.

\subsection{Weighted First-Order Model Counting}

The \textit{first-order model counting problem}~\cite{vandenbroeckLiftedProbabilisticInference2011} asks, when given a domain $\domain$ and a sentence $\sentence$, how many models $\sentence$ has over $\domain$.
The \textit{weighted first-order model counting (\wfomc{}) problem} adds a pair of weighting functions $(\weight, \negweight)$ to the input, that both map the set of all predicates in $\sentence$ to a set of weights: $\mathcal{P}_\sentence \to \real$.
\begin{definition}[Weight of literals]
  Given a pair of weighting function $(\weight, \negweight)$ and a set $L$ of literals, the weight of $L$ under $(\weight, \negweight)$ is defined as
  \begin{equation*}
    \typeweight{L} := \prod_{l\in L_T}\weight(\mathsf{pred}(l)) \cdot \prod_{l\in L_F}\negweight(\mathsf{pred}(l))
  \end{equation*}
  where $L_T$ (resp. $L_F$) denotes the set of true ground (resp. false) literals in $L$, and $\mathsf{pred}(l)$ maps a literal $l$ to its corresponding predicate name.
\end{definition}

\begin{definition}[Weighted first-order model counting]
    Let $(\weight, \negweight)$ be a weighting on a sentence $\sentence$. The \wfomc{} problem on $\sentence$ over a finite domain $\domain$ under $(\weight, \negweight)$ is to compute
    \begin{equation*}
        \symwfomc(\sentence, \domain, \weight, \negweight) := \sum_{\mu\in\fomodels{\sentence}{\domain}}\typeweight{\mu}.
    \end{equation*}
\end{definition}
Note that since these weightings are defined on the predicate level, all groundings of the same predicate get the same weights. For this reason, the notion of \wfomc{} defined here is also referred to as \textit{symmetric} \wfomc{}.

Given a sentence, or a class of sentences, prior research has mainly focused on its \textit{data complexity} for \wfomc{}---the complexity of computing $\symwfomc(\sentence, \domain, \weight, \negweight)$ when fixing the input sentence $\sentence$ and weighting $(\weight, \negweight)$, and treating the domain size $n$ as a unary input.
A sentence, or class of sentences, that exhibits polynomial-time data complexity is said to be \textit{domain-liftable} (or \textit{liftable}). 
Various fragments of first-order logic have been proven to be liftable, such as \ufotwo{}\cite{broeckCompletenessFirstOrderKnowledge2011}, \fotwo{}\cite{vandenbroeck2014Proc.FourteenthInt.Conf.Princ.Knowl.Represent.Reason.}, $\mathbf{S}^2\fotwo{}$~\cite{kazemiDomainRecursionLifted2017}, and $\mathbf{S}^2\mathbf{RU}$~\cite{kazemiDomainRecursionLifted2017}.

For technical purposes, when the domain is fixed, we allow the input sentence to contain some ground literals, e.g., $(\forall x\forall y: fr(x,y) \land sm(x) \Rightarrow sm(y))\land sm(e_1)\land \neg fr(e_1, e_3)$ over a fixed domain of $\{e_1, e_2, e_3\}$.
These ground literals are often called evidence, and \wfomc{} on such sentences is known as \wfomc{} with evidence~\cite{vandenbroeckConditioningFirstorderKnowledge2012,vandenbroeckComplexityApproximationBinary2013}.
In this paper, we also call this counting problem \textit{conditional} \wfomc{}.
An important result of conditional \wfomc{} is its maintenance of polynomial complexity when the ground literals are unary.
This result was provided in \cite{vandenbroeckConditioningFirstorderKnowledge2012} for a specific lifted counting algorithm, called first-order knowledge compilation (refer to \textbf{Positive Result} in Section~4).
We generalize the result to any lifted counting algorithm in Proposition~\ref{pro:liftable_unary_evidence} based on a similar technique.
Please find the details in~\ref{sec:wfomc-unary}.
\begin{restatable}{proposition}{liftableunaryevidence}
\label{pro:liftable_unary_evidence}
    Let $\sentence$ be a domain-liftable first-order sentence, and let $\domain$ be a domain.
    For any set $L$ of unary literals grounding on $\domain$, and any weighting functions $(\weight, \negweight)$, $\symwfomc(\sentence\land\bigwedge_{l\in L}l, \domain, \weight, \negweight)$ can be computed in time polynomial in both the domain size and the size of $L$.
\end{restatable}


\subsection{Types and Tables}

We define a \textit{1-literal} as an atomic predicate or its negation using only the variable $x$, and a \textit{2-literal} as an atomic predicate or its negation using both variables $x$ and $y$.
An atom like $R(x,x)$ or its negation is considered a 1-literal, even though $R$ is a binary relation.
A 2-literal is always of the form $R(x,y)$ and $R(y,x)$, or their respective negations.

Let $\mathcal{P}$ be a finite vocabulary.
A \textit{1-type} over $\mathcal{P}$ is a maximally consistent set\footnote{A set of literals is maximally consistent if it is consistent (does not contain both a literal and its negation) and cannot be extended to a larger consistent set.} of 1-literals formed by $\mathcal{P}$.
Denote the set of all 1-types over $\mathcal{P}$ as $U_\mathcal{P}$.
The size of $U_\mathcal{P}$ is finite and only depends on the size of $\mathcal{P}$.
We often view a 1-type $\tau$ as a conjunction of its elements, and write $\tau(x)$ when viewing $\tau$ as a formula.

A \textit{2-table} over $\mathcal{P}$ is a maximally consistent set of 2-literals formed by $\mathcal{P}$.
We often identify a 2-table $\pi$ with a conjunction of its elements and write it as a formula $\pi(x,y)$.
The total number of 2-tables over $\mathcal{P}$ also only depends on the size of $\mathcal{P}$.

Consider a structure $\structure$ defined over a predicate vocabulary $\mathcal{P}$. 
An element $e$ in the domain of $\structure$ \textit{realizes} the 1-type $\tau$ if $\structure\models\tau(e)$.
Every domain element in $\structure$ realizes exactly one 1-type over $\mathcal{P}$, which we call the \textit{1-type of the element}.
The \textit{2-table of an element tuple} $(a,b) \in \domain^2$ is the unique 2-table $\pi$ that $(a, b)$ satisfies in $\structure$: $\structure\models \pi(a,b)$.
A structure is fully characterized by the 1-types of its elements and the 2-tables of its element tuples.

\begin{example}
  Consider the vocabulary $\mathcal{P} = \{F/2, G/1\}$ and the structure over $\mathcal{P}$
  \begin{equation*}
      \{F(a,a), G(a), F(b,b), \neg G(b), F(a,b), \linebreak \neg F(b,a)\}
  \end{equation*}
  with the domain $\{a, b\}$.
  The 1-types of the elements $a$ and $b$ are $F(x,x)\land G(x)$ and $F(x,x)\land \neg G(x)$ respectively.
  The 2-table of the element tuples $(a,b)$ and $(b,a)$ are $F(x,y) \land \neg F(y,x)$ and $\neg F(x,y) \land F(y,x)$ respectively.
\end{example}

\subsection{Notations}

We will use $[n]$ to denote the set of $\{1, 2, \dots, n\}$.
The notation $\{x_i\}_{i\in[n]}$ represents the set of terms $\{x_1, x_2, \dots, x_n\}$, and $(x_i)_{i\in[n]}$ the vector of $(x_1, x_2, \linebreak \dots, x_n)$, which is also denoted by $\vecx$.
We use $\vecx^\vecy$ to denote the product over element-wise power of two vectors $\vecx$ and $\vecy$: $\vecx^\vecy = \prod_{i\in[n]}x_i^{y_i}$.
Using the vector notation, we write the multinomial coefficient $\binom{N}{x_1, x_2, \dots, x_n}$ as $\binom{N}{\vecx}$.

An important operation in our sampling algorithms is {\em partitioning}.
A partition of a set is defined as a grouping of its elements into disjoint subsets.
In this paper, all partitions under consideration are presumed to be \textit{order dependent}, and are represented by a vector of subsets $(S_i)_{i\in[m]}$.
Given a partition $\mathcal{S} = (S_i)_{i\in[m]}$ of a finite set, we refer to the vector of cardinalities $(|S_i|)_{i\in[m]}$ as the \textit{configuration} of  $\mathcal{S}$~\footnote{The notion of the partition configuration is analogous to many other concepts in the lifted inference literature, e.g., the histogram in lifted variable elimination algorithms~\cite{DBLP:conf/aaai/MilchZKHK08,DBLP:journals/jair/TaghipourFDB13}.}.
We adopt the term \textit{configuration space} to refer to the set of all partition configurations with a constant length over a given domain.
\begin{definition}[\textbf{Configuration space}]
  Given a non-negative integer $M$ and a positive integer $m$, we define the configuration space $\mathcal{T}_{M, m}$ as
  \begin{equation*}
    \mathcal{T}_{M,m} = \left\{(n_i)_{i\in[m]}\mid \sum_{i\in[m]} n_i = M, n_1, n_2, \dots, n_m\in\nat\right\}.
  \end{equation*}
\end{definition}
\noindent We remark that the size of $\mathcal{T}_{M, m}$ is given by $\binom{M + m - 1}{m - 1}$, which is polynomial in $M$ (while exponential in $m$).

\section{Weighted First-Order Model Sampling}
\label{sec:wfoms}

We are now ready to formally define the problem of \textit{weighted first-order model sampling} (WFOMS).

\begin{definition}[\textbf{Weighted first-order model sampling}]
    \label{def:swfos}
  Let $(\weight, \negweight)$ be a pair of weighting functions: $\mathcal{P}_\sentence\to\real_{\ge 0}$~\footnote{The non-negative weights assumption ensures that the sampling probability of a model is well-defined.}.
  The symmetric WFOMS problem on $\sentence$ over a domain $\domain$ under $(\weight, \negweight)$ is to generate a model $\mu\in\fomodels{\sentence}{\domain}$ according to the probability
  \begin{equation}
      \label{eq:WMS}
      \begin{aligned}
      \pro[\mu] = \frac{\typeweight{\mu}}{\symwfomc(\sentence, \domain, \weight, \negweight)}.
      \end{aligned}
  \end{equation}
\end{definition}

We call a probabilistic algorithm that realizes a solution to the WFOMS a \textit{weighted model sampler}.
Throughout this paper, when the context is clear, we simply refer to the weighted model sampler as a \emph{sampler}.
We use $G(\sentence, \domain, \weight, \negweight)$ to denote a sampler of the WFOMS problem on $\sentence$ over $\domain$ under $(\weight, \negweight)$.
We adapt the notion of data complexity of \wfomc{} to the sampling problem, and say a sampler is \textit{domain-lifted} (or simply \textit{lifted}) if the model sampling algorithm runs in time polynomial in the domain size $n$. 
We call a sentence, or class of sentences, that admits a domain-lifted sampler \textit{domain-liftable under sampling} (or simply \textit{sampling liftable}).

\begin{example}
\label{ex:planted_clique}
    A sampler of the sentence $\forall x\forall y: (E(x,y)\to E(y,x))\land \neg E(x,x)$ over a domain of size $n$ under the weighting $\weight(E) = 3, \negweight(E) = 1$ samples undirected graphs, where the probability of each edge is 
    \begin{equation*}
      \frac{w(E) \cdot w(E)}{w(E) \cdot w(E) + w(\neg E) \cdot w(\neg E)} = 0.9.
    \end{equation*}
    The sampler actually corresponds to an Erd\"{o}s-R\'{e}nyi graph $\mathcal{G}_{n,p=0.9}$~\cite{erdHos1960evolution}.
\end{example}

We define the probability of a sentence $\Phi$ conditional on another sentence $\sentence$ over a domain $\domain$ under $(\weight, \negweight)$ as
\begin{equation*}
  \pro[\Phi\mid\sentence; \domain, \weight, \negweight] := \frac{\symwfomc(\Phi\land\sentence, \domain, \weight, \negweight)}{\symwfomc(\sentence, \domain, \weight, \negweight)}.
\end{equation*}
With a slight abuse of notation, we also write the probability of a set $L$ of ground literals conditional on a sentence $\sentence$ over a domain $\domain$ under $(\weight, \negweight)$ in the same form:
\begin{equation*}
  \pro[L \mid \sentence; \domain, \weight, \negweight] := \pro\left[\bigwedge_{l\in L}l \mid \sentence; \domain, \weight, \negweight\right].
\end{equation*}
Then, the required sampling probability of a model $\mu$ in the WFOMS problem can be written as $\pro[\mu] = \pro[\mu\mid \sentence; \domain, \weight, \negweight]$.
When the context is clear, we omit $\domain$ and $(\weight, \negweight)$ in the conditional probability.


In this paper, we often convert one WFOMS problem into another, which is commonly referred to as a \textit{reduction}. 
The essential property of such reductions extensively used in this paper is their \textit{soundness}.
\begin{definition}[\textbf{Sound reduction}]
  \label{de:soundness}
  A reduction of the WFOMS problem of $(\sentence, \domain, \weight, \negweight)$ to $(\sentence', \domain', \weight', \negweight')$ is sound iff there exists a polynomial-time deterministic function $f$, such that $f$ is a
  mapping from $\fomodels{\sentence'}{\domain'}$ to $\fomodels{\sentence}{\domain}$, and for every model $\mu\in\fomodels{\sentence}{\domain}$, 
  \begin{equation}
    \pro[\mu\mid \sentence; \domain, \weight, \negweight] = \sum_{\substack{\mu'\in\fomodels{\sentence'}{\domain'}:\\ f(\mu') = \mu}}\pro[\mu'\mid\sentence'; \domain', \weight', \negweight'].
  \end{equation}
\end{definition}

Through a sound reduction, we can easily transform a sampler $G'$ of $(\sentence', \weight', \negweight', \domain')$ to a sampler $G$ of $(\sentence, \weight, \negweight, \domain)$ by 
\begin{equation*}
    G(\sentence, \domain, \weight, \negweight) := f(G'(\sentence', \domain', \weight', \negweight')).
\end{equation*}
The soundness is transitive, i.e., if the reductions from a WFOMS problem $\mathfrak{S}_1$ to $\mathfrak{S}_2$ and from $\mathfrak{S}_2$ to $\mathfrak{S}_3$ are both sound, the reduction from $\mathfrak{S}_1$ to $\mathfrak{S}_3$ is also sound.
For sound reductions in this paper, the most used mapping function is the $\mathcal{P}_\sentence$-reduct $f(\mu') = \proj{\mu'}{\mathcal{P}_\sentence}$, where $\mathcal{P}_\sentence\subseteq\mathcal{P}_{\sentence'}$.


\section{Universally Quantified \fotwo{} is Sampling Liftable}
\label{sec:ufo2_liftable}

In this section, we provide our first result of domain-liftability under sampling.
We consider the fragment of \ufotwo{}, a fragment containing all sentences in the form $\forall x\forall y: \fotwoformula(x,y)$, where $\fotwoformula(x,y)$ is a quantifier-free formula.


\begin{theorem}
    \label{th:ufo_liftable}
    The fragment of \ufotwo{} is domain-liftable under sampling.
\end{theorem}

We prove the sampling liftability of \ufotwo{} by constructing a lifted sampler. 
Suppose that we wish to sample models from some input \ufotwo{} sentence $\sentence=\forall x\forall y:\fotwoformula(x,y)$ over a domain $\domain = \{e_i\}_{i\in[n]}$ under weights $(\weight, \negweight)$.
Given a $\mathcal{P}_\sentence$-structure $\structure$ over $\domain$, we denote $\tau_i$ the 1-type of the domain element $e_i$ and $\pi_{i,j}$ the 2-table of the elements tuple $(e_i, e_j)$.
The structure $\structure$ is fully characterized by the ground 1-types $\tau_i(e_i)$ and 2-tables $\pi_{i,j}(e_i, e_j)$, and we can write the sampling probability of $\structure$ as
$$\pro[\structure\mid\sentence] = \pro\left[\bigwedge_{i\in[n]}\tau_i(e_i)\land\bigwedge_{i,j\in[n]: i<j}\pi_{i,j}(e_i, e_j)\mid\sentence\right].$$
By the definition of conditional probability, the sampling probability can be further decomposed as
\begin{equation}
    \begin{aligned}
      \pro[\structure\mid\sentence] = \underbrace{\pro\left[\bigwedge_{\substack{i,j\in[n]:\\ i<j}}\pi_{i,j}(e_i, e_j)\mid \sentence\land \bigwedge_{i\in[n]}\tau_i(e_i)\right]}_{\mathfrak{P}_2}\cdot \underbrace{\pro\left[\bigwedge_{i\in[n]}\tau_i(e_i)\mid \sentence\right]}_{\mathfrak{P}_1}.
    \end{aligned}
    \label{eq:decompose_sampling}
\end{equation}
This decomposition naturally gives rise to a two-phase sampling algorithm: 
\begin{enumerate}
  \item sample the 1-types $\tau_i$ of all elements according to the probability $\mathfrak{P}_1$, and
  \item sample the 2-tables $\pi_{i,j}$ of all elements tuples according to $\mathfrak{P}_2$.
\end{enumerate}
The weighted model sampler is shown in Algorithm~\ref{alg:weightedmodelsamplerufo2} with $\code{OneTypesSampler}$ and $\code{TwoTablesSamplerForUFO2}$ being the two sampling algorithms for the two phases, respectively.

\begin{algorithm}[!htb]
  \caption{$\code{WeightedModelSamplerForUFO2}(\sentence, \domain, \weight, \negweight)$} 
  \label{alg:weightedmodelsamplerufo2}
  \textbf{Input}: A \ufotwo{} sentence $\sentence$, a domain $\domain$ of size $n$ and a weighting $(\weight, \negweight)$ \\
  \textbf{Output}: A model $\mu$ of $\sentence$ sampled according to the probability $\pro[\mu\mid\sentence]$
  \begin{algorithmic}[1]
    \State $\{\tau_i(e_i)\}_{i\in[n]} \gets \code{OneTypesSampler}(\sentence, \domain, \weight, \negweight)$
    \State $\{\pi_{i,j}(e_i, e_j)\}_{i,j\in[n]} \gets \mathsf{TwoTablesSamplerForUFO2}( \sentence, \domain, \weight, \negweight, (\tau_i)_{i\in[n]})$
    \State \Return $\{\tau_i(e_i)\}_{i\in[n]} \cup \{\pi_{i,j}(e_i, e_j)\}_{i,j\in[n]}$
  \end{algorithmic}
\end{algorithm}

\begin{example}
    \label{ex:2-colored-graph}
    Consider the sampling of the 2-colored graphs mentioned in the Introduction.
    Let the number of vertices be $4$.
    It corresponds to a WFOMS problem on the sentence~\eqref{eq:2-color-graph} over a domain of size $4$ with the weights being $1$ for all predicates.
    The 1-types that can be sampled are $Red(x)\land \neg Black(x)$ and $\neg Red(x)\land Black(x)$, which correspond to coloring the vertices as red and black, respectively.
    The 2-tables that can be sampled are $E(x,y)\land E(y,x)$ and $\neg E(x,y)\land \neg E(y,x)$, indicating the presence or absence of an edge between two vertices.
    This sampling problem can be decomposed into two steps: coloring of the vertices and sampling of the edges.
\end{example}

\subsection{Sampling 1-types}
\label{sub:sampling-1-types}

Recall that the number of all possible 1-types over a predicates vocabulary $\mathcal{P}$ is $|U_\mathcal{P}|$.
Any assignment of 1-types to elements can be viewed as a $|U_{\mathcal{P}_\sentence}|$-length partition of the domain, where each disjoint subset precisely contains the elements realizing the corresponding 1-type.
Therefore, sampling 1-types is equivalent to randomly partitioning the domain $\domain$ into subsets of size $|U_{\mathcal{P}_\sentence}|$.
Furthermore, the symmetry property of the weighting function guarantees that any permutation of the elements in the domain does not impact the satisfaction or weight of the models. 
Thus, partitions with the same configuration are equally likely to be sampled.
This allows us to further split the sampling problem of 1-types into two stages: 1) sampling a partition configuration and 2) randomly partitioning the domain according to the sampled configuration.
The latter stage of random partitioning is trivial, and we will demonstrate that the first stage of sampling a partition configuration can be accomplished in time polynomial in the domain size.

\begin{example}
    Consider 1-types sampling for Example~\ref{ex:2-colored-graph}.
    Let $\weight(Red) = 2$ and all other weights be $1$.
    The sampling of 1-types is to color the vertices as red and black, and realized by a binary partition.
    The probability of a given partition or coloring scheme is proportional to the summed weight of all graphs with the coloring scheme.
    As shown in Table~\ref{tab:1-type-sampling}, the coloring schemes with the same partition configuration have the same weight $W$, and thus have the same sampling probability.
\end{example}

\begin{table}[!tb]
  \begin{threeparttable}
\centering
\caption{1-types sampling for 2-colored graphs}
\label{tab:1-type-sampling}
\begin{tabular}{@{}clll@{}}
\toprule
Configuration & Coloring scheme & $W$ & $|\mathcal{G}|$ \\ \midrule
$(0,4)$       &       \parbox[c]{.38in}{\includegraphics[width=.38in]{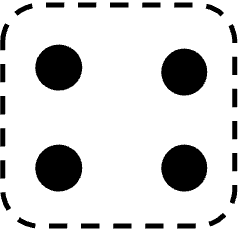}}            & 1 & $2^0 = 1$       \\
$(1,3)$       &   \parbox[c]{1.6in}{\includegraphics[width=1.6in]{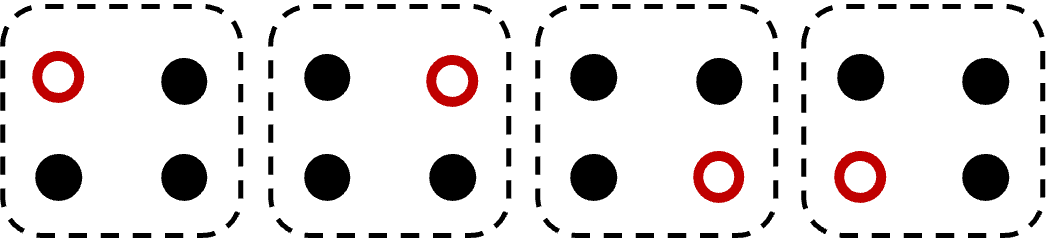}}                 & 2 & $2^3=8$         \\
$(2,2)$       &    \parbox[c]{2.4in}{\includegraphics[width=2.4in]{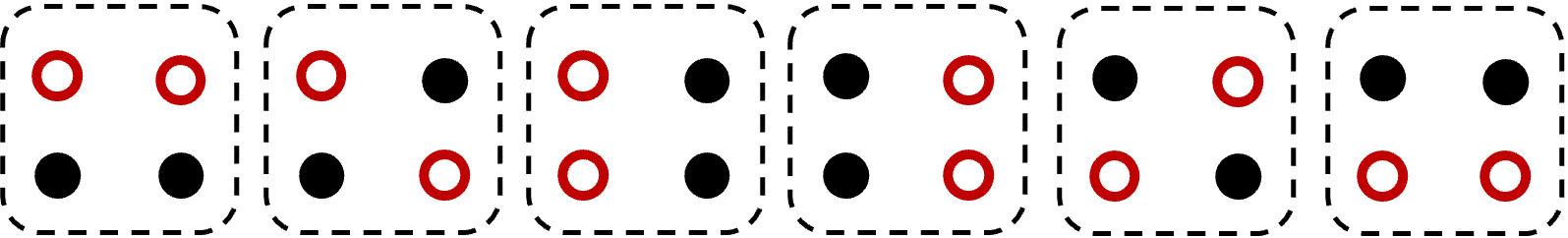}}             & 4   & $2^4=16$        \\
$(3,1)$       &   \parbox[c]{1.6in}{\includegraphics[width=1.6in]{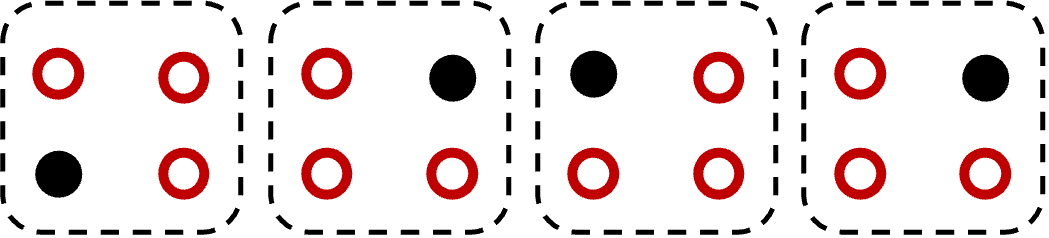}}                 &  8 & $2^3=8$         \\
$(4,0)$       &   \parbox[c]{.38in}{\includegraphics[width=.38in]{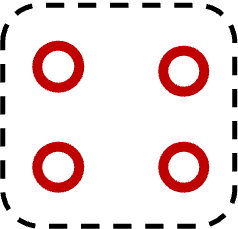}}                 & 16 & $2^0=1$         \\ \bottomrule
\end{tabular}
\begin{tablenotes}
  \tiny
  \item[*] filled: black vertices; empty: red vertices
\end{tablenotes}
\end{threeparttable}
\end{table}

We first observe that the overall number of partition configurations is given by the size of the configuration space $\mathcal{T}_{n, |U_{\mathcal{P}_\sentence}|}$, which is polynomial in the domain size.
This property allows us to use the \emph{enumerative sampling} approach to generate the partition configuration randomly. 
In enumerative sampling, we enumerate all possible values of the random variable, and then pick one of them according to the probability distribution.
To this end, we need to compute the probability of each possible partition configuration.
For any partition configuration $\mathbf{m} = (m_1, m_2, \dots, m_{|U_{\mathcal{P}_\sentence}|})$, there are a total of $\binom{n}{\mathbf{m}}$ partitions.
The symmetry of the weighting function guarantees that all of these partitions have the same sampling probability.
The probability $\mathfrak{P}_1^\mathbf{m}$ can be written in the form of $\mathfrak{P}_1$, shown below, due to the equivalence between partitions and 1-types assignments of domain elements:
\begin{equation}
  \mathfrak{P}_1^\mathbf{m} = \frac{\symwfomc(\sentence\land\bigwedge_{i=1}^n \tau_i(e_i),\domain,\weight, \negweight)}{\symwfomc(\sentence, \domain, \weight, \negweight)},
  \label{eq:1-type-probability}
\end{equation}
where the 1-types $\tau_i$ correspond to a specific partition with the size $\mathbf{m}$.
Thus, the sampling probability of the partition size $\mathbf{m}$ can be calculated as $\mathfrak{P}_1^\mathbf{m}\cdot \binom{n}{\mathbf{m}}$.
Finally, let us demonstrate that the computation of the probability $\mathfrak{P}_1^\mathbf{m}$ is polynomial-time in the domain size.
In \eqref{eq:1-type-probability}, the computation of the denominator \wfomc{} on $\sentence$ is in polynomial-time in the domain size due to the liftability of $\sentence$ for \wfomc{}~\cite{broeckCompletenessFirstOrderKnowledge2011}.
The numerator can be viewed as a \wfomc{} problem of $\sentence$ conditioned on the unary facts in all ground 1-types $\tau_i(e_i)$, whose size is clearly polynomial in the size of the domain.
By Proposition~\ref{pro:liftable_unary_evidence} and the liftability of $\sentence$, it follows that the complexity of computing $\mathfrak{P}_1^\mathbf{m}$ is polynomial in the domain size.
The algorithm for sampling 1-types is presented in Algorithm~\ref{alg:1typesampler}.

\begin{algorithm}[!htb]
  \caption{$\code{OneTypesSampler}(\sentence, \domain, \weight, \negweight)$} 
  \label{alg:1typesampler}
  \textbf{Input}: A \ufotwo{} sentence $\sentence$, a domain $\domain=\{e_i\}_{i\in[n]}$ and a weighting $(\weight, \negweight)$ \\
  \textbf{Output}: 1-types $\{\tau_i(e_i)\}_{i\in[n]}$
  \begin{algorithmic}[1]
    \For{$\mathbf{m}\in \mathcal{T}_{n, |U_{\mathcal{P}_\sentence}|}$}
      \State $p(\mathbf{m})\gets \mathfrak{P}_1^\mathbf{m}\cdot \binom{n}{\mathbf{m}}$ \Comment{Compute the probability of $\mathbf{m}$}
    \EndFor
    \State Sample $\mathbf{m}^*$ from $\mathcal{T}_{n, |U_{\mathcal{P}_\sentence}|}$ according to the probability distribution $p$
    \State Randomly partition $\domain$ into $|U_{\mathcal{P}_\sentence}|$ subsets of size $m^*_i$ for $i\in[|U_{\mathcal{P}_\sentence}|]$
    \For{$i\in[n]$}
      \State $\tau_i\gets$ the 1-type of the element $e_i$ in the partition
    \EndFor
    \State \Return $\{\tau_i(e_i)\}_{i\in[n]}$
  \end{algorithmic}
\end{algorithm}

\subsection{Sampling 2-tables}
\label{sub:sampling-2-tables}

For sampling $\pi_{i,j}$ according to the probability
\begin{equation*}
  \mathfrak{P}_2 = \pro\left[\bigwedge_{\substack{i,j\in[n]:\\ i<j}}\pi_{i,j}(e_i, e_j)\mid \sentence\land \bigwedge_{i\in[n]}\tau_i(e_i)\right],
\end{equation*}
we first ground out $\sentence=\forall x\forall y: \fotwoformula(x,y)$ over the domain $\domain$: 
\begin{equation*}
    \bigwedge_{i,j\in[n]: i<j}\fotwoformula(e_i, e_j)\land \fotwoformula(e_j, e_i),
\end{equation*}
where $\fotwoformula(x,y)$ is the quantifier-free formula in $\sentence$.
Let $\fotwoformula_{i,j}(x,y)$ be the simplified formula of $\fotwoformula(x,y)\land \fotwoformula(y,x)$ obtained by replacing the ground 1-literals with their truth value given by the 1-types $\tau_i$ and $\tau_j$.
The probability $\mathfrak{P}_2$ can be written as
\begin{equation}
  \label{eq:decompose_2_tables_sampling}
    \pro[\bigwedge_{i,j\in[n]:i<j}\pi_{i,j}(e_i, e_j)\mid \bigwedge_{i,j\in[n]: i<j}\fotwoformula_{i,j}(e_i, e_j)].
\end{equation}
In this conditional probability, all ground 2-tables $\pi_{i,j}(e_i,e_j)$ are independent in the sense that they do not share any ground literals.
The independence also holds for the ground formulas $\fotwoformula_{i,j}(e_i, e_j)$, because all ground 1-literals were replaced by their truth values.
It follows that \eqref{eq:decompose_2_tables_sampling} can be factorized into
\begin{equation*}
  \prod_{i,j\in[n]: i<j} \pro[\pi_{i,j}(e_i,e_j)\mid\fotwoformula_{i,j}(e_i, e_j)].
\end{equation*}
Hence, sampling the 2-tables $\pi_{i,j}$ can be solved separately.
The outline of the algorithm is presented in Algorithm~\ref{alg:2tablesampler}.
The overall computational complexity is clearly polynomial in the domain size.

\begin{algorithm}[!htb]
  \caption{$\code{TwoTablesSamplerForUFO2}(\sentence, \domain, \weight, \negweight, (\tau_i)_{i\in[n]})$} 
  \label{alg:2tablesampler}
  \textbf{Input}: A \ufotwo{} sentence $\sentence$, a domain $\domain=\{e_i\}_{i\in[n]}$, a weighting $(\weight, \negweight)$, and 1-types $(\tau_i)_{i\in[n]}$ \\
  \textbf{Output}: 2-tables $\{\pi_{i,j}(e_i, e_j)\}_{i,j\in[n]}$
  \begin{algorithmic}[1]
    \State $\structure \gets \emptyset$
    \State $\Pi\gets$ all 2-tables over $\mathcal{P}_\sentence$ \Comment{Note that the size of $\Pi$ only depends on $\mathcal{P}_\sentence$}
    \For{$i\in[n]$}
      \For{$j\in[n]$}
        \If{$i < j$}
          \State Simplify $\fotwoformula(x,y)\land\fotwoformula(y,x)$ to $\fotwoformula_{i,j}(x,y)$ by replacing the 1-literals with their truth values given by $\tau_i$ and $\tau_j$
          \For{$\pi\in \Pi$}
            \State $p(\pi)\gets \pro[\pi(e_i, e_j)\mid \fotwoformula_{i,j}(e_i, e_j)]$
          \EndFor
          \State Sample $\pi_{i,j}$ from $\Pi$ according to the probability distribution $p$
          \State $\structure\gets \structure\cup\{\pi_{i,j}(e_i, e_j)\}$
        \EndIf
      \EndFor
    \EndFor
    \State \Return $\structure$
  \end{algorithmic}
\end{algorithm}

\begin{example}
    Consider the 2-tables sampling in the WFOMS problem presented in Example~\ref{ex:2-colored-graph}.
    It is clear that the sampling of an edge is fully determined by the colors of its endpoints.
    There are only two cases: if the endpoints share the same color, no edge can exist between them, otherwise the edge is sampled with a probability of $1/2$ (recall that $\weight(E) = \negweight(E) = 1$).
\end{example}

\begin{proof}[Proof of Theorem~\ref{th:ufo_liftable}]
    The procedures presented above for sampling $\tau_i$ and $\pi_{i,j}$ are both polynomial in the domain size, which forms a lifted sampler for $\sentence$, and thus complete the proof.
\end{proof}

\begin{remark}
  Directly extending the approach above to the case of \fotwo{} necessitates another novel and more sophisticated strategy, specifically for the sampling of 2-tables.
  This is due to the fact that it is impossible to decouple the grounding of $\forall x\exists y: \extformula(x,y)$ into a conjunction of independent formulas, even when conditioned on the sampled 1-types.
\end{remark}

\section{\fotwo{} is Sampling Liftable}
\label{sec:fo2_liftable}

We now show the domain-liftability under sampling of the \fotwo{} fragment.
It is common for logical algorithms to operate on normal form representations instead of arbitrary sentences.
The normal form of \fotwo{} used in our sampling algorithm is the \textit{Scott normal form (SNF)}~\cite{scott1962decision}; an \fotwo{} sentence is in SNF, if it is written as:
\begin{equation}
  \label{eq:scott_form}
  \sentence = \forall x\forall y: \fotwoformula(x,y)\land\bigwedge_{k \in[m]} \forall x\exists y: \extformula_k (x,y),
\end{equation}
where the formulas $\fotwoformula(x,y)$ and $\extformula_k(x,y)$ are quantifier-free formulas.
It is well-known that one can convert any \fotwo{} sentence $\sentence$ in polynomial-time into a formula $\sentence_{SNF}$ in SNF such that $\sentence$ and $\sentence_{SNF}$ are equisatisfiable~\cite{gradel1997Bull.Symb.Log.}.
The principal idea is to substitute, starting from the atomic level and working upwards, any subformula $\varphi(x) = Qy: \phi(x,y)$, where $Q\in\{\forall,\exists\}$ and $\phi$ is quantifier-free, with an atomic formula $A_{\varphi}$, where $A_\varphi$ is a fresh predicate symbol.
This novel atom $A_\varphi(x)$ is then separately ``axiomatized'' to be equivalent to $\varphi(x)$.
If the weight of $A_\varphi$ is set to be $\weight(A_\varphi) = \negweight(A_\varphi) = 1$, we have that such reduction is also sound (recall the soundness definition in Definition~\ref{de:soundness}).
\begin{lemma}
  \label{le:snf_sound}
  For any WFOMS problem $\mathfrak{S} = (\sentence, \domain, \weight, \negweight)$ where $\sentence$ is an \fotwo{} sentence, there exists a WFOMS problem $\mathfrak{S}' = (\sentence', \domain, \weight', \negweight')$, where $\sentence'$ is in SNF, such that the reduction from $\mathfrak{S}$ to $\mathfrak{S}'$ is sound.
\end{lemma}
The proof is straightforward, as every novel predicate (e.g., $P_\varphi$) introduced in the SNF transformation is axiomatized to be equivalent to the subformula ($\varphi(x)$), 
and thus fully determined by the subformula in every model of the resulting SNF sentence (see the details in \ref{sub:snf}).

\begin{theorem}
    The fragment \fotwo{} is domain-liftable under sampling.
    \label{th:fotwoliftable}
\end{theorem}

We demonstrate the sampling liftability of \fotwo{} through the development of a lifted sampler that bears resemblance to the framework presented in Section~\ref{sec:ufo2_liftable}.
Specifically, the approach involves a two-stage algorithm derived from the probability decomposition of 1-types and 2-tables in~\eqref{eq:decompose_sampling}, which comprises the sampling of 1-types $\tau_i$ in the first stage, followed by the sampling of 2-tables $\pi_{i,j}$ in the second stage. 
In the first stage, the same technique used for \ufotwo{} as discussed in Section~\ref{sec:ufo2_liftable} can be used. 
The time complexity of this process remains polynomial in the size of the domain following the same reasoning and the domain-liftability of \fotwo{}~\cite{vandenbroeck2014Proc.FourteenthInt.Conf.Princ.Knowl.Represent.Reason.}.
The second stage, which is to sample 2-tables conditional on the sampled 1-types, however, is the most challenging aspect of the sampling problem and will be the main focus of the remainder of this section.

\subsection{A Working Example}
\label{sub:an_intuitive_example}

Before delving into the details of the algorithm, we provide a working example to illustrate the basic idea for sampling 2-tables.
The example is to sample an undirected graph of size $n$ without any isolated vertex uniformly at random.
Its corresponding sentence can be written in SNF:
\begin{equation*}
  \sentence_G := \left(\forall x\forall y: (E(x,y) \Rightarrow E(y,x))\land \neg E(x,x)\right)\land \left(\forall x\exists y: E(x,y)\right),
\end{equation*}
and the sampling problem corresponds to a WFOMS problem on $\sentence_G$ under $\weight(E) = \negweight(E) = 1$ over a domain of vertices $V = \{v_i\}_{i\in[n]}$.
In this sentence, the only 1-type that can be sampled is $\neg E(x,x)$, which does not require any sampling.
The 2-tables that can be sampled are $\pi^1(x,y) = E(x,y)\land E(y,x)$ and $\pi^2(x,y) = \neg E(x,y)\land \neg E(y,x)$ representing the connectedness of two vertices.
In the following, we will focus on the sampling problem of 2-tables (i.e., edges).

We first apply the following transformation
on $\sentence_G$ resulting in $\sentence_{GT}$:
\begin{enumerate}
  \item introduce an auxiliary \textit{Tseitin predicate} $Z/1$ with the weight $\weight(Z)=\negweight(Z) = 1$ that indicates the non-isolation of vertices,
  \item append $\forall x: Z(x)\Leftrightarrow \exists y: E(x,y)$ to $\sentence_G$, and
  \item remove $\forall x\exists y: E(x,y)$.
\end{enumerate}
We then consider a slightly more general WFOMS problem on $\generalsentence_G := \sentence_{GT}\land\bigwedge_{v\in V_\exists}Z(v)$ over a domain $V_\forall$, where $V_\exists\subseteq V_\forall\subseteq V$ and $V_\exists$ represents the set of vertices that should be non-isolated in the graph induced by $V_\forall$.
The original WFOMS problem of $\sentence_G$ can be clearly reduced to the more general problem by setting $V_\exists=V_\forall=V$,
and the reduction is sound with the mapping function $f(\mu') = \proj{\mu'}{\{E\}}$.
Given a $\mathcal{P}_{\generalsentence_G}$-structure $\structure$, the interpretation of the predicate $E$ fully determines the interpretation of $Z$.
Therefore, in the subsequent discussion, any $\mathcal{P}_{\generalsentence_G}$-structure should be understood as a $\mathcal{P}_{\{E\}}$-structure, where $Z$ is omitted.



Given an $\mathcal{P}_{\generalsentence_G}$-structure $\structure$, we denote the substructure of $\structure$ concerning a vertex $v_i\in V_\forall$ by $\structure_i$:
\begin{equation*}
  \structure_i := \bigcup_{v_j\in V_\forall: j\neq i} \pi_{i,j}(v_i, v_j),
\end{equation*}
where $\pi_{i,j}$ is the 2-table of $(v_i, v_j)$.
We then proceed to choose a vertex $v_t$ from $V_\forall$, and decompose the sampling probability of $\structure$ as follows:
\begin{equation}
  \pro[\structure\mid \generalsentence_G] 
  =\pro\left[\structure\mid\generalsentence_G\land\structure_t\right]\cdot\pro[\structure_t\mid\generalsentence_G].
  \label{eq:decomposition_nonisolated_graph}
\end{equation}
The decomposition leads to two successive sub-problems of the general WFOMS problem: the first one is to sample a substructure $\structure_t$ from $\generalsentence_G$; and the second can be viewed as a new WFOMS problem on $\generalsentence_G$ conditioned on the sampled $\structure_t$.

\begin{figure}[!bt]
  \centerline{\includegraphics[width=.67\textwidth]{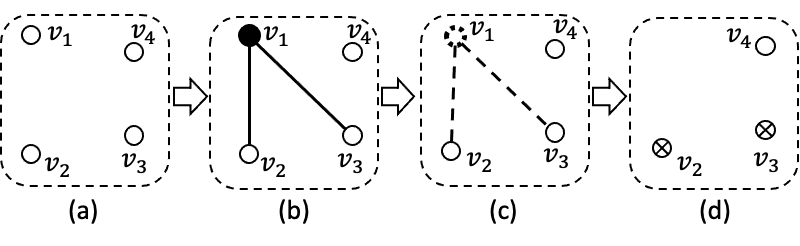}}
  \caption{A sampling step for an undirected graph with no isolated vertices: (a) begins with an initial graph that has no edges, and in the more general sampling problem, $V_\forall = V_\exists = \{v_1,v_2,v_3,v_4\}$; (b) sample edges for the vertex $v_1$; (c) remove the vertex $v_1$ with its sampled edges; (d) and obtain a graph with some vertices already non-isolated ($v_2$ and $v_3$), resulting in a new sampling problem with $V_\forall' = \{v_2,v_3,v_4\}$ and $V_\exists' = \{v_4\}$.}
  \label{fig:non-isolated_graph}
\end{figure}

We first show that if $\structure_t$ is valid w.r.t. $(\generalsentence_G, V_\forall)$, the new WFOMS problem can be reduced to the general WFOMS problem but with the smaller domain
$V_\forall' = V_\forall \setminus \{v_t\}$ and
\begin{equation*}
  V_\exists' = V_\exists \setminus \left(\{v_t\}\cup \{v_i\mid v_i\in V_\exists: \pi_{t,i} = \pi^1 \}\right).
\end{equation*}
The set $V_\exists'$ removes the vertices that become non-isolated after the sampling of $\structure_t$ from $V_\exists$.
It is easy to check that the reduction is sound, because every model of $\sentence_{GT}\land\bigwedge_{v\in V_\exists'}Z(v)$ over $V_\forall'$ can be mapped to a unique model of $\generalsentence_G\land\structure_t$ over $V_\forall$, and vice versa, without affecting the satisfaction and weight of the models.

By the reduction above, the decomposition of \eqref{eq:decomposition_nonisolated_graph} can be performed recursively on the WFOMS problem on $\generalsentence_G$ over $V_\forall$.
Specifically, the recursive algorithm takes $V_\forall$ and $V_\exists$ as input, 
\begin{enumerate}
  \item selects a vertex $v_t$ from $V_\forall$, 
  \item samples its substructure $\structure_t$ according to the probability $\pro[\structure_t\mid\generalsentence_G]$, and
  \item obtains a new problem with updated $V_\forall'$ and $V_\exists'$ for recursion.
\end{enumerate}
The recursion terminates when all substructures $\structure_i$ are sampled (i.e., $V_\forall$ contains a single vertex), or the problem degenerates to a WFOMS problem on \ufotwo{} sentence (i.e., $V_\exists$ is empty).
The number of recursions is less than $|V|$, the total number of vertices.
An example of a single recursive step is shown in Figure~\ref{fig:non-isolated_graph}: the vertex $v_1$ from $V_\forall = V_\exists = \{v_1, v_2, v_3, v_4\}$ is selected with the edges incident to it sampled; the vertex $v_1$ is then removed from $V_\forall$, and the vertices $v_2$ and $v_3$ that become non-isolated are also removed from $V_\exists$.

The remaining problem is to sample the substructure $\structure_t$ according to $\pro[\structure_t\mid\generalsentence_G]$.
Recall that $\structure_t$ determines the edges between $v_t$ and vertices in $V_\forall'$.
Let $V_1 = V_\forall' \setminus V_\exists$ and $V_2 = V_\forall' \setminus V_1$.
We can effectively generate a sample of $\structure_t$ by sampling two binary partitions of $V_1$ and $V_2$, respectively, yielding the sets ${V_{11}, V_{12}}$ and ${V_{21}, V_{22}}$; the vertices in $V_{11}$ and $V_{21}$ will be connected to $v_t$, while the vertices in $V_{12}$ and $V_{22}$ will be disconnected to $v_t$.
It can be demonstrated that the sampling probability of a substructure $\structure_t$ only depends on the two partition configurations $(|V_{11}|, |V_{12}|)$ and $(|V_{21}|, |V_{22}|)$.
The proof of this claim can be found in Section~\ref{subsub:samplingsubstru}, where the more general case of \fotwo{} sampling is addressed.
As a result, the sampling of $\structure_t$ can be achieved by first sampling the two partition configurations, followed by two random partitions on $V_1$ and $V_2$ with the respective sampled configurations.
To sample a tuple of partition configurations, we use the enumerative sampling method.
The number of all possible tuples of partition configurations is clearly polynomial in $|V_\forall'|$, and it will be shown in Section~\ref{subsub:samplingsubstru} that the sampling probability of each configurations tuple can be computed in time polynomial in $|V_\forall'|$. 
Therefore, the complexity of the sampling algorithm is polynomial in the number of vertices.
This, together with the complexity of the recursion procedure, which is also polynomial in the number of vertices, implies that the whole sampling algorithm is lifted.

\subsection{Domain Recursive Sampling for 2-tables}
\label{sub:domain_recursive_sampling_for_2_tables}

\interfootnotelinepenalty=10000

We now present our algorithm for sampling 2-tables, which uses the technique illustrated in the previous subsection.
The core idea, called \textit{domain recursion}, involves considering individual objects from the domain at a time, sampling their corresponding atomic facts, and subsequently obtaining a new sampling problem that has a similar form to the original one but with a smaller domain and potentially fewer existentially-quantified formulas. 
This process is repeated recursively on the reduced sampling problems until the domain has become a singleton or all the existentially quantified formulas have been eliminated.

Let us consider the WFOMS problem with fixed 1-types 
\begin{equation*}
  (\sentence\land\bigwedge_{i\in[n]} \tau_i(e_i),\domain, \weight, \negweight)
\end{equation*} 
where $\sentence$ is a sentence in SNF~\eqref{eq:scott_form}, $\domain=\{e_i\}_{i\in[n]}$ is a domain of size $n$, and each $\tau_i$ is the sampled 1-type of the element $e_i$.
W.l.o.g.\footnote{
Any SNF sentence can be transformed into such form by introducing an auxiliary predicate $R_k$ with weights $\weight(R_k)=\negweight(R_k)=1$ for each $\extformula_k(x,y)$, appending $\forall x\forall y: R_k(x,y) \Leftrightarrow \extformula_k(x,y)$ to the sentence, and replacing $\extformula_k(x,y)$ with $R_k(x,y)$. The transformation is obviously sound when viewing it as a reduction in WFOMS.}
, we suppose that each formula $\extformula_k(x,y)$ in the SNF sentence~\eqref{eq:scott_form} is an atomic formula $R_k(x,y)$, where each $R_k$ is a binary predicate with weights $\weight(R_k) = \negweight(R_k) = 1$.

\subsubsection{A More General Sampling Problem}
\label{sub:generalproblem}

We first construct the following sentence from $\sentence$:
\begin{equation}
  \sentence_T := \forall x\forall y: \fotwoformula(x,y) \land \bigwedge_{k\in [m]} \forall x: Z_k(x) \Leftrightarrow (\exists y: R_k(x,y)),
  \label{eq:tseitin_reduction}
\end{equation}
where each $Z_k/1$ is a Tseitin predicate with the weight $\weight(Z_k) = \negweight(Z_k) = 1$.
Note that in~\eqref{eq:tseitin_reduction}, the interpretation of $R_k$ fully determines the interpretation of $Z_k$.
Once the 2-tables are sampled, the interpretation of $Z_k$ is also fixed, adding no additional cost to the sampling problem.
Therefore, for ease of presentation, we will omit the handling of $Z_k$ in the following discussion.

We then consider a more general WFOMS problem of the following sentence
\begin{equation}
    \sentence_T \land \bigwedge_{i\in [n]} \left(\tau_i(e_i)\land \mathcal{C}_i\right)
  \label{eq:general_sentence}
\end{equation}
over a domain of $\{e_i\}_{i\in[n]}$,
where each $\mathcal{C}_i$ is a conjunction over a subset of the ground atoms $\{Z_k(e_i)\}_{k\in[m]}$.
We call $\mathcal{C}_i$ the \textit{existential constraint} on the element $e_i$ and allow $\mathcal{C}_i = \top$, which means $e_i$ is not existentially quantified.

The more general sampling problem has the necessary structure for the domain recursion algorithm to be performed.
To verify it, the original WFOMS problem on $\sentence\land_{i\in[n]} \tau_i(e_i)$ can be reduced to the more general problem by setting all existential constraints to be $\bigwedge_{k\in [m]} Z_k(x)$.
On the other hand, the WFOMS problem on the \ufotwo{} sentence $\forall x\forall y: \fotwoformula(x,y)\land \bigwedge_{i\in[n]} \tau_i(e_i)$ is also reducible to the problem with $\mathcal{C}_i = \top$ for all $i\in[n]$. 
It is easy to check that these two reductions are both sound.

\subsubsection{Block and Cell Types}
\label{subsub:block_cell_types}

It is worth noting that the Tseitin predicates $Z_k$ introduced in the more general sampling problem are not contained in the given 1-types $\tau_i$.
Therefore, to incorporate these predicates into the sampling problem, we introduce the notions of block and cell types.

Consider a set of Tseitin predicates $\{Z_k\}_{k\in[m]}$.
A \textit{block type} $\beta$ is a subset of the atoms $\{Z_k(x)\}_{k\in[m]}$.
The number of the block types is $2^m$, where $m$ is the number of existentially-quantified formulas.
We often represent a block type as $\beta(x)$ and view it as a conjunctive formula over the atoms within the block.
It is important to note that the block types only indicate which Tseitin atoms should hold for a given element, while the Tseitin atoms not covered by the block types are left unspecified.
In contrast, the 1-types explicitly determine the truth values of all unary and reflexive atoms, excluding the Tseitin atoms.

The block type and 1-type work together to define the cell type.
A \textit{cell type} $\eta = (\beta, \tau)$ is a pair of a block type $\beta$ and a 1-type $\tau$.
We also write a cell type as a conjunctive formula of $\eta(x) = \beta(x) \land \tau(x)$.
The cell type of an element is simply the tuple of its block type (given by the sentence $\generalsentence$) and 1-type (already sampled in the first stage of the algorithm).
The cell types of elements naturally produce a partition of the domain, similar to how a 1-types assignment divides a domain into disjoint subsets.
Each disjoint subset of elements is called a \textit{cell}.
The \textit{cell configuration} is defined as the configuration of the corresponding partition.

With the notion of block and cell type, we can write the existential constraint $\mathcal{C}_i$ in \eqref{eq:general_sentence} as $\beta_i(e_i)$, where $\beta_i$ is the block type of the element $e_i$, and the sentence \eqref{eq:general_sentence} as $\sentence_T\land \bigwedge_{i\in[n]} \eta_i(e_i)$, where $\eta_i = (\beta_i, \tau_i)$.

\subsubsection{Domain Recursion}

We now show how to apply the domain recursion scheme to the WFOMS problem $(\sentence_T\land\bigwedge_{i\in[n]}\eta_i(e_i),\domain, \weight, \negweight)$.
Let $\structure_i$ be the set of ground 2-tables concerning the element $e_i$:
\begin{equation}
  \label{eq:substructure}
  \structure_i := \bigcup_{j\in[n]: j\neq i} \pi_{i,j}(e_i, e_j).
\end{equation}
By the domain recursion, we select an arbitrary element $e_t$ from $\domain$ and decompose the sampling probability of a structure $\structure$ into
\begin{align*}
  \pro\left[\structure \mid \sentence_T\land\bigwedge_{i\in[n]}\eta_i(e_i)\land\structure_t \right]\cdot \pro\left[\structure_t\mid\sentence_T\land\bigwedge_{i\in[n]}\eta_i(e_i)\right].
\end{align*}
We always assume that the substrcture $\structure_t$ is valid w.r.t. $(\sentence_T\land\bigwedge_{i\in[n]}\eta_i(e_i), \domain)$, i.e., $\sentence_T\land \bigwedge_{i\in[n]}\eta_i(e_i)\land\structure_t$ is satisfiable, or the sampling probability of $\structure_t$ is zero.
We will demonstrate that the WFOMS problem specified by the first probability in the above equation can be reduced to a new WFOMS problem of the same form as the original sampling problem, but over a smaller domain $\domain' = \domain \setminus {e_t}$.

We first introduce the notion of the \textit{relaxing} on block types and cell types, which is the basic operation of the reduction.
Given a 2-table $\pi$ and a block type $\beta$, the relaxed block of $\beta$ under $\pi$ is defined as
\begin{equation*}
  \relaxcelltype{\beta}{\pi} := \beta \setminus \{Z_k(x)\mid k\in[m]: R_k(y,x)\in\pi\}.
\end{equation*}
Similarly, we can apply the relaxation under $\pi$ on a cell type $\eta = (\beta, \tau)$, resulting in a relaxed cell type $\relaxcelltype{\eta}{\pi} = (\relaxcelltype{\beta}{\pi}, \tau)$.

\begin{example}
  Consider a WFOMS problem on the sentence 
  \begin{align*}
    \forall x&: (Z_1(x) \Leftrightarrow (\exists y: R_1(x,y))\land \\
    \forall x&: (Z_2(x) \Leftrightarrow (\exists y: R_2(x,y)) \\
    \land &\dots \land Z_1(a)\land Z_2(a)\land \dots,
  \end{align*}
  where only the block type $\beta=\{Z_1(x), Z_2(x)\}$ of the element $a$ is shown.
  Suppose the sampled 2-table of $(e_t, a)$ is $$\pi=\{\neg R_1(x,y), R_1(y,x), R_2(x,y), \neg R_2(y,x)\}.$$
  Then the relaxed block of the element $a$ is $\relaxcelltype{\beta}{\pi} = \{Z_2(x)\}$, as the corresponding quantified formula $\exists y: R_1(a,y)$ of $Z_1(a)$ is satisfied by the fact $R_1(a, e_t)$ in $\pi(e_t, a)$.
\end{example}

Let 
\begin{equation}
  \recursivesentence := \sentence_T\land\bigwedge_{i\in[n]}\eta_i(e_i)\land \structure_t,
  \label{eq:recursive_sentence}
\end{equation}
and
\begin{equation}
  \recursivesentence' := \sentence_T\land \bigwedge_{i\in[n]\setminus \{t\}}\relaxcelltypeb{\eta_i}{\pi_{t,i}}(e_i).
  \label{eq:reduced_sentence}
\end{equation}We have the following lemma.
\begin{lemma}
  \label{lemma:modular}
  If $\recursivesentence$ is satisfiable, the reduction from the WFOMS problem $(\recursivesentence, \domain, \linebreak, \weight, \negweight)$ to $(\recursivesentence', \domain', \weight, \negweight)$ is sound.
\end{lemma}

\begin{proof}
  We prove the soundness of the reduction by constructing a mapping function from models of $\recursivesentence'$ to models of $\recursivesentence$.
  This function is defined as $f(\mu') = \mu'\cup \structure_t\cup \tau_t(e_t)$.
  The function $f$ is clearly deterministic, polynomial-time.
  To simplify the remaining arguments of the proof, we first show that $f$ is bijective.
  Write the sentences $\recursivesentence$ and $\recursivesentence'$ as
  \begin{align*}
    \recursivesentence = \forall x\forall y: \fotwoformula(x,y)\land \bigwedge_{i\in[n]}\left(\tau_i(e_i)\land \beta_i(e_i)\right) \land \structure_t \land \Lambda,\\
    \recursivesentence' = \forall x\forall y: \fotwoformula(x,y)\land \bigwedge_{i\in[n]\setminus\{t\}}\left(\tau_i(e_i)\land \relaxcelltypeb{\beta_i}{\pi_{t, i}}(e_i)\right) \land \Lambda,
  \end{align*}
  where $\Lambda = \bigwedge_{k\in[m]}\forall x: Z_k(x)\Leftrightarrow (\exists y: R_k(x, y))$.
  
  \paragraph*{($\Rightarrow$)}
  For any model $\mu'$ in $\fomodels{\recursivesentence'}{\domain'}$, we prove that $f(\mu') = \mu'\cup \structure_t \cup \tau_t(e_t)$ is a model in $\fomodels{\recursivesentence}{\domain}$.
  First, one can easily check that $f(\mu')$ satisfies $\forall x\forall y: \fotwoformula(x,y)\land \bigwedge_{i\in[n]}\tau_i(e_i)$.
  Next, we show that $f(\mu')$ also satisfies $\bigwedge_{i\in[n]}\beta_i(e_i) \land \Lambda$.
  For any element $i\in[n]\setminus\{t\}$, we have that $\exists y: R_k(e_i, y)$ is true in $\mu'$ (and also $f(\mu')$) for all $k\in[m]$ such that $Z_k(x)\in \relaxcelltype{\beta_i}{\pi_{t,i}}$.
  By the definition of the relaxed block, for any $k\in[m]$ such that $Z_k(x)\in \beta_i\setminus\relaxcelltypeb{\beta_i}{\pi_{t, i}}$, the ground relation $R_k(e_i, e_t)$ is in $\structure_t$, and thus $\exists y: R_k(e_i, y)$ is also satisfied in $f(\mu')$.
  It follows that $\exists y: R_k(e_i, y)$ is true in $f(\mu')$ for all $k\in[m]$ such that $Z_k(x)\in \beta_i$.
  Furthermore, the satisfaction of $\exists y: R_k(e_t, y)$ for all $k\in[m]$ such that $Z_k(x)\in \beta_t$ is guaranteed by the satisfiability of $\recursivesentence$.
  Therefore, it is easy to show that $f(\mu')$ satisfies $\bigwedge_{i\in[n]}\beta_i(e_i) \land \Lambda$, which together with the satisfaction of $\forall x\forall y: \fotwoformula(x,y)\land \bigwedge_{i\in[n]}\tau_i(e_i)$ implies that $f(\mu')$ is a model of $\recursivesentence$.
  
  \paragraph*{($\Leftarrow$)} 
  For any model $\mu$ in $\fomodels{\recursivesentence}{\domain}$, we show that there exists a unique model $\mu'$ in $\fomodels{\recursivesentence'}{\domain'}$ such that $f(\mu') = \mu$.
  Let the respective structure be $\mu' = \mu \setminus ((\structure_t) \cup \tau_t(e_t))$, and the uniqueness is is clear from the definition of $f$.
  First, it is easy to check that $\mu'$ satisfies $\forall x\forall y: \fotwoformula(x,y)\land \bigwedge_{i\in[n]\setminus\{t\}}\tau_i(e_i)$.
  Then we show that $\mu'$ also satisfies $\bigwedge_{i\in[n]\setminus\{t\}}\relaxcelltypeb{\beta_i}{\pi_{t, i}}(e_i)\land \Lambda$.
  For any $i\in[n]\setminus\{t\}$, we have that $\exists y: R_k(e_i, e)$ is true in $\mu$.
  Grounding $\exists y: R_k(e_i, e)$ over $\domain$ and replacing the atoms $R_k(e_t, e_j)$ with their truth values in $\structure_t$, we have that $\bigvee_{j\in[n]\setminus\{t\}}R_k(e_i, e_j)$ is true in $\mu'$ for all $k\in[m]$ such that $Z_k(x)\in \relaxcelltype{\beta_i}{\pi_{t,i}}$.
  Thus, $\mu'$ also satisfies $\bigwedge_{i\in[n]\setminus\{t\}}\relaxcelltypeb{\beta_i}{\pi_{t, i}}(e_i)\land \Lambda$.
  This, along with the satisfaction of $\forall x\forall y: \fotwoformula(x,y)\land \bigwedge_{i\in[n]\setminus\{t\}}\tau_i(e_i)$, leads to the conclusion.

  Now, we are prepared to demonstrate the consistency of sampling probability through the mapping function.
  Since $f$ is bijective, it is sufficient to show
  \begin{equation*}
    \pro[f(\mu') \mid\recursivesentence; \domain, \weight, \negweight] = \pro[\mu' \mid \recursivesentence'; \domain', \weight, \negweight]
  \end{equation*}
  for any model in $\fomodels{\recursivesentence'}{\domain'}$.
  By the definition of the mapping function $f$, we have
  \begin{equation*}
    \typeweight{f(\mu')} = \typeweight{\mu'}\cdot\typeweight{\structure_t}\cdot\typeweight{\tau_t}.
  \end{equation*}
  Moreover, due to the bijectivity of $f$, we have
  \begin{equation}
    \begin{aligned}
      \symwfomc(\recursivesentence, \domain, \weight, \negweight) &=\sum_{\mu\in\fomodels{\recursivesentence}{\domain}}\typeweight{\mu}\\
      &=\sum_{\mu'\in\fomodels{\recursivesentence'}{\domain'}}\typeweight{f(\mu')}\\
      &=\typeweight{\structure_t}\cdot\typeweight{\tau_t}\cdot \sum_{\mu'\in\fomodels{\recursivesentence'}{\domain'}}\typeweight{\mu'}\\
      &=\typeweight{\structure_t}\cdot\typeweight{\tau_t}\cdot\symwfomc(\recursivesentence', \domain', \weight, \negweight).
    \end{aligned}
    \label{eq:modular_wfomc}
  \end{equation}
  Finally, by the definition of conditional probability, we can write
  \begin{equation}
    \begin{aligned}
      \pro[f(\mu')\mid\recursivesentence;\domain, \weight, \negweight]&=\frac{\typeweight{f(\mu')}}{\symwfomc(\recursivesentence, \domain, \weight, \negweight)}\\
      &=\frac{\typeweight{\mu'}\cdot\typeweight{\structure_t}\cdot\typeweight{\tau_t}}{\typeweight{\structure_t}\cdot\typeweight{\tau_t}\cdot\symwfomc(\recursivesentence', \domain', \weight, \negweight)}\\
      &=\frac{\typeweight{\mu'}}{\symwfomc(\recursivesentence', \domain', \weight, \negweight)}\\
      &=\pro[\mu'\mid\recursivesentence';\domain', \weight, \negweight],
    \end{aligned}
    \label{eq:consistent_weight}
  \end{equation}
  completing the proof.
\end{proof}

With the sound reduction presented above, one can readily perform the domain recursion on the more general WFOMS problem.
The only remaining task is the sampling of the substructure $\structure_t$ at each recursive step, which will be discussed in the following subsection.

\subsubsection{Sampling substructures}
\label{subsub:samplingsubstru}

For the sake of brevity, we shall limit our focus to the initial recursive step involving sampling $\structure_t$ from $\sentence_T\land\bigwedge_{i\in[n]} \eta_i(e_i)$. 
The subsequent steps can be executed by the same process.

We follow a similar approach as in the 1-types sampling, which samples $\structure_t$ through random partitions on cells.
Let $N_c$ be the total number of cell types, and fix the linear order of cell types as $\eta^1, \eta^2, \dots, \eta^{N_c}$.
For all $i\in[N_c]$, denote by $C_{\eta^i}\subseteq \domain'$ the cells of $\eta^i$ constructed from the cell types $\eta_1, \eta_2, \dots, \eta_{n-1}$:
\begin{equation*}
  C_{\eta^i} := \{e_j \mid j\in[n-1]: \eta_j = \eta^i\}.
\end{equation*}
Let $N_b$ be the number of all 2-tables, and fix the linear order of 2-tables as $\pi^1, \pi^2, \dots, \pi^{N_b}$.
Recall that $\structure_t$ consists of the ground 2-tables of all tuples of $e_t$ and the elements in $\domain'$.
Any substructure $\structure_t$ can be represented as a collection of $N_c$ partitions; each partition is applied to a cell $C_{\eta^i}$ to split it into $N_b$ disjoint subsets; each of these subsets is associated with a specific 2-table $\pi^j$, and contains precisely the elements $e$, whose combination $(e_t, e)$ with $e_t$ has the 2-table $\pi^j$.

\begin{figure}[!bt]
  \centerline{\includegraphics[width=.5\textwidth]{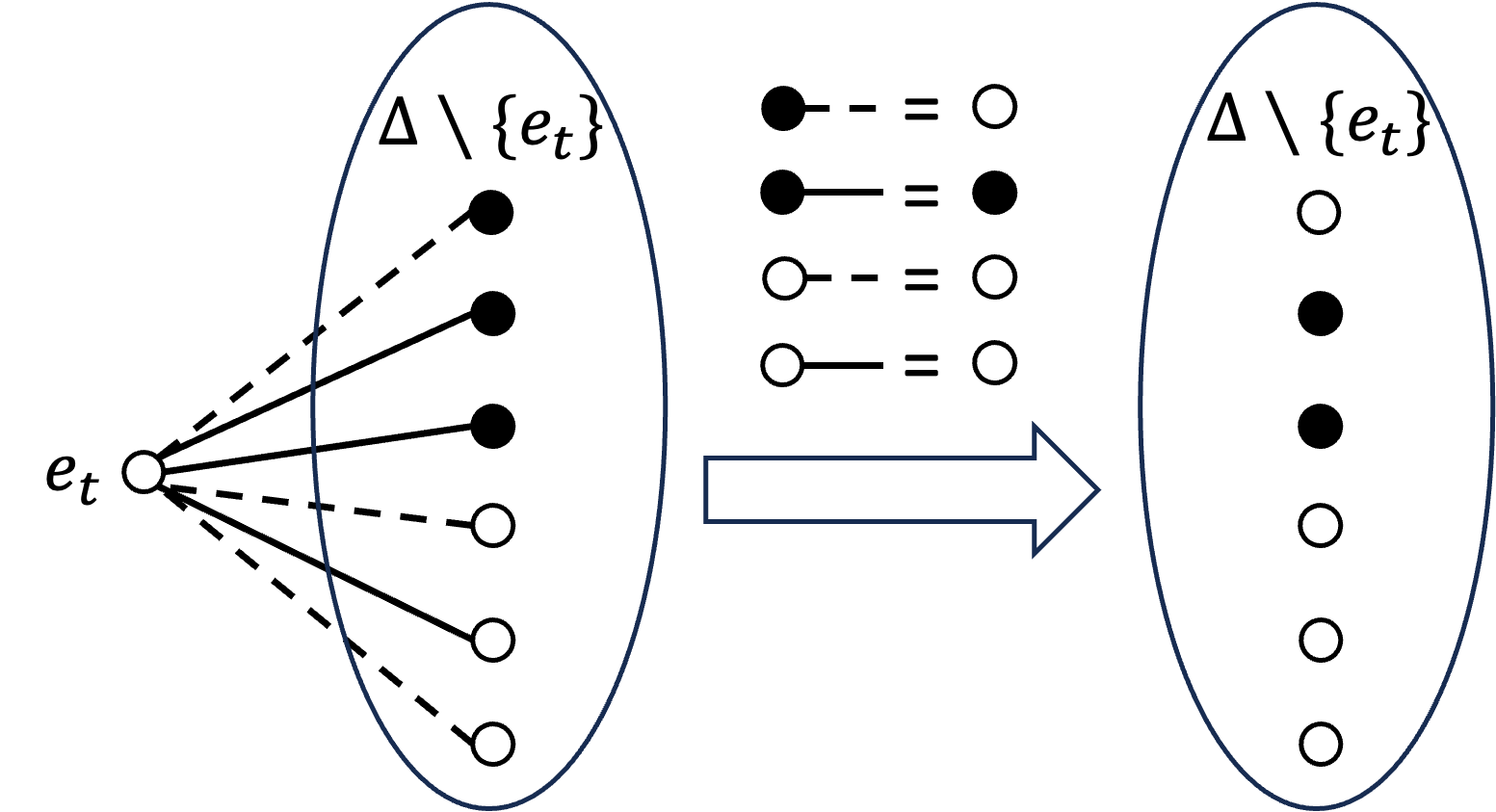}}
  \caption{
    An example of sampling a substructure. 
    There are two cell types, denoted by \faCircle{} and \faCircleThin{}, and two 2-tables, denoted by \dashed{} and \full{}. The sampled substructure $\structure_t$ is represented on the left, whose 2-tables configuration is $(g_{\text{\faCircle{}}, \text{\dashed{}}}^{\structure_t}, g_{\text{\faCircle{}}, \text{\full{}}}^{\structure_t}, g_{\text{\faCircleThin{}}, \text{\dashed{}}}^{\structure_t}, g_{\text{\faCircleThin{}}, \text{\full{}}}^{\structure_t}) = (1, 2, 2, 1)$. 
    With the relaxations of cell types defined above the arrow, where, e.g., $\text{\faCircle{}}\text{\dashed}=\text{\faCircleThin{}}$ means the relaxed cell type of \faCircle{} under \dashed{} is \faCircleThin{}, the reduced cell types for each element is presented on the right. 
    One can easily obtain the reduced cell configuration from the 2-table configuration $(g_{\text{\faCircle{}}, \text{\dashed{}}}^{\structure_t}, g_{\text{\faCircle{}}, \text{\full{}}}^{\structure_t}, g_{\text{\faCircleThin{}}, \text{\dashed{}}}^{\structure_t}, g_{\text{\faCircleThin{}}, \text{\full{}}}^{\structure_t})$.
  }
  \label{fig:substrcutresampling}
\end{figure}

Consider a substructure $\structure_t$ and its corresponding partitions.
For any cell $C_{\eta^i}$, let $g_{\eta^i, \pi^j}^{\structure_t}$ represents the cardinality of the subset in $C_{\eta^i}$ associated with the 2-table $\pi^j$.
We can write the configuration of the partition of cell $C_{\eta^i}$ as $(g_{\eta^i, \pi^j}^{\structure_t})_{j\in[N_b]}$, and denote it by the vector $\vecg_{\eta^i}^{\structure_t}$.
The \textit{2-tables configuration} $\vecg$ of $\structure_t$ is then defined as the concatenation of partition configurations over all cells, i.e., $\vecg^{\structure_t}:= \bigoplus_{i\in[N_c]}\vecg_{\eta^i}^{\structure_t}$, where $\oplus$ denotes the concatenation operator.
Figure~\ref{fig:substrcutresampling} illustrates an example of a substructure and its 2-tables configuration.
In the following, we show that the sampling probability of $\structure_t$ is entirely determined by its corresponding 2-tables configuration.

To begin, it is clear that the sampling probability of $\structure_t$ is proportional to the value of $\symwfomc(\sentence_T\land\bigwedge_{i\in[n]}\eta_i(e_i)\land\structure_t, \domain, \weight, \negweight)$.
As stated by~\eqref{eq:modular_wfomc}, we can write it as
\begin{equation}
  \symwfomc(\recursivesentence', \domain', \weight, \negweight)\cdot \typeweight{\tau_t}\cdot \typeweight{\structure_t}
  \label{eq:sampling_a_n}
\end{equation}
where $\recursivesentence'=\sentence_T\land \bigwedge_{i\in[n]\setminus \{t\}}\relaxcelltypeb{\eta_i}{\pi_{t,i}}(e_i)$ is the reduced sentence given the substrcture $\structure_t$.
Denote by $\vecweight = \left(\typeweight{\pi^i}\right)_{i\in N_b}$ the weight vector of 2-tables.
We can write \eqref{eq:sampling_a_n} as
\begin{equation}
  \symwfomc(\recursivesentence', \domain', \weight, \negweight)\cdot \typeweight{\tau_t} \cdot \prod_{i\in[N_c]} \vecweight^{\vecg_{\eta^i}^{\structure_t}}.
  \label{eq:structure_weight}
\end{equation}
In the expression above, the 1-types $\tau_t$ is fixed, and the last term fully depends on the 2-tables configurations $\vecg^{\structure_t}$.
The only part that needs further analysis is the \wfomc{} problem on $\recursivesentence'$.

The \wfomc{} problem on $\recursivesentence'$ can be viewed as the \wfomc{} problem on $\sentence_T$ conditioned on the cell types $\relaxcelltype{\eta_i}{\pi_{t,i}}$.
By the symmetry of the weighting function, its value does not depend on the specific cell types $\relaxcelltype{\eta_i}{\pi_{t,i}}$ assigned to each element, but instead relies on the overall cell configuration corresponding to the cell types $\relaxcelltype{\eta_i}{\pi_{t,i}}$.
We denote this cell configuration by a vector $(n_{\eta^i})_{i\in[N_c]}$, where $n_{\eta^i}$ represents the number of elements whose cell type is $\eta^i$ under the relaxation of 2-tables in $\structure_t$.
By the definition of $g_{\eta^i, \pi^j}^{\structure_t}$, which is the number of elements in the cell $C_{\eta^i}$ that are relaxed by the 2-table $\pi^j$, we can write each $n_{\eta^i}$ as
\begin{equation}
  n_{\eta^i} = \sum_{j\in[N_c]}\sum_{k\in[N_b]: \relaxcelltype{\eta^j}{\pi^k} = \eta^i} g_{\eta^j, \pi^k}^{\structure_t}.
  \label{eq:configuration_reduction}
\end{equation}
For the example in Figure~\ref{fig:substrcutresampling}, there are two cell types \faCircle{} and \faCircleThin{} and two block types \dashed{} and \full{}.
The relaxed cell type of \faCircleThin{} under both \dashed{} and \full{} is \faCircleThin{}, while the relaxed cell types of \faCircle{} under \dashed{} and \full{} are \faCircleThin{} and \faCircle{}, respectively.
By \eqref{eq:configuration_reduction}, we have $n_{\text{\faCircle{}}} = g_{\text{\faCircle{}}, \text{\full{}}}^{\structure_t}$ and $n_{\text{\faCircleThin{}}} = g_{\text{\faCircle{}}, \text{\dashed{}}}^{\structure_t} + g_{\text{\faCircleThin{}}, \text{\dashed{}}}^{\structure_t} + g_{\text{\faCircleThin{}}, \text{\full{}}}^{\structure_t}$.
Therefore, we have that the 2-tables configuration fully determines the cell configuration and, consequently, the value of $\symwfomc(\recursivesentence', \domain', \weight, \negweight)$.


According to the aforementioned reasoning, the sampling probability of a substructure is completely determined by its 2-tables configuration.
Thus, we can adopt a similar approach to the one we used to sample 1-types in Section~\ref{sub:sampling-1-types} and sample $\structure_t$ by first sampling a 2-tables configuration, and then randomly partitioning the cells accordingly.
The sampling for the 2-tables configuration can be achieved by the enumerative sampling method.
For any 2-tables configuration $\vecg^{\structure_t}$, its sampling probability is proportional to the value of \eqref{eq:structure_weight} multiplied by $\prod_{i\in[N_c]}\binom{|C_{\eta^i}|}{\vecg_{\eta^i}^{\structure_t}}$, where $|C_{\eta^i}|$ is the cardinality of $C_{\eta^i}$.

Finally, we need to ensure that the sampled substructure $\structure_t$ is valid w.r.t. $(\sentence_T\land \bigwedge_{i\in[n]}\eta_i(e_i), \domain)$.
It can be easily achieved by imposing some constraints on the 2-tables configuration to be sampled.
A 2-table $\pi$ is called \textit{coherent} with a 1-types tuple $(\tau, \tau')$ if, for some domain elements $a$ and $b$, the interpretation of $\tau(a)\cup\pi(a,b)\cup\tau'(b)$ satisfies the formula $\fotwoformula(a,b)\land\fotwoformula(b,a)$.
The following two constraints on the 2-tables configuration can make the sampled substructure valid:
\begin{itemize}
    \item Any 2-table $\pi_{t,i}$ contained in $\structure_t$ must be coherent with $\tau_t$ and $\tau_i$, the 1-types of $e_t$ and $e_i$. 
    This translates to a requirement on the 2-tables configuration that when sampling a configuration of partition of a cell $C_{\eta^i}$, the cardinality of 2-tables that are not coherent with $\tau_t$ and $\tau_i$ is restricted to be $0$;
    \item For any Tseitin $Z_k$ in the block type $\beta_t$, the substructure $\structure_t$ must contain at least one ground atom $R_k(e_t,a)$, where $a$ is a domain element, to make $\structure_t$ satisfy the existentially quantified formula $\exists y:R_k(e_t,y)$. 
    This means that there must be at least one 2-table $\pi$ such that $R_k(x,y)\in\pi$, whose cardinality in some cells is nonzero.
\end{itemize}

\subsubsection{Sampling Algorithm}

\begin{algorithm}[!htb]
  \caption{$\code{TwoTablesSamplerForFO2}(\sentence_T, \domain, \weight, \negweight, (\eta_i)_{i\in[n]})$} 
  \label{alg:drsampler}
  \textbf{Input:} A sentence $\sentence_T$ of the form~\eqref{eq:tseitin_reduction}, a domain $\domain=\{e_i\}_{i\in[n]}$, a weighting $(\weight, \negweight)$ and the cell type $\eta_i = (\tau_i,\beta_i)$ of each domain element $e_i$\\
  \textbf{Output:} 2-tables $\{\pi_{i,j}\}_{i,j\in[n]}$
  \begin{algorithmic}[1]
    \If{$n = 1$}
    \State \Return $\emptyset$
    \EndIf
    \If{all block types $\beta_1,\dots, \beta_n$ are $\top$}
      \State \Return $\code{TwoTablesSamplerForUFO2}(\forall x\forall y: \fotwoformula(x,y), \domain, \weight, \negweight, (\tau_i)_{i\in[n]})$
    \EndIf
    \State Choose $t\in[n]$; $\domain'\gets\domain \setminus \{e_t\}$; $\mathbf{G}\gets\emptyset$
    \State Obtain the cells $C_{\eta^1},\dots,C_{\eta^{N_c}}$ given by $\eta_1,\dots,\eta_{n-1}$
    \For{$\left(\vecg_{\eta^i}\right)_{i\in[N_c]}\gets\bigotimes_{i\in[N_c]} \mathcal{T}_{|C_{\eta^i}|, N_b}$} \Comment{$\bigotimes$ is the Cartesian product}
    \State $\vecg\gets\bigoplus_{i\in[N_c]}\vecg_{\eta^i}$ 
    \If{$\code{ExSat}\left(\vecg, \eta_t\right)$}
      \State Obtain the reduced cell configuration $(n_{\eta^i})_{i\in[N_c]}$ from $\vecg$ by \eqref{eq:configuration_reduction}
      \State $p(\vecg) \gets \eqref{eq:structure_weight} \cdot \prod_{i\in[N_c]} \binom{n_{\eta^i}}{\vecg_{\eta^i}}$
      \State $\mathbf{G} \gets \mathbf{G} \cup \left\{\vecg\right\}$
    \EndIf
    \EndFor
    \State $\forall \vecg \in \mathbf{G}, p(\vecg) \gets p(\vecg)/\sum_{\vecg'\in\mathbf{G}}p(\vecg')$
    \State Sample a 2-tables configuration $\vecg^*$ from $\mathbf{G}$ with the probability $p$
    \State Remove $e_t$ from the cell $C_{\eta_t}$
    \State $\structure\gets \emptyset$
    \For{$i\in[N_c]$}
      \State Randomly partition the cell $C_{\eta^i}$ into $\left\{G_{\eta^i, \pi^j}\right\}_{j\in[N_b]}$ according to $\vecg^*_{\eta^i}$
      \For{$j\in[N_b]$}
        \State $\structure \gets \structure \cup \left\{\pi^j(e_t, e)\right\}_{e\in G_{\eta^i, \pi^j}}$
        \State $\forall e_s\in G_{\eta^i, \pi^j}, \eta_s' \gets \relaxcelltype{\eta_s}{\pi^j}$
      \EndFor
    \EndFor
    \State $\structure\gets \structure \cup \code{TwoTablesSamplerForFO2}(\sentence_T, \domain', \weight, \negweight, \left(\eta_i'\right)_{i\in[n-1]})$
    \State \Return $\structure$
  \end{algorithmic}
\end{algorithm}

By combining all the ingredients discussed above, we now present our sampling algorithm for 2-tables, as shown in Algorithm~\ref{alg:drsampler}.
The overall structure of the algorithm follows a recursive approach, where a recursive call with a smaller domain and relaxed cell types is invoked at Line~31.
The algorithm terminates when the input domain contains a single element (at Line~1) or there are no existential constraints on the elements (at Line~4).
In the latter case, the algorithm resorts to the 2-tables sampling algorithm for \ufotwo{} presented in Section~\ref{sub:sampling-2-tables}.
In Lines~10-22, all possible 2-tables configurations are enumerated.
For each configuration, we compute its corresponding weight in Lines~13-14 and decide whether it should be sampled in Lines~15-20, where $\code{Uniform}(0,1)$ is a random number uniformly distributed in $[0,1]$.
When the 2-tables configuration has been sampled, we randomly partition the cells, and then update the sampled 2-tables and the cell type of each element respectively at Line~27 and~28.
The submodule $\code{ExSat}(\vecg, \eta)$ at Line~12 is used to check whether the 2-tables configuration $\vecg$ guarantees the validity of the sampled substructures.
Its pseudo-code can be found in \ref{sub:exsat}.

\begin{lemma}
  \label{le:drsamplercomplexity}
  The complexity of $\code{TwoTablesSamplerForFO2}$ in Algorithm~\ref{alg:drsampler} is polynomial in the size of the input domain.
\end{lemma}
\begin{proof}
  The algorithm $\code{TwoTablesSamplerForFO2}$ is invoked at most $n$ times, where $n$ is the size of the domain.
  The main computation of each recursive call is the for-loop, where we need to iterate over all $\prod_{i\in[N_c]} |\mathcal{T}_{n_{\eta^i}, N_b}|$ possible configurations.
  Recall that the size of the configuration space $\mathcal{T}_{M, m}$ is polynomial in $M$.
  Thus the number of iterations executed by the loop is also polynomial in the domain size.
  The other complexity arises from the \wfomc{} problems on $W$ and \eqref{eq:structure_weight}.
  These problems can be viewed as conditional \wfomc{} on a set of unary facts, whose size is clearly polynomial in the domain size.
  Therefore, according to Proposition~\ref{pro:liftable_unary_evidence} and the liftability of \fotwo{}, the aforementioned counting problems can be solved in polynomial time in the domain size.
  As a result, the complexity of the entire algorithm is polynomial in the size of the input domain.
\end{proof}

\begin{remark}
  We note that there are several optimizations to \code{TwoTablesSamplerForFO2} in our implementation, e.g., heuristically choosing the domain element for recursion so that the algorithm can quickly reach the terminal condition.
  However, the algorithm as described here is easier to understand and efficient enough to prove our main result, so we leave the discussion of some of the optimizations to \ref{sub:optwms}.
\end{remark}

\begin{algorithm}[!htb]
  \caption{$\code{WeightedModelSamplerForFO2}(\sentence, \domain, \weight, \negweight)$} 
  \label{alg:weightedmodelsamplerfo2}
  \textbf{Input}: A SNF sentence $\sentence$, a domain $\domain$ of size $n$ and a weighting $(\weight, \negweight)$ \\
  \textbf{Output}: A model $\mu$ of $\sentence$ sampled according to the probability $\pro[\mu\mid\sentence]$
  \begin{algorithmic}[1]
    \State $\{\tau_i(e_i)\}_{i\in[n]} \gets \code{OneTypesSampler}(\sentence, \domain, \weight, \negweight)$
    \State Transform $\sentence$ into $\sentence_T$ by~\eqref{eq:tseitin_reduction}
    \State Set the block and cell types $\beta_i(x)$ and $\eta_i(x)$ for each $e_i$ as $\bigwedge_{k\in[m]}Z_k(x)$ and $(\beta_i, \tau_i)$ respectively
    \State $\{\pi_{i,j}(e_i, e_j)\}_{i,j\in[n]} \gets \mathsf{TwoTablesSamplerForFO2}(\sentence_T, \domain, \weight, \negweight, (\eta_i)_{i\in[n]})$
    \State \Return $\{\tau_i(e_i)\}_{i\in[n]} \cup \{\pi_{i,j}(e_i, e_j)\}_{i,j\in[n]}$
  \end{algorithmic}
\end{algorithm}

With the proposed \code{TwoTablesSamplerForFO2}, we can now prove Theorem~\ref{th:fotwoliftable}.

\begin{proof}[Proof of Theorem~\ref{th:fotwoliftable}]
  By Lemma~\ref{le:snf_sound}, it is sufficient to demonstrate that all SNF sentences are sampling liftable.
  To achieve this, we construct a lifted sampler for SNF sentences as shown in Algorithm~\ref{alg:weightedmodelsamplerfo2}.
  The sampler consists of two stages, one for sampling 1-types and the other for 2-tables, in a manner similar to the sampler for \ufotwo{}.
  The sampling of 1-types can be realized using the 1-types sampler for \ufotwo{} presented in Section~\ref{sub:sampling-1-types}.
  Due to the domain-liftability of \fotwo{}, this approach retains its polynomial complexity w.r.t. the domain size.
  The sampling of 2-tables is handled by \code{TwoTablesSamplerForFO2}, whose complexity has been proved polynomial in the domain size, according to Lemma~\ref{le:drsamplercomplexity}.
  Therefore, the entire sampling algorithm runs in time polynomial in the domain size, and hence is domain-lifted.
\end{proof}

\section{\ctwo{} is Sampling Liftable}
\label{sec:c2_liftable}

In this section, we extend the sampling liftability of \fotwo{} to \ctwo{}, the 2-variables fragment of first-order logic with counting quantifiers $\exists_{=k}$, $\exists_{\le k}$ and $\exists_{\ge k}$.
These counting quantifiers are defined as follows.
Let $\structure$ be a structure defined on a domain $\domain$.
Then the sentence $\exists_{=k} x: \extformula(x)$ is true in $\structure$ if there are exactly $k$ distinct elements $t_1,\dots, t_k\in\domain$ such that $\structure\models \extformula(t_i)$.
For example, the sentence 
$$\forall x\forall y: (E(x,y)\Rightarrow E(y,x))\land \forall x: \neg E(x,x)\land \forall x\exists_{=2}y: E(x,y)$$ 
encodes 2-regular graphs, i.e., graphs where each vertex has exactly two neighbors.
The other two counting quantifiers can be defined: $(\exists_{\le k} x: \extformula(x))\Leftrightarrow (\forall x: \neg \extformula(x))\lor \bigvee_{i\in[k]} (\exists_{=i}x: \extformula(x))$ and $(\exists_{\ge k} x: \extformula(x))\Leftrightarrow \neg (\exists_{\le k-1} x: \extformula(x))$.
For ease of presentation, we allow the counting parameter $k=0$, and define the quantifier $\exists_{=0}$ by $(\exists_{=0}x: \extformula(x)) \Leftrightarrow (\forall x: \neg \extformula(x))$.
Note that the existential quantifiers $\exists$ can be always written as $\exists_{\ge 1}$, and thus we omit $\exists$ in the following discussion and assume that \ctwo{} is obtained by adding the counting quantifiers to \ufotwo{}.
The notation of $[\cdot]$ is extended such that $[i, j]$ denotes the set of integers $\{i, i+1, \dots, j\}$.

The sampling liftability and the lifted sampling algorithm for \ctwo{} are built upon the framework for \fotwo{} as in Section~\ref{sec:fo2_liftable}, which includes the following components.

\boldparagraph{Normal form}
We first introduce the following sentence as a normal form of \ctwo{}:
\begin{equation}
  \sentence_{\forall} \land \bigwedge_{i\in[m]}(\forall x: A_i(x) \Leftrightarrow (\exists_{=k_i} y: R_i(x,y))),
  \label{eq:ctwonromal}
\end{equation}
where $\sentence_{\forall}$ is a \ufotwo{} sentence, each $k_i$ is a non-negative integer, $R_k(x,y)$ is an atomic formula, and $A_k$ is a unary predicate.
Any \ctwo{} sentence can be converted into this normal form and the corresponding reduction is sound.
This process involves converting all $\exists_{\le k}$ and $\exists_{\ge k}$ quantifiers into $\exists_{=k}$ quantifiers according to the definition of counting quantifiers\footnote{We stress that the formula $\exists_{\ge k} x: R(x,y)$ is converted to $\neg (\bigvee_{i\in[0, k-1]} \exists_{=i} x: R(x,y))$ rather than $\bigvee_{i\in[k, n]} \exists_{=i} x: R(x,y)$, since the latter one depends on the domain size $n$.}.
Then a similar approach for converting \fotwo{} sentence into SNF in Section~\ref{sec:fo2_liftable} is used to substitute each $\exists_{=k} y: R_i(x, y)$ with an auxiliary atom $A_i(x)$.
The details can be found in \ref{sub:ctwonormalform}.
\begin{lemma}
    For any WFOMS problem $\mathfrak{S} = (\sentence, \domain, \weight, \negweight)$ where $\sentence$ is a \ctwo{} sentence, there exists a WFOMS problem $\mathfrak{S}'=(\sentence', \domain, \weight', \negweight')$ where $\sentence'$ is of the form~\eqref{eq:ctwonromal} and $\max_{i\in[m]} k_i \le |\domain|$, such that the reduction from $\mathfrak{S}$ to $\mathfrak{S}'$ is sound.
    \label{lemma:ctwo_normal}
\end{lemma}

\boldparagraph{Sampling 1-types}
The 1-type for each element $e_i$ can be sampled by the same approach as in Section~\ref{sec:fo2_liftable}.
Let $\tau_1, \tau_2, \dots, \tau_n$ be the sampled 1-types.
The predicates $A_i$ are contained in these 1-types, and thus are determined after the 1-types sampling.


\boldparagraph{A more general WFOMS problem for 2-tables sampling}
Similar to the case of \fotwo{}, we need to transform the 2-tables sampling problem into a more general form to apply the domain recursion scheme.
For each formula $\exists_{=k_j} y: R_j(x,y)$, we introduce $2(k_j + 1)$ new unary predicates $Z^\exists_{j, 0}, Z^\exists_{j, 1},\dots, Z^\exists_{j, k_j}$, $Z^\nexists_{j, 0}, Z^\nexists_{j, 1}, \dots, Z^\nexists_{j, k_j}$, and append the conjunction
\begin{equation}
  \forall x: \left(Z^\exists_{j, q}(x) \Leftrightarrow (\exists_{=q} y: R_j(x,y))\right) \land \left(Z^\nexists_{j, q}(x)\Leftrightarrow \neg (\exists_{=q} y: R_j(x,y))\right)
  \label{eq:ctwotransformed}
\end{equation}
over $q\in[0, k_j]$ to $\sentence_{\forall}$, resulting in a new sentence $\sentence_T$.
Let 
\begin{equation*}
  \mathcal{Z}_j := \{Z^\exists_{j, 0}(x),\dots, Z^\exists_{j, k_j}(x), Z^\nexists_{j, 0}(x), \dots, Z^\nexists_{j, k_j}(x), \top\},
\end{equation*}
and 
\begin{equation}
  \label{eq:ctwo2table_z}
  \mathcal{Z} := \{\{Z_1(x), Z_2(x), \dots, Z_m(x)\}\mid Z_1(x)\in \mathcal{Z}_1, \dots, Z_m(x)\in\mathcal{Z}_m\}.
\end{equation}
The more general WFOMS problem is then defined on the sentence
\begin{equation}
  \sentence_T \land \bigwedge_{i\in[n]} \tau_i(e_i) \land \bigwedge_{i\in[n]} \nu_i(e_i),
  \label{eq:ctwogeneral}
\end{equation}
where each $\nu_i(x)$ is a conjunction over a set of atomic formulas in $\mathcal{Z}$:
\begin{equation*}
  \nu_i(x) \in \left\{\bigwedge_{j\in[m]} Z_j(x)\mid \{Z_1(x), Z_2(x), \dots, Z_m(x)\}\in \mathcal{Z}\right\}.
\end{equation*}
Let $\weight(Z_{j,q}^*) = \negweight(Z_{j,q}^*) = 1$ for all $j\in[m], q\in[0, k_j]$ and $* \in \{\exists, \nexists\}$.
It is easy to check that the 2-tables sampling problem is reducible to the more general WFOMS problem, and the reduction is sound.
\begin{example}
  \label{ex:2regular_colored}
  Consider the WFOMS problem on the following sentence:
  \begin{align*}
    &\forall x\forall y: (E(x,y)\Rightarrow E(y,x))\land \forall x: \neg E(x,x)\land \\
    &\forall x: Red(x)\Leftrightarrow (\exists_{=2} y: E(x,y)).
  \end{align*}
  It encodes the colored graphs where a vertex is colored red if and only if it has exactly two neighbors.
  Suppose the domain is $\{v_1, v_2, v_3, v_4\}$, and the sampled 1-types for each element are $\tau_1 = \{Red(x)\}$, $\tau_2 = \{\neg Red(x)\}$, $\tau_3 = \{\neg Red(x)\}$, $\tau_4 = \{Red(x)\}$.
  The transformed sentence for the more general WFOMS problem is
  \begin{equation*}
    \begin{aligned}
      &\forall x\forall y: E(x,y) \Rightarrow E(y,x)\land \forall x: \neg E(x,x)\land \\
      &Red(v_1)\land \neg Red(v_2)\land Red(v_3)\land Red(v_4)\land \\
      &\forall x: (Z^\exists_{1, 0}(x) \Leftrightarrow (\exists_{=0} y: E(x,y)))\land \forall x: (Z^\nexists_{1, 0}(x) \Leftrightarrow \neg (\exists_{=0} y: E(x,y)))\land \\
      &\forall x: (Z^\exists_{1, 1}(x) \Leftrightarrow (\exists_{=1} y: E(x,y)))\land \forall x: (Z^\nexists_{1, 1}(x) \Leftrightarrow \neg (\exists_{=1} y: E(x,y)))\land \\
      &\forall x: (Z^\exists_{1, 2}(x) \Leftrightarrow (\exists_{=2} y: E(x,y)))\land \forall x: (Z^\nexists_{1, 2}(x) \Leftrightarrow \neg (\exists_{=2} y: E(x,y)))\land\\
      &Z_{1, 2}^\exists(v_1)\land Z_{1, 2}^\nexists(v_2)\land Z_{1, 2}^\nexists(v_3)\land Z_{1, 2}^\exists(v_4).
    \end{aligned}
  \end{equation*}
  The block types for each element are $\nu_1(x) = \{Z^\exists_{1, 2}(x)\}$, $\nu_2(x) = \{Z^\nexists_{1, 2}(x)\}$, $\nu_3(x) = \{Z^\nexists_{1, 2}(x)\}$, $\nu_4(x) = \{Z^\exists_{1, 2}(x)\}$.
\end{example}

\begin{figure}[!bt]
  \centerline{\includegraphics[width=.9\textwidth]{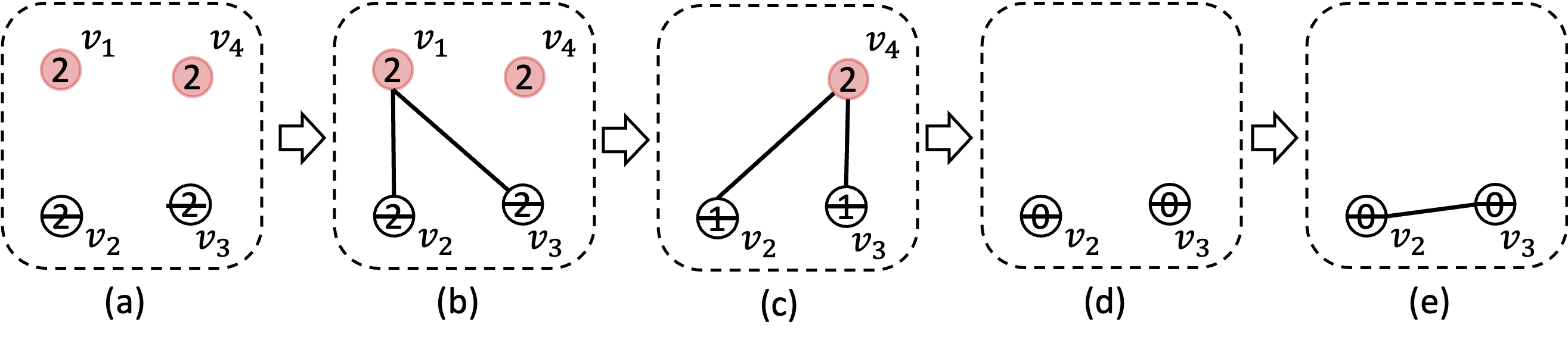}}
  \caption{
    An illustration of domain recursion for the WFOMS problem in Example~\ref{ex:2regular_colored}.
    The vertices $v_1$ and $v_4$ are colored red, while the vertices $v_2$ and $v_3$ are not colored, according to the sampled 1-types.
    The number shown in each vertex corresponds to its block type, denoting the number of neighbors that the vertex should or should not have, e.g., $v_1$ should have exactly two neighbors, while $v_2$ should not have two neighbors.
    In each domain recursion step, the block types of the selected element are relaxed according to the sampled edges.
    Note that in the final step, the edge between $v_1$ and $v_4$ is always sampled, as the block type of $v_1$ and $v_3$ requires them to have at least one neighbor.
  }
  \label{fig:regular_colored_graphs}
\end{figure}

\boldparagraph{Domain recursion}
The domain recursion scheme is still applicable to the WFOMS problem on the sentence \eqref{eq:ctwogeneral}, where we view the sets in $\mathcal{Z}(x)$ as ``block types'', and $\nu_i(x)$ as the block type of $e_i$.
When the substructure $\structure_t$ of an element $e_t$ has been sampled, the block types of the remaining elements are relaxed by the 2-tables in $\structure_t$.
This leads to a new WFOMS problem on the sentence in the same form of \eqref{eq:ctwogeneral}, but with a smaller domain.
We can show that this reduction to the new WFOMS problem is sound.
The argument is similar to what we have done for \fotwo{}, and is deferred to \ref{sec:missingproofs}.
An example of this reduction is illustrated in Figure~\ref{fig:regular_colored_graphs}.
\begin{restatable}{lemma}{ctwodomianrecursion}
  \label{lemma:ctwo-domain-recursion}
  Define the relaxed block type $\relaxcelltype{\nu}{\pi}$ of a block type $\nu$ under a 2-table $\pi$ as a set in $\mathcal{Z}$ such that for each $Z^*_{j,q}\in \nu$,
  \begin{itemize}
    \item if $*=\nexists$, $q=0$ and $R_j(y,x)\in \pi$, then $\top \in \relaxcelltype{\nu}{\pi}$;
    \item otherwise, if $R_j(y,x)\in \pi$, then\footnote{The corner case where $q=0$ and $R_j(y,x)\in\pi$, resulting in negative counting parameters, cannot occur during the domain recursion, and thus we ignore it here.} $Z^*_{j,q-1}\in \relaxcelltype{\nu}{\pi}$, and if $R_j(y,x)\notin \pi$, then $Z^*_{j,q}\in \relaxcelltype{\nu}{\pi}$.
  \end{itemize}
  For any substructure $\structure_t$ defined in~\eqref{eq:substructure}, let 
  \begin{equation}
    \label{eq:ctwo_recursive_sentence}
    \recursivesentence := \sentence_T\land \bigwedge_{i\in[n]} \tau_i(e_i)\land \bigwedge_{i\in[n]} \nu_i(e_i)\land \structure_t,
  \end{equation}
  and 
  \begin{equation}
    \label{eq:ctwo_reduced_sentence}
    \recursivesentence' := \sentence_T\land \bigwedge_{i\in[n]\setminus\{t\}} \tau_i(e_i)\land \bigwedge_{i\in[n]\setminus\{t\}} \relaxcelltypeb{\nu_i}{\pi_{t, i}}(e_i).
  \end{equation}
  If $\recursivesentence$ is satisfiable, the reduction from the WFOMS problem of $(\recursivesentence, \domain, \weight, \negweight)$ to  $(\recursivesentence', \domain\setminus \{e_t\}, \weight, \negweight)$ is sound.
\end{restatable}

\boldparagraph{Sampling algorithm}
By the domain recursion scheme, it is easy to devise a recursive algorithm for sampling 2-tables in a manner similar to Algorithm~\ref{alg:drsampler}. 
In fact, the procedure in Algorithm~\ref{alg:drsampler} remains the same except for the following modifications:
\begin{itemize}
  \item \textbf{M1}: All block types in the algorithm, including those in the cell types, are changed with $\nu(x)$,
  \item \textbf{M2}: The termination condition on block types becomes that all block types are $\{\top\}$.
  \item \textbf{M3}: The subroutine $\code{ExSat}$ now includes an additional check for the satisfaction of block types.
  Specifically, the sampled 2-tables concerning the selected element must satisfy the block type of the element, i.e., for any $Z^\exists_{j,q}(x)$ (resp. $Z^\nexists_{j,q}(x)$) in $\nu_t(x)$, there must exist exactly $q$ (resp. must not exist $q$) elements $e\in\domain\setminus\{e_t\}$ such that $R_j(e_t, e)\in \pi$.
\end{itemize}

\begin{remark}
  As mentioned at the beginning of the section, the existential quantifiers $\exists$ can be replaced by $\exists_{\ge 1}$.
  Then the auxiliary predicates $Z^\nexists_{j, 0}$ are exactly the Tseitin predicates $Z_j$ introduced in~\eqref{eq:tseitin_reduction}.
  The block types $\nu(x)$, which can only contain $\top$ and $Z^\nexists_{j,0}(x)$, degenerates to the ones we defined in Section~\ref{subsub:block_cell_types} for \fotwo{}.
  Lemma~\ref{lemma:ctwo-domain-recursion} and the sampling algorithm above is equivalent to Lemma~\ref{lemma:modular} and Algorithm~\ref{alg:drsampler} respectively.
\end{remark}

\begin{theorem}
  \label{th:ctwoliftable}
  The \ctwo{} fragment is domain-liftable under sampling.
\end{theorem}
\begin{proof}
  By Lemma~\ref{lemma:ctwo_normal}, it is sufficient to demonstrate that all \ctwo{} sentences in the normal form~\eqref{eq:ctwonromal} are sampling liftable.
  We prove it by showing that the sampling algorithm presented above is lifted.
  The sampling of 1-types is clearly polynomial-time in the domain size by the same argument as in the proof of Theorem~\ref{th:fotwoliftable} and the liftability of \ctwo{}~\cite{kuzelkaWeightedFirstorderModel2021}.
  For sampling 2-tables, we show that the modifications $\mathbf{M1}, \mathbf{M2}$ and $\mathbf{M3}$ on Algorithm~\ref{alg:drsampler} do not change the polynomial-time complexity in the domain size.
  The complexity is not affected by $\mathbf{M2}$, since the number of domain recursion steps is still the domain size.
  Furthermore, the modification $\mathbf{M3}$ only introduces a minimal overhead to each recursion step, which is not dependent on the domain size, and thus does not change the polynomial-time complexity.
  For $\mathbf{M1}$, we need some additional arguments.
  First, by the definition of block types, the number of all possible block types is $\prod_{j\in[m]} 2(k_j + 2)$, which is independent of the domain size.
  Therefore, the complexity of the main loop in Algorithm~\ref{alg:drsampler} is still polynomial-time in the domain size.
  With $\mathbf{M1}$, the algorithm now needs to solve the \wfomc{} problems on sentences in the form of \eqref{eq:ctwogeneral}.
  These counting problems can be again viewed as conditional \wfomc{} with unary evidence, whose complexity is clearly polynomial in the domain size, following from Proposition~\ref{pro:liftable_unary_evidence} and the liftability of \ctwo{}.
  As a result, the complexity of the modified algorithm is still polynomial in the domain size, and hence the \ctwo{} fragment is domain-liftable under sampling.
\end{proof}

\section{Sampling Liftability with Cardinality Constraints}
\label{sec:ccliftability}

In this section, we extend our result to the case containing the \textit{cardinality constraints}.
A single cardinality constraint is a statement of the form $|P|\bigstar q$, where $\bigstar$ is a comparison operator (e.g., $=$, $\le$, $\ge$, $<$, $>$) and $q$ is a natural number. 
These constraints are imposed on the number of distinct positive ground literals in a structure $\structure$ formed by the predicate $P$. 
For example, a structure $\structure$ satisfies the constraint $|P| \le q$ if there are at most $q$ literals for $P$ that are true in $\structure$.
For illustration, we allow cardinality constraints as atomic formulas in the FO formulas, e.g., $(|E| = 2)\land (\forall x\forall y: E(x,y) \Rightarrow E(y,x))$ (its models can be interpreted as undirected graphs with exactly one edge) and the satisfaction relation $\models$ is extended naturally.
Note the difference between cardinality constraints and counting quantifiers: the former is a statement about the number of positive ground literals for a given predicate that are true in a structure, while the latter is a statement about the number of elements in a structure that satisfy certain properties.
Another important difference is in the data complexity that we will show: the counting parameter $k_i$ of counting quantifiers is regarded as a part of the sentence and is fixed when considering the data complexity, while the cardinality constraints allow the parameter $q$ to be a part of the input instance, while still guaranteeing polynomial runtime.

\subsection{\ctwo{} with Cardinality Constraints}

We first establish the domain-liftability under sampling for the fragment \ctwo{} augmented with cardinality constraints.
Since \ctwo{} is a superset of \fotwo{} and \ufotwo{}, this result also implies the domain-liftability under sampling of \fotwo{} and \ufotwo{} with cardinality constraints.
Let $\sentence$ be a \ctwo{} sentence and 
\begin{equation}
  \Upsilon := \varphi(|P_1|\bigstar q_1,\dots,|P_M|\bigstar q_M),
  \label{eq:cc}
\end{equation}
where $\varphi$ is a Boolean formula, $\{P_i\}_{i\in[M]}\subseteq \mathcal{P}_\sentence$, and $\forall i\in[M], q_i\in \nat$.
Let us consider the WFOMS problem on $\sentence\land\Upsilon$ over the domain $\domain$ under $(\weight, \negweight)$.

The sampling algorithm for $\sentence\land\Upsilon$ keeps the same structure as those for \ufotwo, \fotwo{} and \ctwo{}, containing two successive sampling routines for 1-types and 2-tables respectively.
As usual, we only focus on the sampling of 2-tables in the following, since the process for sampling 1-types is identical
to those for \ufotwo{}, \fotwo{}, and \ctwo{}.


We first show that the domain recursive property still holds in the sampling problem of 2-tables.
Given a set $L$ of ground literals and a predicate $P$, we define $N(P, L)$ as the number of positive ground literals for $P$ in $L$.
Given a substructure $\structure_t$ of the element $e_t$, denote the 1-type of $e_t$ by $\tau_t$ as usual, let $q_i' = q_i - N(P_i, \structure_t) - N(P_i, \tau_t(e_t))$ for every $i\in[M]$, and define
\begin{equation}
  \Upsilon' = \varphi(|P_1| \bigstar q_1',\dots,|P_M|\bigstar q_M').
  \label{eq:reduced_cc}
\end{equation}
Let $\recursivesentence_{C} = \recursivesentence \land \Upsilon$ and $\recursivesentence_C' = \recursivesentence' \land \Upsilon'$, where $\recursivesentence$ and $\recursivesentence'$ are the original and reduced sentences defined as \eqref{eq:ctwo_recursive_sentence} and \eqref{eq:ctwo_reduced_sentence} respectively.
Then the reduction from the WFOMS problem on $\recursivesentence_C$ to $\recursivesentence_C'$ is sound.
The proof follows the same argument for Lemma~\ref{lemma:modular} and Lemma~\ref{lemma:ctwo-domain-recursion}, and is deferred to \ref{sec:missingproofs}.
\begin{restatable}{lemma}{ccmodular}
  \label{lemma:ccmodular}
  The reduction from the WFOMS problem of $(\recursivesentence_C, \domain, \weight, \negweight)$ to $(\recursivesentence_C', \domain', \weight, \negweight')$ is sound.
\end{restatable}

By Lemma~\ref{lemma:ccmodular}, we develop a recursive sampling algorithm for 2-tables in Algorithm~\ref{alg:ccdrsampler}.
This algorithm is derived from Algorithm~\ref{alg:drsampler} with the redundant lines not shown in the pseudocode.
The differences from the original algorithm are:
\begin{itemize}
  \item \textbf{M1}: All \wfomc{} problems now contain the cardinality constraints, e.g., $\Upsilon$ in Line~4 and $\Upsilon'$ in Line~10,
  \item \textbf{M2}: The terminal condition, which previously checked the block types, is removed\footnote{We can keep this terminal condition and invoke a more efficient sampler for \ufotwo{} with cardinality constraints, e.g., the one from our previous work \cite{wangDomainLiftedSamplingUniversal2022}. However, removing the condition does not change the polynomial-time complexity of the algorithm.}, and
  \item \textbf{M3}: The validity check for the sampled 2-tables configuration in $\code{ExSat}$ now includes an additional check for the well-definedness of the reduced cardinality constraints $\Upsilon'$, returning $\code{False}$ if there is any $q_i' \notin \nat$ for $i\in[M]$.
\end{itemize}

\begin{algorithm}[!htb]
  \caption{$\code{TwoTablesSamplerForCC}(\sentence_T, \Upsilon, \domain, \weight, \negweight, (\eta_i)_{i\in[n]})$} 
  \label{alg:ccdrsampler}
  \begin{algorithmic}[1]
    \State \dots
    \For{$\left(\vecg_{\eta^i}\right)_{i\in[N_c]}\gets\code{Prod}(\mathcal{T}_{n_{\eta^1}, N_b}, \dots, \mathcal{T}_{n_{\eta^{N_c}}, N_b})$}
      \If{$\code{ExSat}\left(\vecg, \eta_t, \Upsilon\right)$}
        \State Obtain the reduced cell configuration $(n_{\eta^i})_{i\in[N_c]}$ from $\vecg$ by \eqref{eq:configuration_reduction}
        \State Get the new cardinality constraints $\Upsilon'$ w.r.t. $\vecg$ by \eqref{eq:reduced_cc}
        \State $p(\vecg)\gets \symwfomc(\recursivesentence'\land\Upsilon', \weight, \negweight, \domain)\cdot \typeweight{\tau_t} \cdot  \prod_{i\in[N_c]} \binom{n_{\eta^i}}{\vecg_{\eta^i}}\cdot \mathbf{w}^{\vecg_{\eta^i}}$
        \State $\mathbf{G}\gets \mathbf{G}\cup \{\vecg\}$
      \EndIf
    \EndFor
    \State \dots
    \State Obtain the reduced cardinality constraints $\Upsilon'$ w.r.t. $\vecg^*$ by~\eqref{eq:reduced_cc}
    \State $\structure\gets \structure \cup \code{TwoTablesSamplerForCC}(\sentence_T, \Upsilon', \domain', \weight, \negweight, \left(\eta_i'\right)_{i\in[n-1]})$
    \State \Return $\structure$
  \end{algorithmic}
\end{algorithm}

\begin{theorem}
  Let $\sentence$ be a \ctwo{} sentence and $\Upsilon$ be of the form \eqref{eq:cc}.
  Then $\sentence\land\Upsilon$ is domain-liftable under sampling.
  \label{th:ccliftability}
\end{theorem}
\begin{proof}
  We prove the sampling liftability of \ctwo{} with cardinality constraints by showing the sampler presented above is lifted.
  The complexity of sampling 1-types is polynomial in the domain size by the same argument as in the proof of Theorem~\ref{th:ufo_liftable} and the liftability of \ctwo{} with cardinality constraints~\cite{kuzelkaWeightedFirstorderModel2021}.
  Next, we show that the modifications $\mathbf{M1}, \mathbf{M2}$ and $\mathbf{M3}$ to Algorithm~\ref{alg:drsampler} do not affect the polynomial-time complexity of the algorithm.
  First, it can be observed that $\mathbf{M3}$ has a negligible impact on the algorithm's complexity, and $\mathbf{M2}$ does not affect the polynomial-time complexity of the algorithm.
  Then, by the liftability of \ctwo{} with cardinality constraints and Proposition~\ref{pro:liftable_unary_evidence}, we have that the new \wfomc{} problems with additional cardinality constraints are still liftable.
  Therefore, the entire complexity of Algorithm~\ref{alg:ccdrsampler} remains polynomial in the domain size, and thus the sampler combining the 1-types sampling and Algorithm~\ref{alg:ccdrsampler} is lifted, which completes the proof.
\end{proof}


\subsection{A More Efficient Sampler for \sctwo{}}

With the lifted sampler for \fotwo{} with cardinality constraints, we can provide a practically more efficient sampler for some subfragment of \ctwo{}\footnote{Although the sampler for \ctwo{} proposed in Section~\ref{sec:c2_liftable} has been proved to be lifted, which means its complexity is polynomial in the domain size, the exponents of the polynomials are usually very large.}.
Specifically, we focus on the fragment \sctwo{}, which consists of sentences of the form 
\begin{align*}
  \sentence_\forall&\land (\forall x\exists_{=k_1}y: \phi_1(x,y))\land\dots\land(\forall x\exists_{=k_{m}}:\phi_{m'}(x,y))\\
  &\land (\exists_{=k_1'}x\forall y: \phi'_1(x,y))\land\dots\land(\exists_{=k_{m'}'}x\forall y:\phi'_{m'}(x,y)),
\end{align*}
where $\sentence_\forall$ is a \ufotwo{} sentence.
We call this fragment \textit{two-variable logic with counting in SNF} (\sctwo{}), as the conjunction apart from $\sentence_\forall$ resembles SNFs.
The sentence~\eqref{eq:k-regular-graph} for encoding k-regular graphs is an example of \sctwo{} sentences.

The more efficient sampler for \sctwo{} draws inspiration from the work conducted by \citet{kuzelkaWeightedFirstorderModel2021}.
The primary findings of their study were partially obtained through a reduction from the \wfomc{} problem on \sctwo{} sentences to \ufotwo{} sentences with cardinality constraints.
We demonstrate that this reduction can be also applied to the sampling problem and it is sound.
The proof follows a similar technique used in \cite{kuzelkaWeightedFirstorderModel2021}, and the details are deferred to \ref{sub:sctworeduction}.

\begin{lemma}
  For any WFOMS problem $\mathfrak{S} = (\Phi, \domain,\weight, \negweight)$ where $\Phi$ is a \sctwo{} sentence, there exists a WFOMS problem $\mathfrak{S}'=(\sentence'\land\Upsilon,\domain,\weight', \negweight')$, where $\sentence'$ is a \ufotwo{} sentence, $\Upsilon$ denotes cardinality constraints of the form \eqref{eq:cc}, such that the reduction from $\mathfrak{S}$ to $\mathfrak{S}'$ is sound.
  \label{le:ctworeduction}
\end{lemma}

Using the Lemma above and the lifted sampler in Algorithm~\ref{alg:ccdrsampler} for \ufotwo{} with cardinality constraints, it is easy to devise a lifted sampler for \sctwo{} without involving the counting quantifiers.
It is known that the counting algorithm for \ctwo{} sentences usually needs more complicated and sophisticated techniques than the one for \ufotwo{} sentences with cardinality constraints~\cite{kuzelkaWeightedFirstorderModel2021}.
Thus, the new sampler for \sctwo{} based on Lemma~\ref{le:ctworeduction} and Algorithm~\ref{alg:ccdrsampler} is more efficient and easier to implement than the one based on the sampling algorithm presented in Section~\ref{sec:c2_liftable}.
We also note that further generalizing this technique to the \ctwo{} sentences is infeasible, since the reduction from \ctwo{} to \sctwo{} for \wfomc{} used in \cite{kuzelkaWeightedFirstorderModel2021} introduced some negative weights on predicates, which would make the corresponding sampling problem ill-defined.
As a result, for WFOMS problems on general \ctwo{} sentences, one has to resort to the lifted sampling algorithm presented in Section~\ref{sec:c2_liftable}.

\section{Complexity Discussion}
\label{sec:complexity}

In this section, we delve further into the exact complexity of our sampling algorithms, and provide some insights into the factors that affect the complexity of the algorithms.
Given that our sampling algorithms heavily depend on the \wfomc{} solver, which is the primary bottleneck in the algorithms' execution, we will primarily focus on evaluating complexity in terms of the number of calls made to the \wfomc{} solver.
The domain size is denoted by $n$.

For the sampling of 1-types, all the algorithms presented in this paper employ the same technique, i.e., enumerating all possible partition configurations over the domain and sampling from them according to their weights.
Computing the weight of every partition configuration requires a call to the \wfomc{} solver (Line~2 in Algorithm~\ref{alg:1typesampler}).
There are at most $\mathcal{T}_{n, |U_{\mathcal{P}_\sentence}|}$ partition configurations, where $U_{\mathcal{P}_\sentence}$ is the set of 1-types in the sentence.
Observe that not all 1-types have to always be valid in the sentence, and the number of valid 1-types is usually much smaller than $|U_{\mathcal{P}_\sentence}|$ in practice.
For instance, in the sentence~\eqref{eq:2-color-graph} of 2-colored graphs, the number of valid 1-types is 2, i.e., red and black nodes, while the total number of 1-types is $2^3=8$.
Instead of enumerating configurations on all 1-types, the algorithm can enumerate configurations just on the valid 1-types.
Let $u$ be the number of valid 1-types in the sentence, then the enumeration complexity is bounded by $O(n^{u})$.

For the sampling of 2-tables, the complexity of the algorithms varies depending on the fragment to which the sentence belongs.
\begin{itemize}
  \item For \ufotwo{}, the sampling of 2-tables is decomposed into independent sampling problems on pairs of elements, where no \wfomc{} call is needed, and the complexity of each problem is constant in the domain size.
  \item For \fotwo{}, there are at most $n$ recursion steps.
  In each step, the algorithm needs to solve one \wfomc{} problem for each reduced cell configuration.
  The number of reduced cell configurations is bounded by $\prod_{i\in[N_c]} \mathcal{T}_{|C_{\eta^i}|, N_b}$, where $N_c$ and $N_b$ is the number of cell types and 2-tables respectively, and $C_{\eta^i}$ is the set of elements whose cell types are $\eta^i$.
  Similar to the analysis for 1-types, we consider the number of valid cells and valid 2-tables.
  Let $m$ be the number of existentially quantified formulas in the SNF sentence.
  The number of valid cells is bounded by $u2^m$.
  For 2-tables, we consider the maximum number of 2-tables coherent with any 1-types tuple and denote it by $b$.
  With these notations, the size of $C_{\eta^i}\le n$, $N_c\le u2^m$, $N_b\le b$, and thus the number $\prod_{i\in[N_c]} \mathcal{T}_{|C_{\eta^i}|, N_b}$ is in $O((n^b)^{u2^m}) = O(n^{ub2^m})$.
  The total number of \wfomc{} calls for sampling 2-tables is in $O(n^{ub2^m+1})$.
  \item For \ctwo{}, the complexity is similar to that of \fotwo{}, except that more block types are considered.
  Suppose there are $m'$ counting quantifiers in the normal form~\eqref{eq:ctwonromal} of \ctwo{}, and the counting parameters are bounded by $k$. 
  The notations of $u$ and $b$ are the same as those in \fotwo{}.
  Then the number of block types is bounded by $(2(k+1))^{m'}$, and the number of cell types is bounded by $u(2(k+1))^{m'}$.
  The total number of \wfomc{} calls for sampling 2-tables is in $O(n\cdot (n^b)^{u(2(k+1))^{m'}}) = O(n^{ub2^{m'}(k+1)^{m'}+1})$.
  \item For fragments with cardinality constraints, the cardinality constraints do not introduce additional \wfomc{} calls (see Algorithm~\ref{alg:ccdrsampler}), and thus the complexity is the same as the complexity for the corresponding fragment without cardinality constraints.
\end{itemize}

The overall complexity of the sampling algorithms for different fragments is summarized in Table~\ref{tab:complexity}.
It is important to note that the complexity provided is measured in terms of the number of \wfomc{} calls, and the actual runtime of the algorithms is influenced by the performance of the \wfomc{} solver. 

\begin{table}[!tbp]
  \centering
  \caption{The complexity of the sampling algorithms for different fragments}
  \label{tab:complexity}
  \begin{tabular}{cc}
  \hline
  Fragments (with constraints)                               & \#WFOMC                                                    \\ \hline
  $\mathbf{UFO}^2$                                           & $O(n^u)$                                              \\
  $\mathbf{FO}^2$                                            & $O(n^{ub2^m+1})$                                      \\
  $\mathbf{C}^2$                                             & $O(n^{ub2^{m'}(k+1)^{m'}+1})$                        \\
  Cardinality constraints & The same as the corresponding fragment \\ \hline
  \end{tabular}
  \end{table}

\section{Experimental Results}
\label{sec:exp}


To assess the efficacy of our sampling algorithms, we conducted a series of experiments aimed at evaluating their performance.
We also conducted statistical tests to check the correctness of our implementation.
All algorithms were implemented in Python~\footnote{The code can be found in \url{https://github.com/lucienwang1009/lifted_sampling_fo2}} and the experiments were performed on a computer with an 8-core Intel i7 3.60GHz processor and 32 GB of RAM.

Many sampling problems can be expressed as WFOMS problems.
Here we consider two typical ones:
\begin{itemize}
  \item \textbf{Sampling combinatorial structures}: The uniform generation of some combinatorial structures can be directly reduced to a WFOMS problem, e.g., the uniform generation of \textit{graphs with no isolated vertices} and \textit{$k$-regular graphs} in Section~\ref{sub:an_intuitive_example} and the introduction.
  We added four more combinatorial sampling problems to these two for evaluation: \textit{functions}, \textit{functions w/o fix-points} (i.e., the functions $f$ satisfying $f(x) \neq x$), \textit{permutations} and \textit{permutations w/o fix-points}.
  The details of these problems are described in \ref{sub:expsettings}.
  \item \textbf{Sampling from MLNs}: Our algorithms can be also applied to sample possible worlds from MLNs.
  An MLN defines a distribution over structures (i.e., possible worlds in SRL literature), and its respective sampling problem is to randomly generate possible worlds according to this distribution.
  There is a standard reduction from the sampling problem of an MLN to a WFOMS problem (see~\ref{sub:expsettings} and also \cite{wangDomainLiftedSamplingUniversal2022}).
  We use three MLNs in our experiments: 
  \begin{itemize}
    \item A variant of the classic \textit{friends-smokers} MLN with the constraint that every person has at least one friend:
    \begin{align*}
      \{&(+\infty, \forall x: \neg fr(x,x)), \\
      &(+\infty, \forall x\forall y: fr(x,y)\Rightarrow fr(y,x)), \\
      &(+\infty, \forall x\exists y: fr(x,y))\\
      &(0, sm(x)),\\
      &(0.2, fr(x,y)\land sm(x)\Rightarrow sm(y))\}.
    \end{align*}
    \item The \textit{employment} MLN used in \cite{vandenbroeck2014Proc.FourteenthInt.Conf.Princ.Knowl.Represent.Reason.}:
    \begin{equation*}
      \{(1.3, \exists y: workfor(x,y)\lor boss(x))\},
    \end{equation*}
    which states that with high probability, every person either is employed by a boss or is a boss.
    \item The \textit{deskmate} MLN:
    \begin{align*}
      \{&(+\infty, \forall x: \neg mate(x,x)\land\neg fr(x,x)), \\
      &(+\infty, \forall x\forall y: mate(x,y)\Rightarrow mate(y,x)), \\
      &(+\infty, \forall x\exists_{=1}y: mate(x,y)), \\
      &(+\infty, \forall y\exists_{=1}x: mate(x,y)), \\
      &(1.0, mate(x,y)\Rightarrow fr(x,y))\},
    \end{align*}
    which states that every student has exactly one deskmate, and if two students are deskmates, then they are probably friends.
  \end{itemize}
  The details about the reduction from sampling from MLNs to WFOMS problems and the resulting WFOMS problems of these two MLNs can be found in \ref{sub:expsettings}.
\end{itemize}

\subsection{Correctness of the Implementation}

We first performed a statistical test on our implementation by focusing on the uniform generation of combinatorial structures within small domains, where exact sampling is feasible via enumeration-based techniques; 
we choose the domain size of $5$ for evaluation.
To serve as a benchmark, we have implemented a simple ideal uniform sampler, denoted by IS, by enumerating all the models and then drawing samples uniformly from these models; this is also why we use such a small domain consisting only of five elements in this experiment.
For each combinatorial structure encoded into an \fotwo{} sentence $\sentence$, a total of $100 \times |\fomodels{\sentence}{\domain}|$ models were generated from both IS and our weighted model sampler.
Figure~\ref{fig:uniformity} depicts the model distribution produced by these two algorithms---the horizontal axis represents models numbered lexicographically, while the vertical axis represents the generated frequencies of models.
The figure suggests that the distribution generated by our weighted model sampler is indistinguishable from that of IS.
Furthermore, a statistical test on the distributions produced by the weighted model sampler was performed, and no statistically significant difference from the uniform distribution was found.
The details of this test can be found in \ref{sub:expresults}.

\begin{figure}[!tb]
  \centerline{\includegraphics[width=.8\textwidth]{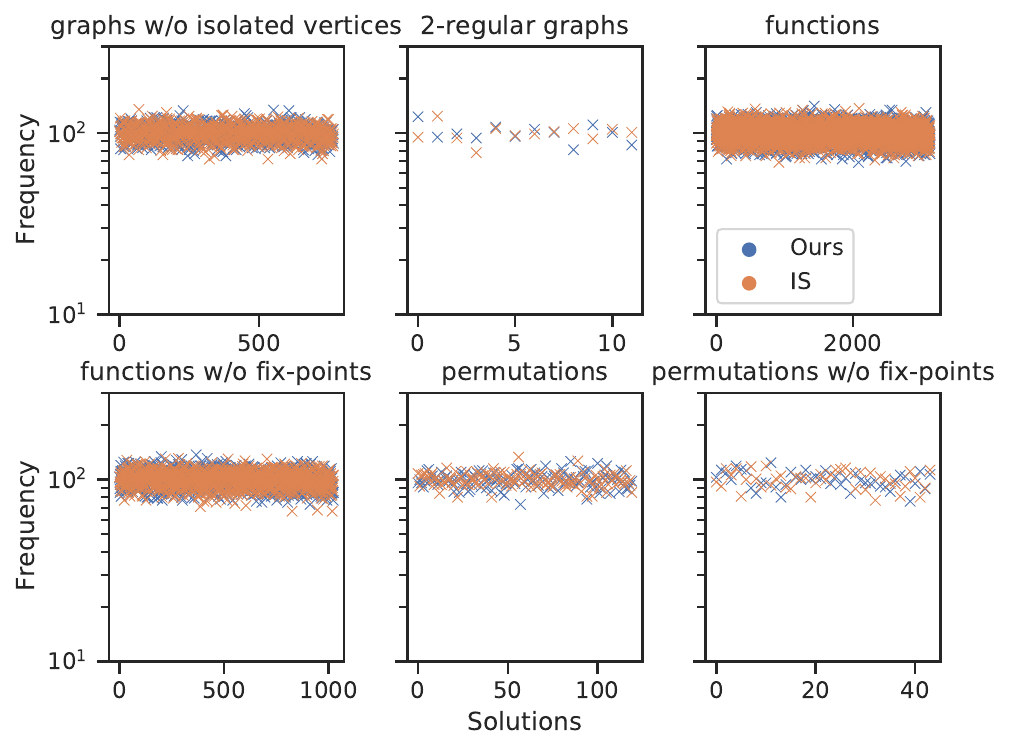}}
  \caption{Uniformity comparison between an ideal sampler (IS) and our weighted model sampler.}
  \label{fig:uniformity}
\end{figure}

\begin{figure}[!tb]
  \centering
  \begin{subfigure}[b]{0.7\textwidth}
      \centerline{
        \includegraphics[width=\textwidth]{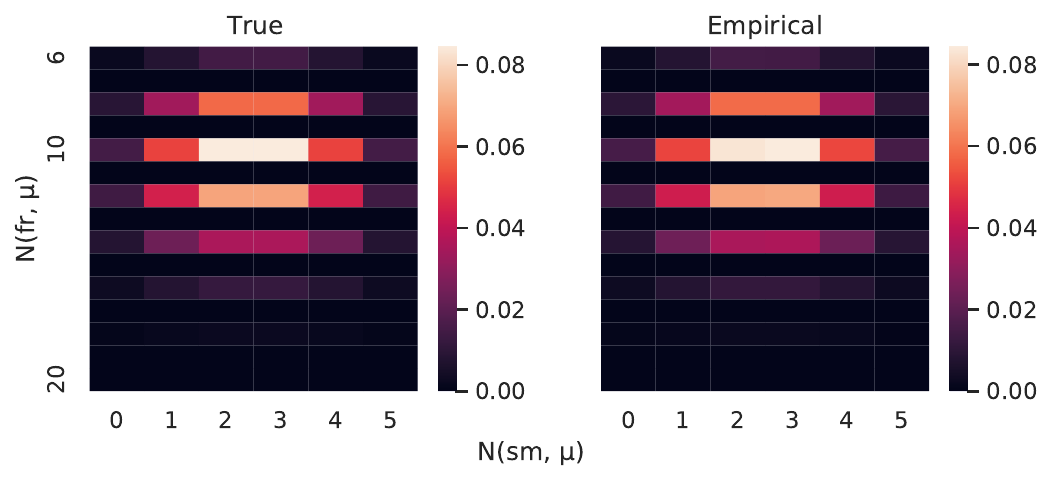}
      }
      \caption{friends-smokers}
      \label{fig:fr-sm}
  \end{subfigure}
  \begin{subfigure}[b]{0.7\textwidth}
      \centerline{
        \includegraphics[width=\textwidth]{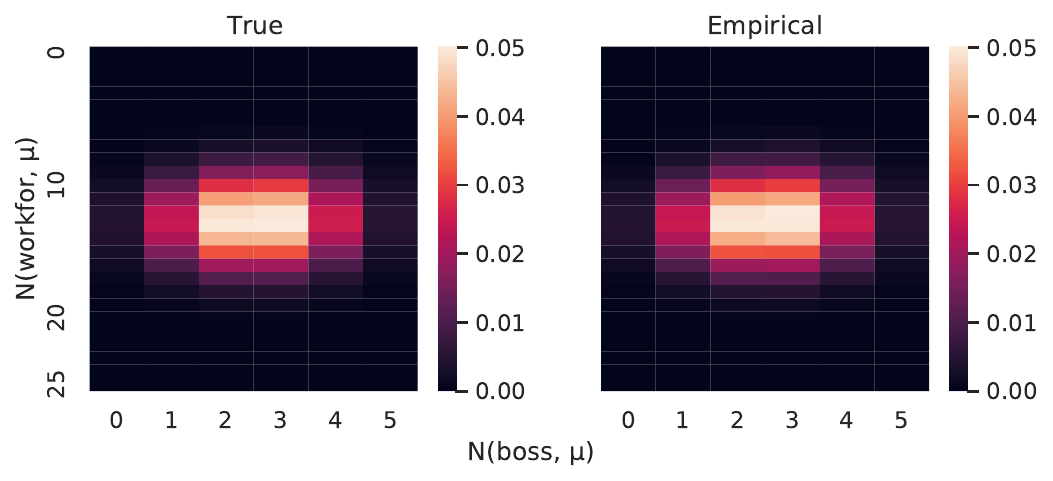}
      }
      \caption{employment}
      \label{fig:employ}
  \end{subfigure}
  \begin{subfigure}[b]{0.7\textwidth}
      \centerline{
        \includegraphics[width=\textwidth]{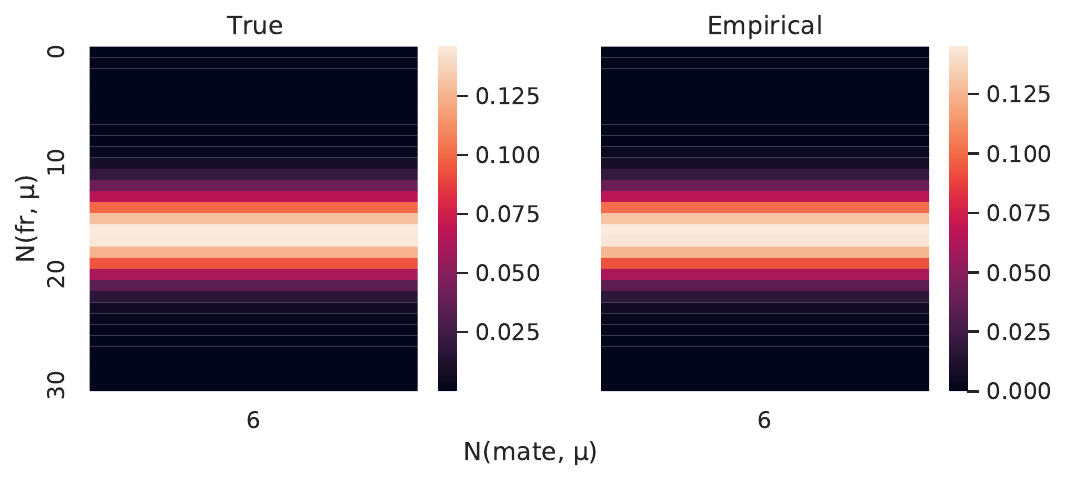}
      }
      \caption{deskmate}
      \label{fig:deskmate}
  \end{subfigure}
  \caption{Conformity testing for the count distribution of MLNs.}
  \label{fig:mlnconformity}
\end{figure}

For sampling problems from MLNs, enumerating all the models is infeasible even for a domain of size $5$, e.g., there are $2^{5^2 + 5} = 2^{30}$ models in the employment MLN.
That is why we test the \textit{count distribution} of predicates from the problems.
Instead of specifying the probability of each model, the count distribution only tells us the probability that a certain number of predicates are interpreted to be true in the models.
An advantage of testing count distributions is that they can be efficiently computed for our MLNs.
Please refer to~\cite{kuzelkaWeightedFirstorderModel2021} for more details about count distributions.
We also note that the conformity of count distribution is a necessary condition for the correctness of algorithms.
We kept the domain size to $5$ for friends-smokers and employment MLNs and set it to $6$ for the deskmate MLN (due to the counting quantifiers in the MLN).
We would like to emphasize that the choice of small domains was not made for scalability purposes, but rather to facilitate the statistical tests.
In the following section, we will demonstrate that our approach can indeed scale to larger domains.
A total number of $10^5$ models were generated from the weighted model sampler for each MLN.
The empirical distributions of count-statistics, along with the true count distributions, are shown in Figure~\ref{fig:mlnconformity}.
It is easy to check the conformity of the empirical distribution to the true one from the figure.
The statistical test was also performed on the count distribution, and the results confirm the conclusion drawn from the figure (also see \ref{sub:expresults}).

\subsection{Performance}

\begin{figure}[!tb]
  \centering
  \includegraphics[width=.95\textwidth]{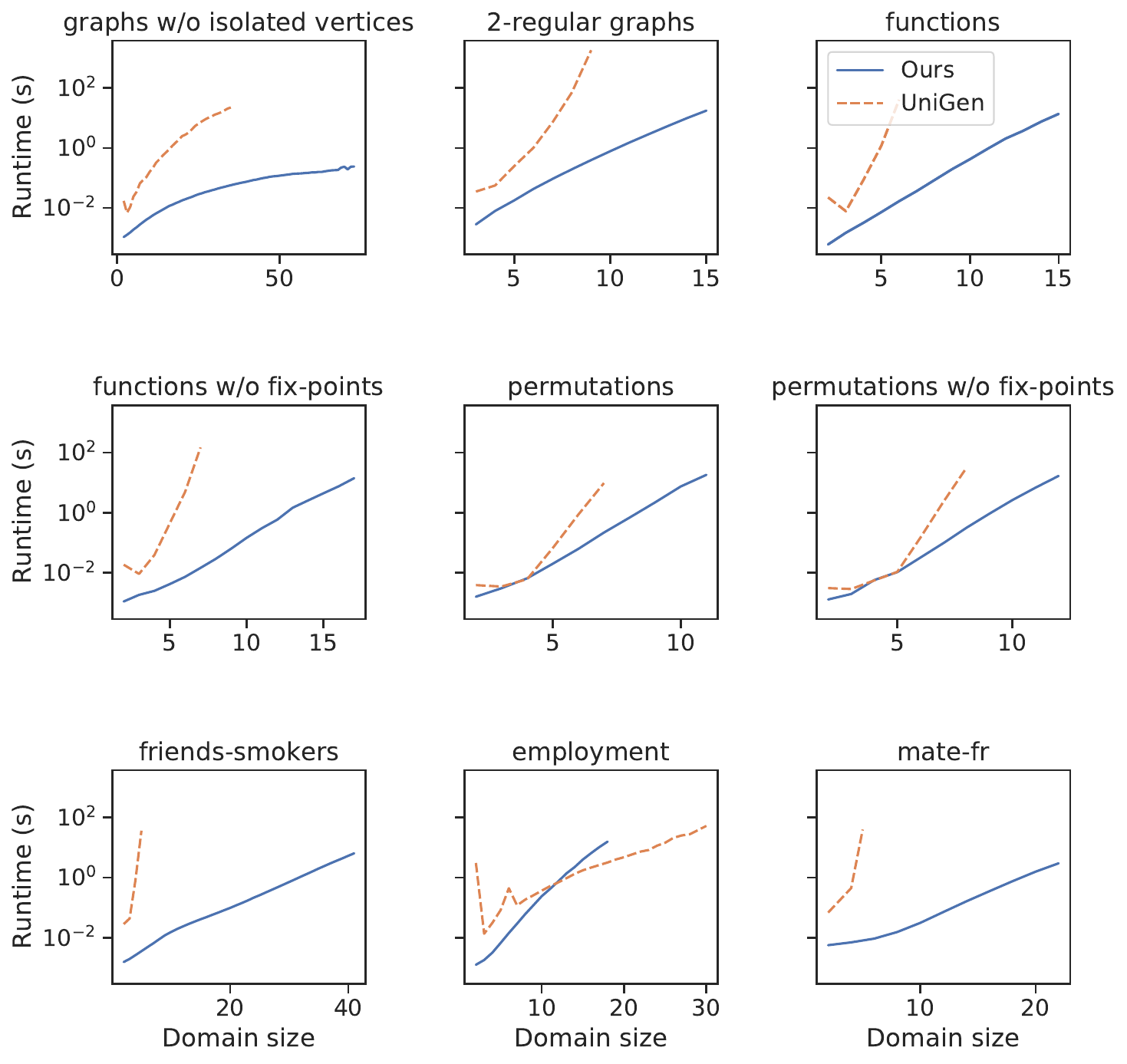}
  \caption{Performance of weighted model sampler versus UniGen.}
  \label{fig:performance}
\end{figure}

To evaluate the performance, we compare our weighted model samplers with Unigen\footnote{\url{https://github.com/meelgroup/unigen}}~\cite{chakrabortyScalableNearlyUniform2013,soosTintedDetachedLazy2020}, the state-of-the-art approximate sampler for Boolean formulas.
Note that there is no guarantee of polynomial complexity for Unigen, and its runtime highly depends on the underlying SAT solver it uses.
We reduce the WFOMS problems to the sampling problems of Boolean formulas by grounding the input first-order sentence over the given domain.
Since Unigen only works for uniform sampling, we employ the technique in~\cite{chakrabortyWeightedUnweightedModel2015} to encode the weighting function in the WFOMS problem into a Boolean formula.

For each sampling problem, we randomly generated $1000$ models by our weighted model sampler and Unigen respectively, and computed the average sampling time per one model.
The performance comparison is shown in Figure~\ref{fig:performance}.
In most cases, our approach is much faster than UniGen.
The exception in the employment MLN, where UniGen performed better than the weighted model sampler, is likely due to the simplicity of this specific instance (almost every structure is a model of this MLN).
This coincides with the theoretical result that our weighted model sampler is polynomial-time in the domain size, while UniGen usually needs a significant number of expensive SAT calls on the grounded formulas.

The full \ctwo{} sampler proposed in Section~\ref{sec:c2_liftable} was not implemented in our study due to the extremely slow computational speed of the existing \wfomc{} solver~\cite{kuzelkaWeightedFirstorderModel2021} for full \ctwo{}, despite its polynomial complexity. 
Developing faster \wfomc{} algorithms is orthogonal to the present work where we only use the existing solver as a black box.

\section{Related Work}
\label{sec:related_work}

The studies on model counting and sampling problems date back to the 1970s~\cite{valiant1979SIAMJ.Comput.,valiant1979Theor.Comput.Sci.,jerrumRandomGenerationCombinatorial1986}.
In the seminal paper by \citet{jerrumRandomGenerationCombinatorial1986}, a significant connection was established between the random generation of combinatorial structures (i.e., model sampling) and the problem of model counting. 
Specifically, this paper showed that the random generation of combinatorial structures can be reduced in polynomial time to the model counting problem, under the condition of \textit{self-reducibility}.
Self-reducibility refers to the property where the solution set for a given instance of a problem can be expressed in terms of the solution sets of smaller instances of the same problem. 
In the context of propositional logic, self-reducibility naturally holds as the solution set of a Boolean formula $\sentence$ can be expressed in terms of the solution sets of $\sentence_{a}$ and $\sentence_{\neg a}$, where $a$ is an arbitrary atom in $\sentence$, and $\sentence_{a}$ and $\sentence_{\neg a}$ are the formulas obtained from $\sentence$ by setting $a$ to be true and false respectively.
This property allows for a polynomial-time reduction from WMS to WMC, where the truth value of each atom in $\sentence$ is incrementally sampled based on the ratio of the WMC of $\sentence_a$ and $\sentence_{\neg a}$.
However, in first-order logic, self-reducibility is not generally guaranteed. 
Conditioning on a ground atom in a first-order sentence may make the resulting problem intractable, even if the original problem is tractable.
For example, \citet{vandenbroeckConditioningFirstorderKnowledge2012} have proven that there exists a \ufotwo{} sentence with \textit{binary} ground atoms, whose model count cannot be computed in time polynomial in the size of the binary atoms, unless \P=\NP.
This result implies that the reduction derived from self-reducibility is not applicable to WFOMS, motivating this paper to develop new techniques for sampling from first-order formulas.


The approach taken in this paper, as well as the formal \textit{liftability} notions considered here, are inspired by the \textit{lifted inference} literature~\cite{liftedinferencebook}. 
In lifted inference, the goal is to perform probabilistic inference in statistical-relational models in a way that exploits symmetries in the high-level structure of the model. 
Models that are amenable to scalable inference as domain size (or population size) increases are dubbed \textit{domain-liftable}, in a similar spirit to the notion of \textit{domain-liftability under sampling} presented here. 
There exists a very extensive literature on lifted inference, which has been explored from both graphical models and logic perspectives. 
In the graphical models literature, the focus is on leveraging the symmetries of random variables and factor graphs. 
This involves grouping symmetric variables and factors into their parameterized counterparts, leading to the construction of a relational graphical model on which lifted inference can be performed~\cite{poole2003first,richardsonMarkovLogicNetworks2006,DBLP:conf/kr/KazemiBKNP14}.
The objective of these approaches is typically to extend the inference algorithms used in traditional graphical models to relational models. 
Examples of such algorithms include lifted variable elimination~\cite{poole2003first,DBLP:conf/ijcai/BrazAR05,DBLP:conf/aaai/MilchZKHK08,DBLP:journals/jair/TaghipourFDB13}, lifted belief propagation~\cite{singla2008AAAI,kerstingCountingBeliefPropagation2009}, and lifted junction tree~\cite{DBLP:conf/ki/BraunM16}.
Another strand of research in the lifted inference models the relations between random variables using propositional or first-order logic, and then performs inference by exploiting the symmetries in the logical structure of the model~\cite{vandenbroeckLiftedProbabilisticInference2011,vandenbroeckCompletenessFirstorderKnowledge2011,vandenbroeck2014Proc.FourteenthInt.Conf.Princ.Knowl.Represent.Reason.,DBLP:conf/ijcai/BremenK20,gribkoffUnderstandingComplexityLifted2014}.
In these studies, the \wfomc{} problem serves as an ``assembly language'' for the probabilistic inference problem, and the goal is to develop efficient algorithms for solving the \wfomc{} problem.
Of particular interest is the work by~\citet[Appendix~C]{beameSymmetricWeightedFirstOrder2015}, which studies the data complexity of \wfomc{} of \ufotwo{}. 
The general argument used there---namely, the analysis of a two-variable sentence in terms of its cell types---forms a basis for our sampling approaches discussed in the paper.
Finally, though not directly related to our work, we note that the lifting notion, which groups symmetric components in a problem and performs computation on a higher level, has also been explored in other areas, such as probabilistic theorem proving~\cite{DBLP:conf/uai/GogateD11a} as well as linear and convex quadratic programming~\cite{DBLP:journals/jmlr/MladenovAK12,DBLP:conf/aaai/MladenovKK17}.


The domain recursion scheme, another important approach adopted in this paper, is similar to the \textit{domain recursion rule} used in weighted first-order model counting~\cite{broeckCompletenessFirstOrderKnowledge2011,kazemi2016new,kazemiDomainRecursionLifted2017,toth2022ArXivPrepr.ArXiv221101164}.
The domain recursion rule for \wfomc{} is a technique that uses a gradual grounding process on the input first-order sentence, where only one element of the domain is grounded at a time. 
As each element is grounded, the partially grounded sentence is simplified until the element is entirely removed, resulting in a new \wfomc{} problem with a smaller domain. 
With the domain recursion rule, one can apply the principle of induction on the domain size, and solve the counting problem by dynamic programming.
A closely related work to this paper is the approach presented by \citet{kazemi2016new}, where they used the domain recursion rule to solve the \wfomc{} problem on \fotwo{} sentences without using Skolemization~\cite{vandenbroeck2014Proc.FourteenthInt.Conf.Princ.Knowl.Represent.Reason.}, which introduces negative weights.
However, it is important to note that their approach can be only applied to some specific first-order formulas, whereas the domain recursion scheme presented in this paper, mainly designed for eliminating the existentially-quantified formulas, supports the entire \ctwo{} fragment with cardinality constraints.

It is also worth mentioning that weighted model sampling is a relatively well-studied area~\cite{gomesModelCounting2009,chakrabortyScalableNearlyUniform2013,chakrabortyParallelScalableUniform2015}.
However, many real-world problems can be represented more naturally and concisely in first-order logic, and suffer from a significant increase in formula size when grounded out to propositional logic. 
For example, a formula of the form $\forall x\exists y: \extformula$ is encoded as a Boolean formula of the form $\bigwedge_{i=1}^n\bigvee_{j=1}^n l_{i,j}$, whose length is quadratic in the domain size $n$. 
Since even finding a solution to a such large ground formula is challenging, most sampling approaches for propositional logic instead focus on designing approximate samplers. 
We also note that these approaches are not polynomial-time in the length of the input formula, and rely on access to an efficient SAT solver.
An alternative strand of research~\cite{guo2019uniform,he2021perfect,feng2021sampling} on combinatorial sampling, focuses on the development of near-uniform and efficient sampling algorithms. 
However, these approaches can only be employed for specific Boolean formulas that satisfy a particular technical requirement known as the Lovász Local Lemma. 
The WFOMS problems studied in this paper do not typically meet the requisite criteria for the application of these techniques.

\section{Conclusion and Future Work}
\label{sec:conclusion}

In this paper, we prove the domain-liftability under sampling for the \ctwo{} fragment.
The result is further extended to the fragment of \ctwo{} with the presence of cardinality constraints.
The widespread applicability of WFOMS renders the proposed approach a promising candidate to serve as a universal paradigm for a plethora of sampling problems.

A potential avenue for further research is to expand the methodology presented in this paper to encompass more expressive first-order languages. 
Specifically, the domain recursion scheme employed in this paper could be extended to other tractable fragments, as its analogous technique, the domain recursion rule, has been demonstrated to be effective in proving the domain-liftability of the fragments $\mathbf{S}^2\fotwo{}$ and $\mathbf{S}^2\mathbf{RU}$ for \wfomc{}~\cite{kazemiDomainRecursionLifted2017}.

In addition to extending the input logic, other potential directions for future research include incorporating elementary axioms, such as tree axiom~\cite{vanbremenLiftedInferenceTree2021} and linear order axiom~\cite{toth2022ArXivPrepr.ArXiv221101164}, as well as more general weighting functions that involve negative weights. 
However, it is important to note that these extensions would likely require a more advanced and nuanced approach than the one proposed in this paper, and may present significant challenges.

Finally, the lower complexity bound of WFOMS is also an interesting open problem.
We have discussed in the introduction that it is unlikely for an (even approximate) lifted weighted model sampler to exist for full first-order logic. 
However, the establishment of a tighter lower bound for fragments of FO, such as $\mathbf{FO}^3$, remains an unexplored and challenging area that merits further investigation.

\section*{Acknowledgement}

Yuanhong Wang and Juhua Pu are supported by the National Key R\&D Program of China (2021YFB2104800) and the National Science Foundation of China (62177002). Ond\v{r}ej Ku\v{z}elka's work is supported by the Czech Science Foundation project 20-19104Y and partially also 23-07299S (most of the work was done before the start of the latter project).

\clearpage

\appendix


\section{\wfomc{} with Unary Evidence}
\label{sec:wfomc-unary}

In this section, we show how to deal with unary evidence in conditional \wfomc{}.

\citet{vandenbroeckConditioningFirstorderKnowledge2012} handled the unary evidence by the following transformation of the input sentence $\sentence$.
For any unary predicate $P$ appearing in the evidence, they split the domain into $\domain_\top^P$, $\domain_\bot^P$ and $\domain_\emptyset^P$, where $\domain_\top^P$ and $\domain_\bot^P$ contains precisely the elements with evidence $P(a)$ and $\neg P(a)$, respectively, and $\domain_\emptyset^P = \domain \setminus(\domain_\top^P\cup \domain_\bot^P)$ is the remaining elements.
Then the sentence $\sentence$ was transformed into
\begin{align*}
  \left(\forall x\in\domain_\top^P: \sentence_\top^P\right)\land \left(\forall x\in\domain_\bot^P: \sentence_\bot^P\right)\land \left(\forall x\in\domain_\emptyset^P: \sentence\right),
\end{align*}
where $\sentence_\top^P$ and $\sentence_\bot^P$ were obtained from $\sentence$ by replacing all occurrences of $P$ with $\mathsf{True}$ and $\mathsf{False}$, respectively, and $\forall x\in \domain'$ was a domain constraint that restricts the quantifier to the domain $\domain'$.
The procedure could be repeated to support multiple unary predicates, and the resulting sentence was then compiled into an FO d-DNNF circuits~\cite{vandenbroeckLiftedProbabilisticInference2011} for model counting.
The domain constraints are natively supported by FO d-DNNF circuits, and the compilation and model counting have been shown to be in time polynomial in the domain size.

However, not all \wfomc{} algorithms can effectively support the domain constraints, and efficiently count the model of the transformed sentence.
In the following, we provide a simpler approach to deal with the unary evidence, without the need for domain constraints.

We first introduce the notion of evidence type.
An \textit{evidence type} $\sigma$ over a predicate vocabulary $\mathcal{P}$ is a consistent set of 1-literals formed by $\mathcal{P}$.
For instance, both $\{P(x), \neg Q(x)\}$ and $\{P(x)\}$ are evidence types over $\{P/1, Q/1\}$.
The evidence type can be also viewed as a conjunction of its elements, where $\sigma(x)$ denotes a quantifier-free formula.
If the evidence type $\sigma$ is an empty set, then $\sigma(x)$ is defined as $\top$.
The number of evidence types over $\mathcal{P}$ is finite, and independent of the size of the domain.
Given a set $L$ of ground 1-literals, the evidence type of an element is defined as the set of all literals in $L$ that are associated with the element.
For example, if $L = \{P(a), \neg Q(a), R(b)\}$ and $\domain = \{a, b, c\}$, then the evidence type of $a$ is $\{P(x), \neg Q(x)\}$, the evidence type of $b$ is $\{R(x)\}$, and the evidence type of $c$ is $\emptyset$.

\liftableunaryevidence*
\begin{proof}
  Let $\sigma^1, \sigma^2,\dots, \sigma^m$ be the distinct evidence types of elements given by $L$.
  The number of elements with evidence type $\sigma^i$ is denoted by $n_i$.
  We first transform the input sentence into
  \begin{equation*}
    \sentence' := \sentence\land \bigwedge_{i\in[m]} (\forall x: \xi^i(x)\Rightarrow \sigma^i(x))\land \Lambda,
  \end{equation*}
  where each $\xi^i/1$ is a fresh predicate for $\sigma^i$ with the weight $\weight(\xi^i) = \negweight(\xi^i) = 1$, and
  \begin{equation}
    \Lambda = \left(\forall x: \bigvee_{i\in[m]} \xi^i(x)\right)\land \left(\forall x: \bigwedge_{i,j\in[m]: i\neq j} (\neg \xi^i(x) \lor \neg \xi^j(x))\right).
    \label{eq:wfomc-evidence-exactlyone}
  \end{equation}
  The formula $\Lambda$ states that
  each element $a$ has exactly one $\xi^i(a)$ that is \textsf{True}.
  In other words, the interpretation of predicates $\xi^i$ can be seen as a partition of the domain, where each disjoint subset in the partition contains precisely the elements with evidence type $\sigma^i$.
  Then we have 
  \begin{equation}
    \begin{aligned}
      \symwfomc(&\sentence\land\bigwedge_{l\in L}l, \domain, \weight, \negweight) = \\
      &\frac{1}{\binom{n}{n_1, n_2, \dots, n_m}}\cdot \symwfomc(\sentence'\land \bigwedge_{i\in[m]} (|\xi^i| = n_i), \domain, \weight, \negweight),
    \end{aligned}
    \label{eq:wfomc-evidence-cc}
  \end{equation}
  where $|\xi^i| = n_i$ is a cardinality constraint that restricts the number of \textsf{True} $\xi^i$ in the model to be $n_i$.

  The reasoning is as follows.
  Denote by $k_1, k_2, \dots, k_n$ the indices of evidence type of each element in the domain, i.e., $\sigma^{k_i}$ is the evidence type of $i$-th element in the domain, and $\xi^{k_i}$ is its corresponding predicate.
  We first show that
  \begin{equation}
    \symwfomc(\sentence\land\bigwedge_{l\in L}l, \domain, \weight, \negweight) = \symwfomc(\sentence'\land \bigwedge_{i\in[n]} \xi^{k_i}(e_i), \domain, \weight, \negweight),
    \label{eq:wfomc-evidence-cc-2}
  \end{equation}
  where $e_i$ is the $i$-th element in the domain.
  The proof is built on the fact that the reduct mapping $\proj{\cdot}{\mathcal{P}_\sentence}$ is a bijective function from the models of $\Psi' = \sentence'\land \bigwedge_{i\in[n]} \xi^{k_i}(e_i)$ to the models of $\Psi=\sentence\land\bigwedge_{l\in L}l$.
  First, it is easy to show that the mapping $\proj{\cdot}{\mathcal{P}_\sentence}$ is a function from the models of $\Psi'$ to the models of $\Psi$, since for any model $\mu'$ of $\Psi'$, $\proj{\mu'}{\mathcal{P}_\sentence}$ is a model of $\Psi$.
  Second, the mapping $\proj{\cdot}{\mathcal{P}_\sentence}$ is injective.
  This is because, for any two models $\mu_1'$ and $\mu_2'$ of $\Psi'$, they must share the same interpretation of the predicates $\xi^i$ due to the constraint $\Lambda$ and the evidence $\xi^{k_i}(e_i)$; and if $\proj{\mu_1'}{\mathcal{P}_\sentence} = \proj{\mu_2'}{\mathcal{P}_\sentence}$, it follows that these two models also have the same interpretation of the predicates in $\mathcal{P}_\sentence$, and thus $\mu_1' = \mu_2'$.
  Finally, the mapping is also surjective: for any model $\mu$ of $\Psi$, we can always construct a model $\mu'$ of $\Psi'$ by expanding $\mu$ with the evidence $\xi^{k_i}(e_i)$ for each $i\in[n]$, such that $\proj{\mu'}{\mathcal{P}_\sentence} = \mu$.
  Thus, we have \eqref{eq:wfomc-evidence-cc-2}.

  Next, let us consider the two \wfomc{} problems of $\sentence'\land \bigwedge_{i\in[m]} (|\xi^i| = n_i)$ and $\sentence'\land \bigwedge_{i\in[n]} \xi^{k_i}(e_i)$.
  Observe that any interpretation of $\xi^i$ that satisfies $\sentence'$ can be viewed as an order-dependent partition of the domain with the configuration of the partition being $(|\xi^1|, |\xi^2|, \dots, |\xi^m|)$.
  Thus, $\symwfomc(\sentence'\land \bigwedge_{i\in[m]} (|\xi^i| = n_i), \domain, \weight, \negweight)$ can be written as the summation of $\symwfomc(\sentence'\land \bigwedge_{i\in[n]} \xi^{k_i'}(e_i), \domain, \weight, \negweight)$ over all possible interpretations $\xi^{k_i'}(e_i)$, whose corresponding partition configuration is $(n_1, \dots, n_m)$.
  There are totally $\binom{n}{n_1, n_2, \dots, n_m}$ such interpretations, and the interpretation $\xi^{k_i}(e_i)$ is one of them.
  Moreover, due to the symmetry of \wfomc{}, all these interpretations result in the same result of \wfomc{}.
  Thus, we can write
  \begin{equation}
    \begin{aligned}
      \symwfomc(&\sentence'\land \bigwedge_{i\in[m]} (|\xi^i| = n_i), \domain, \weight, \negweight) \\
      &= \binom{n}{n_1, n_2, \dots, n_m}\cdot \symwfomc(\sentence'\land \bigwedge_{i\in[n]} \xi^{k_i}(e_i), \domain, \weight, \negweight).
    \end{aligned}
    \label{eq:wfomc-evidence-cc-3}
  \end{equation}
  Combining \eqref{eq:wfomc-evidence-cc-2} and \eqref{eq:wfomc-evidence-cc-3} yields \eqref{eq:wfomc-evidence-cc}.
  Finally, computing \eqref{eq:wfomc-evidence-cc} has been shown to be in time polynomial in the domain size in \cite{vandenbroeckLiftedProbabilisticInference2011} for any domain-liftable sentence $\sentence$, which completes the proof.
\end{proof}

\section{Normal Forms}

\subsection{Scott Normal Form}
\label{sub:snf}

We briefly describe the transformation of \fotwo{} formulas to SNF and prove the soundness of its corresponding reduction on the WFOMS problems.
The process is well-known, so we only sketch the related details.
Interested readers can refer to~\citet{kuusistoWeightedModelCounting2018} for more details.

Let $\sentence$ be a sentence of \fotwo{}.
To put it into SNF, consider a subformula $\varphi(x) = Qy: \phi(x, y)$, where $Q\in\{\forall, \exists\}$ and $\phi$ is quantifier-free.
Let $A_\varphi$ be a fresh unary predicate\footnote{If $\varphi(x)$ has no free variables, e.g., $\exists x: \phi(x)$, the predicate $A_\varphi$ is nullary.} and consider the sentence
\begin{equation*}
  \forall x: (A_\varphi(x) \Leftrightarrow (Qy: \phi(x,y)))
\end{equation*}
which states that $\varphi(x)$ is equivalent to $A_\varphi(x)$.
Let $Q'$ denote the dual of $Q$, i.e., $Q' \in \{\forall,\exists\}\setminus \{Q\}$, this sentence can be seen equivalent to
\begin{align*}
  \sentence' := &\forall x Qy: (A_\varphi(x) \Rightarrow \phi(x, y)) \\
  &\land \forall x Q'y: (\phi(x,y) \Rightarrow A_\varphi(x)).
\end{align*}

Let 
\begin{equation*}
  \sentence'' := \sentence' \land \sentence[\varphi(x)/A_\varphi(x)],
\end{equation*} where $\sentence[\varphi(x)/A_\varphi(x)]$ is obtained from $\sentence$ by replacing $\varphi(x)$ with $A_\varphi(x)$.
For any domain $\domain$, every model of $\sentence''$ over $\domain$ can be mapped to a unique model of $\sentence$ over $\domain$.
The bijective mapping function is simply the projection $\proj{\cdot}{\mathcal{P}_\sentence}$.
Let both the positive and negative weights of $A_\varphi$ be $1$ and denote the new weighting functions as $\weight'$ and $\negweight'$.
It is clear that the reduction from $(\sentence, \domain, \weight, \negweight)$ to $(\sentence'', \domain, \weight', \negweight')$ is sound.
Repeat this process from the atomic level and work upwards until the sentence is in SNF.
The whole reduction remains sound due to the transitivity of soundness.

\subsection{Normal Form for \ctwo{}}
\label{sub:ctwonormalform}

We show that any \ctwo{} sentence can be converted into the normal form:
\begin{equation*}
  \sentence_{\forall} \land \bigwedge_{i\in[m]}(\forall x: A_i(x) \Leftrightarrow (\exists_{=k_i} y: R_i(x,y))),
\end{equation*}
where $\sentence_{\forall}$ is a \fotwo{} sentence, each $k_i$ is a non-negative integer, $R_i(x,y)$ is an atomic formula, and $A_i$ is an unary predicate.
The process is as follows:
\begin{enumerate}
  \item Convert each counting-quantified formula of the form $\exists_{\ge k} y: \phi(x,y)$ to $\neg (\exists_{\le k-1} y: \phi(x,y))$.
  \item Decompose each $\exists_{\le k} y: \phi(x,y)$ into $\bigvee_{i\in[0,k]} (\exists_{=i} y: \phi(x,y))$.
  \item Replace each subformula $\exists_{=k} y: \phi(x,y)$, where $k > |\domain|$, with $\code{False}$.
  \item Starting from the atomic level and working upwards, replace any subformula $\exists_{=k} \phi(x,y)$, where $\phi(x,y)$ is a formula that does not contain any counting quantifier, with $A(x)$; and append $\forall x\forall y: R(x,y) \Leftrightarrow \phi(x,y)$ and $\forall x: A(x)\Leftrightarrow (\exists_{=k}y: R(x,y))$, where $R$ is an auxiliary binary predicate, to the original sentence.
\end{enumerate}
It is easy to check that the reduction presented above is sound and independent of the domain size if the domain size is greater than the maximum counting parameter $k$ in the input sentence.\footnote{This condition does not change the data complexity of the problem, as parameters $k$ of the counting quantifiers are considered constants but not the input of the problem.}

\section{A Sound Reduction from \sctwo{} to \fotwo{} with Cardinality Constraints}
\label{sub:sctworeduction}

In this section, we show the sound reduction from a WFOMS problem on \sctwo{} sentence to a WFOMS problem on \fotwo{} sentence with cardinality constraints.

We first need the following two lemmas, which are based on the transformations from~\cite{kuzelkaWeightedFirstorderModel2021}.

\begin{lemma}
  Let $\sentence$ be a first-order logic sentence, and let $\domain$ be a domain.
  Let $\Phi$ be a first-order sentence with cardinality constraints, defined as follows:
  \begin{align*}
    \Pi := &(|P|=k\cdot |\domain|)\\
    &\land (\forall x\forall y: P(x,y)\Leftrightarrow (R_1^P(x,y)\lor \dots\lor R_k^P(x,y)))\\
    &\land \bigwedge_{i\in[k]}(\forall x\exists y: R_i^P(x,y))\\
    &\land \bigwedge_{i,j\in[k]: i\neq j} (\forall x\forall y: \neg R_i^P(x,y)\lor \neg R_j^P(x,y)),
  \end{align*}
  where $R_i^P$ are auxiliary predicates not in $\mathcal{P}_\sentence$ with weight $\weight(R_i^P) = \negweight(R_i^P) = 1$.
  Then the reduction from the WFOMS $(\sentence\land \forall x\exists_{=k} y: P(x,y), \domain, \weight, \negweight)$ to $(\sentence\land\Pi, \domain, \weight, \negweight)$ is sound.
  \label{le:forallexistsred}
\end{lemma}
\begin{proof}
  Let $f(\cdot) = \proj{\cdot}{\mathcal{P}_\sentence\cup\{P\}}$ be a mapping function.
  We first show that $f$ is from $\fomodels{\sentence\land\Pi}{\domain}$ to $\fomodels{\sentence\land\forall x\exists_{=k}: P(x,y)}{\domain}$: if $\structure\models \sentence\land \Pi$ then $f(\structure)\models\sentence\land\forall x\exists_{=k}y: P(x,y)$.
  
  The sentence $\Pi$ means that for every $c_1, c_2\in\domain$ such that $P(c_1,c_2)$ is true, there is exactly one $i\in[k]$ such that $R_i^P(c_1, c_2)$ is true.
  Thus we have that $\sum_{i\in[k]} |R_i^P| = |P| = k\cdot|\domain|$, which together with $\bigwedge_{i\in[k]}\forall x\exists y: R_i^P(x,y)$ implies that $|R_i^P| = k$ for $i\in[k]$.
  We argue that each $R_i^P$ is a function predicate in the sense that $\forall x\exists_{=1}y: R_i^P(x,y)$ holds in any model of $\sentence\land\Pi$.
  Let us suppose, for contradiction, that $(\forall x\exists y: R_i^P(x,y))\land (|R_i^P| = k)$ holds but there is some $a\in\domain$ such that $R_i^P(a, b)$ and $R_i^P(a,b')$ are true for some $b \neq b'\in\domain$.
  We have $|\{(x,y)\in\domain^2\mid R_i^P(x,y)\land x\neq a\}| \ge |\domain| - 1$ by the fact $\forall x\exists y: R_i^P(x,y)$. 
  It follows that $|R_i^P| \ge |\{(x,y)\in\domain^2\mid R_i^P(x,y)\land x\neq a\}| + 2 > |\domain|$, which leads to a contradiction.
  Since all of $R_i^P$ are function predicates, it is easy to check $\forall x\exists_{=k}y: P(x,y)$ must be true in any model $\mu$ of $\sentence\land\Pi$, i.e., $f(\mu)\models \sentence\land\forall x\exists_{=k}y: P(x,y)$.

  To finish the proof, one can easily show that, for every model $\mu\in\fomodels{\sentence\land\forall x\exists_{=k}y: P(x,y)}{\domain}$, there are exactly $(k!)^{|\domain|}$ models $\mu'\in\fomodels{\sentence\land\Pi}{\domain}$ such that $f(\mu') = \mu$.
  The reason for this is that 1) if, for any $a\in\domain$, we permute $b_1, b_2, \dots, b_k$ in $R_1^P(a, b_1), R_2^P(a, b_2), \dots, R_k^P(a, b_k)$ in the model $\mu'$, we get another model of $\sentence\land\Pi$, and 2) up to these permutations, the predicates $R_i^P$ in $\mu'$ are determined uniquely by $\mu$.
  Finally, the weights of all these $\mu'$ are the same as those of $\mu$, and we can write
  \begin{align*}
    \sum_{\substack{\mu'\in\fomodels{\sentence\land\Pi}{\domain}:\\f(\mu')=\mu}} \pro[\mu'\mid\sentence\land\Pi] &= \frac{\sum_{\substack{\mu'\in\fomodels{\sentence\land\Pi}{\domain}:\\f(\mu')=\mu}} \typeweight{\mu'}}{\symwfomc(\sentence\land\Pi,\domain,\weight, \negweight)}\\
    & =\frac{(k!)^{|\domain|}\cdot \typeweight{\mu}}{(k!)^{|\domain|}\cdot \symwfomc(\sentence\land\forall x\exists_{=k}y: P(x,y), \domain, \weight, \negweight)}\\
    & =\frac{\typeweight{\mu}}{\symwfomc(\sentence\land\forall x\exists_{=k}y: P(x,y), \domain, \weight, \negweight)}\\
    & =\pro[\mu\mid\sentence\land\forall x\exists_{=k}y: P(x,y)],
  \end{align*}
  which completes the proof.
\end{proof}

\begin{lemma}
  Let $\sentence$ be a first-order logic sentence, $\domain$ be a domain, and $P$ be a predicate.
  Then the WFOMS $(\sentence\land\forall_{=k}\forall y: P(x,y), \domain, \weight, \negweight)$ can be reduced to $(\sentence\land(|U| = k)\land (\forall x: U(x)\Leftrightarrow (\forall y: P(x,y))), \domain, \weight, \negweight)$, where $U$ is an auxiliary unary predicate with weight $\weight(U) = \negweight(U) = 1$, and the reduction is sound.
  \label{le:existsforallred}
\end{lemma}
\begin{proof}
  The proof is straightforward.
\end{proof}

\begin{proof}[Proof of Lemma~\ref{le:ctworeduction}]
  We can first get rid of all formulas of the form $\exists_{=k}x \forall y: P(x,y)$ by repeatedly using Lemma~\ref{le:existsforallred}.
  Then we can use Lemma~\ref{le:forallexistsred} repeatedly to eliminate the formulas of the form $\forall x\exists_{=k} y: P(x,y)$.
  The whole reduction is sound due to the transitivity of soundness.
\end{proof}

\section{Missing Details of Weighted Model Sampler}

\subsection{Optimizations for Weighted Model Sampler}
\label{sub:optwms}

There exist several optimizations to make the sampling algorithm more practical. 
Here, we present some of them that are used in our implementation.
\paragraph{Heuristic element selection} 
The complexity of $\code{TwoTablesSamplerForFO2}$ heavily depends on the recursion depth.
In our implementation, when selecting a domain element $e_t$ for sampling its substructure, we always chose the element with the ``strongest'' existential constraint that contains the most Tseitin atoms $Z_k(x)$.
It would help $\code{TwoTablesSamplerForFO2}$ reach the condition that the existential constraint for all elements is $\top$.
In this case, $\code{TwoTablesSamplerForFO2}$ will invoke the more efficient sampler for \ufotwo{} to sample the remaining substructures.

\paragraph{Further decomposition of 2-tables sampling} 
Let $\mathcal{P}_\exists$ be the union of predicates vocabularies of the existentially quantified formulas
\begin{equation*}
    \mathcal{P}_\exists := \{R_1, R_2, \dots, R_m\}. 
\end{equation*}
We further decomposed the sampling probability of 2-tables $\pro[\structure\mid\sentence_T\land \bigwedge_{i\in[n]}\eta_i(e_i)]$ into
\begin{equation*}
  \pro\left[\structure\mid\sentence_T\land\bigwedge_{i\in[n]}\eta_i(e_i)\land\structure_\exists\right]\cdot \pro\left[\structure_\exists\mid\sentence_T\land \bigwedge_{i\in[n]}\eta_i(e_i)\right],
\end{equation*}
where $\structure_\exists$ is a $\mathcal{P}_\exists$-structure over $\domain$.
It decomposes the conditional sampling problem of $\structure$ into two subproblems---one is to sample $\structure_\exists$ and the other to sample the remaining substructures conditional on $\structure_\exists$.
The subproblem of sampling $\structure$ conditioned on $\structure_\exists$ can be reduced into a sampling problem on a \ufotwo{} sentence with evidence
\begin{equation}
  \label{eq:exsatreduction}
  \sentence' = \forall x\forall y: \fotwoformula(x,y)\land \bigwedge_{i\in[m]}\tau_i(e_i)\land \structure_\exists,
\end{equation}
since all existentially-quantified formulas have been satisfied with $\structure_\exists$.
Recall that for a \ufotwo{} sentence, when the 1-types $\tau_i$ of all elements has been fixed, the sampling problem of 2-tables can be decomposed into multiple independent sampling problems of the 2-table for each elements tuple (see Section~\ref{sub:sampling-2-tables}).
Sampling 2-tables from~\eqref{eq:exsatreduction} can be solved in a similar way, by decomposing the sampling probability into
\begin{equation*}
  \pro\left[\bigwedge_{i,j\in[n]: i<j}\pi_{i,j}(e_i, e_j)\mid \sentence'\right] = \prod_{i,j\in[n]: i<j} \pro\left[\pi_{i,j}(e_i, e_j)\mid \fotwoformula'(e_i, e_j)\right],
\end{equation*}
where $\fotwoformula'(e_i, e_j)$ is the simplified formula of $\fotwoformula(e_i, e_j)\land \fotwoformula(e_j, e_i)$ obtained by replacing the ground 1-literals with their truth values by the 1-types $\tau_i(e_i)$ and $\tau_j(e_j)$, and the sampled structure $\structure_\exists$.
The 2-tables $\pi_{i,j}$ as well as the simplified formulas $\fotwoformula'(e_i, e_j)$ are independent of each other, and thus can be efficiently sampled in parallel.
The other subproblem of sampling $\structure_\exists$ can be solved by an algorithm similar to $\code{TwoTablesSamplerForFO2}$.
In this algorithm, the 2-tables used to partition cells are now defined over $\mathcal{P}_\exists$, whose size is only $4^m$.
Compared to the number $4^B$ of 2-tables in the original algorithm $\code{TwoTablesSampler}$, where $B\ge m$ is the number of binary predicates in $\sentence$, the number of 2-tables in the new algorithm can be exponentially smaller.
As a result, the for-loop that enumerates all 2-tables configurations in Algorithm~\ref{alg:drsampler} is exponentially faster.

\paragraph{Cache} We cached the weight $\typeweight{\tau^i}$ of all 1-types and the weight $\typeweight{\pi^i}$ of all 2-tables, which are often used in our sampler. 

\subsection{The Submodule $\code{ExSat}(\cdot, \cdot)$}
\label{sub:exsat}

The pseudo-code of $\code{ExSat}$ is presented in Algorithm~\ref{alg:exsat}.

\begin{algorithm}[H] 
  \caption{ExSat($\vecg, \eta$)} 
  \label{alg:exsat}
  \begin{algorithmic}[1]
  \State Decompose $\vecg$ into $\{g_{\eta^i,\pi^j}\}_{i\in[N_c],j\in[N_b]}$
  \State $(\beta, \tau)\gets (\eta)$
  \State // Check the coherence of 2-tables
  \For{$i\in[N_u]$}
    \For{$j\in[N_b]$}
        \State // $\tau(\eta^i)$ is the 1-type in $\eta^i$
        \If{$\pi^j$ is not coherent with $\tau(\eta^i)$ and $\tau$ and $g_{\eta^i, \pi^j} > 0$}
            \State \Return \code{False}
        \EndIf
    \EndFor
  \EndFor
  \State // Check the satisfaction of existentially-quantified formulas
  \State $\forall j\in[N_b], h_{\pi^j}\gets\sum_{i\in[N_c]} g_{\eta^i, \pi^j}$
  \For{$Z_k(x)\in\beta$}
    \For{$j\in[N_b]$}
        \If{$R_k(x,y)\in\pi^j$ and $h_{\pi^j} > 0$}
            \State \Return \code{True}
        \EndIf
    \EndFor
  \EndFor
  \State \Return \code{False}
  \end{algorithmic}
\end{algorithm}

\section{Missing Details of Experiments}

\subsection{Experiment Settings}
\label{sub:expsettings}

\paragraph{Sampling Combinatorial Structures}

The corresponding WFOMS problems for the uniform generation of combinatorial structures used in our experiments are presented as follows. 
The weighting functions $\weight$ and $\negweight$ map all predicates to $1$.
\begin{itemize}
  \item Functions: 
  \begin{equation*}
    \forall x\exists_{=1} y: f(x,y).
  \end{equation*}
  \item Functions w/o fix points:
  \begin{equation*}
    (\forall x\exists_{=1}y: f(x,y))\land (\forall x: \neg f(x,x)).
  \end{equation*}
  \item Permutations:
  \begin{equation*}
    (\forall x\exists_{=1}y: Per(x,y)) \land (\forall y\exists_{=1}x: Per(x,y)).
  \end{equation*}
  \item Permutation without fix-points:
  \begin{equation*}
    (\forall x\exists_{=1}y: Per(x,y)) \land (\forall y\exists_{=1}x: Per(x,y))\land (\forall x: \neg Per(x,x)).
  \end{equation*}
\end{itemize}

\paragraph{Sampling from MLNs}

An MLN is a finite set of weighted first-order formulas $\{(\weight_i, \formula_i)\}_{i\in[m]}$, where each $\weight_i$ is either a real-valued weight or $\infty$, and $\formula_i$ is a first-order formula.
Let $\mathcal{P}$ be the vocabulary of $\formula_1,\formula_2,\dots,\formula_m$.
An MLN $\Phi$ paired with a domain $\domain$ induces a probability distribution over $\mathcal{P}$-structures (also called possible worlds):
\begin{equation*}
  p_{\mln, \domain}(\world) := \begin{cases}
    \frac{1}{Z_{\mln,\domain}}\exp\left(\sum_{(\formula, \weight)\in\mln_\real}\weight\cdot \#(\formula, \world)\right) & \textrm{if } \world\models \mln_\infty\\
    0 & \textrm{otherwise}
  \end{cases}
\end{equation*}
where $\mln_\real$ and $\mln_\infty$ are the real-valued and $\infty$-valued formulas in $\mln$ respectively, and $\#(\formula, \world)$ is the number of groundings of $\formula$ satisfied in $\world$.
The sampling problem on an MLN $\mln$ over a domain $\domain$ is to randomly generate a possible world $\world$ according to the probability $p_{\mln, \domain}(\world)$.

The reduction from the sampling problems on MLNs to WFOMS can be performed following the same idea as in~\cite{vandenbroeckLiftedProbabilisticInference2011}.
For every real-valued formula $(\formula_i, \weight_i)\in\mln_\real$, where the free variables in $\formula_i$ are $\vecx$, we introduce a novel auxiliary predicate $\xi_i$ and create a new formula $\forall \vecx: \xi_i(\vecx)\Leftrightarrow \formula_i(\vecx)$.
For formula $\formula_i$ with infinity weight, we instead create a new formula $\forall \vecx: \formula_i(\vecx)$.
Denote the conjunction of the resulting set of sentences by $\sentence$, and set the weighting function to be $\weight(\xi_i)=\exp(\weight_i)$ and $\negweight(\xi_i) = 1$, and for all other predicates, we set both $\weight$ and $\negweight$ to be $1$.
Then the sampling problem on $\mln$ over $\domain$ is reduced to the WFOMS $(\sentence, \domain, \weight, \negweight)$.

By the reduction above, we can write the two MLNs used in our experiments to WFOMS problems.
The weights of predicates are all set to be $1$ unless otherwise specified.
\begin{itemize}
  \item Friends-smokers MLN: the reduced sentence is
  \begin{align*}
    &\forall x: \neg fr(x,x)\land\\
    & \forall x\forall y: fr(x,y)\Leftrightarrow fr(y,x)\land \\
    &\forall x\exists y: fr(x,y) \land\\
    &\forall x\forall y: (\xi_1(x)\Leftrightarrow sm(x))\land \\
    &\forall x\forall y: \xi_2(x,y)\Leftrightarrow (fr(x,y)\land sm(x)\Rightarrow sm(y)),
  \end{align*}
  and the weight of $\xi_1$ and $\xi_2$ is set to be $\weight(\xi) = \exp(0), \weight(\xi_2) = \exp(0.2)$.
  \item Employment MLN: the corresponding sentence is
  \begin{align*}
    \forall x: \xi(x)\Leftrightarrow (\exists y: workfor(x,y)\lor boss(x)),
  \end{align*}
  and the weight of $\xi$ is set to be $\exp(1.3)$.
  \item Deskmate MLN: the corresponding sentence is
  \begin{align*}
    &\forall x: \neg mate(x,x)\land \neg fr(x,x)\land \\
    &\forall x\forall y: mate(x,y)\Rightarrow mate(y,x)\land \\
    &\forall x\exists_{=1} y: mate(x,y)\land \\
    &\forall y\exists_{=1} x: mate(x,y)\land\\
    &\forall x: \xi(x,y) \Leftrightarrow (mate(x,y)\Rightarrow fr(x,y)),
  \end{align*}
  and the weight of $\xi$ is set to be $\exp(1.0)$.
\end{itemize}

\subsection{More Experimental Results}
\label{sub:expresults}


We use the Kolmogorov-Smirnov (KS) test~\cite{massey1951kolmogorov} to validate the conformity of the (count) distributions produced by our algorithm to the reference distributions.
The KS test used here is based on the multivariate Dvoretzky–Kiefer–Wolfowitz (DKW) inequality recently proved by \cite{naaman2021tight}.

Let $\mathbf{X}_1 = (X_{1i})_{i\in[k]}, \mathbf{X}_2 = (X_{2i})_{i\in[k]},\dots,\mathbf{X}_n = (X_{ni})_{i\in[k]}$ be $n$ real-valued independent and identical distributed multivariate random variables with cumulative distribution function (CDF) $F(\cdot)$.
Let $F_n(\cdot)$ be the associated empirical distribution function defined by
\begin{equation*}
  F_n(\vecx) := \frac{1}{n}\sum_{i\in[n]} \indicator_{X_{i1}\le x_1,X_{i2}\le x_2,\dots, X_{ik}\le x_k}, \qquad \vecx\in\real^k.
\end{equation*}
The DKW inequality states
\begin{equation}
  \pro\left[\sup_{\vecx\in\real^k}|F_n(\vecx) - F(\vecx)| > \epsilon\right] \ge (n+1)ke^{-2n\epsilon^2}
  \label{eq:dkw}
\end{equation}
for every $\epsilon, n, k > 0$.
When the random variables are univariate, i.e., $k=1$, we can replace $(n+1)k$ in the above probability bound by a tighter constant $2$.

\begin{table}[!htb]
  \caption{The Kolmogorov-Smirnov Test}
  \label{ta:kstest}
  \centering
  \begin{tabular}{|c|c|c|}
    \hline
    Problem                      & Maximum deviation & Upper bound \\ \hline \hline
    graphs w/o isolated vertices & 0.0036            & 0.0049      \\ \hline
    2-regular graphs             & 0.0065            & 0.0069      \\ \hline
    functions                    & 0.0013            & 0.0024       \\ \hline
    functions w/o fix-points     & 0.0027            & 0.0042      \\ \hline
    permutations                 & 0.0071            & 0.0124      \\ \hline
    permutations w/o fix-points  & 0.019            & 0.02      \\ \hline
    friends-smokers              & 0.0021            & 0.0087      \\ \hline
    employment                  & 0.0030            & 0.0087      \\ \hline
    deskmate                  & 0.0022            & 0.0087      \\ \hline
  \end{tabular}
\end{table}

In the KS test, the null hypothesize is that the samples $\mathbf{X}_1, \mathbf{X}_2,\dots,\mathbf{X}_n$ are distributed according to some reference distribution, whose CDF is $F(\cdot)$.
By \eqref{eq:dkw}, with probability $1-\alpha$, the maximum deviation $\sup_{\vecx\in\real^k}|F_n(\vecx) - F(\vecx)|$ between empirical and reference distributions is bounded by $\epsilon = \sqrt{\frac{\ln(k(n+1)/\alpha)}{2n}}$ ($\frac{\sqrt{\ln(2 / \alpha)}}{2n}$ for the univariate case).
If the actual value of the maximum deviation is larger than $\epsilon$, we can reject the null hypothesis at the confidence level $\alpha$.
Otherwise, we cannot reject the null hypothesis, i.e., the empirical distribution of the samples is not statistically different from the reference one.
In our experiments, we choose $\alpha = 0.05$ as a significant level.

For the uniform generation of combinatorial structures, we assigned each model a lexicographical number and treated the model indices as a multivariate random variable with a discrete multi-dimensional uniform distribution\footnote{
There is another setting where the model indices are treated as multiple i.i.d. random variables, and the hypotheses testing is performed on each individual random variable.
In this setting, the significant level should be adjusted to $\alpha = 0.05/|\mathcal{M}|$, where $|\mathcal{M}|$ is the number of total models, according to the Bonferroni correction.
Therefore, the setting we use in the experiments is actually more strict.}.
For the sampling problems of MLNs, we test their count distributions against the true count distributions.
Table~\ref{ta:kstest} shows the maximum deviation between the empirical and reference cumulative distribution functions, along with the upper bound set by the DKW inequality. 
As shown in Table~\ref{ta:kstest}, all maximum deviations are within their respective upper bounds.
Therefore, we cannot reject any null hypotheses, i.e., there is no statistically significant difference between the two sets of distributions.

\section{Missing Proofs}
\label{sec:missingproofs}

{\renewcommand\footnote[1]{}\ctwodomianrecursion*}
\begin{proof}
  Following a similar argument as in the proof of Lemma~\ref{lemma:modular}, the soundness of the reduction can be proved by showing that the function $f(\mu') = \mu'\cup \structure_t\cup \tau_t(e_t)$ is bijective. 
  Let $\Lambda$ be the conjunction of
  \begin{equation*}
    \forall x: \left(Z^\exists_{j, q}(x) \Leftrightarrow (\exists_{=q} y: R_j(x,y))\right) \land \left(Z^\nexists_{j, q}(x)\Leftrightarrow \neg (\exists_{=q} y: R_j(x,y))\right)
  \end{equation*}
  over all $j\in[m]$ and $q\in[0, k_j]$.
  We write the sentence $\recursivesentence$ and $\recursivesentence'$ as
  \begin{align*}
    \recursivesentence &= \sentence_\forall \land \bigwedge_{i\in[n]}\left(\tau_i(e_i) \land \nu_i(e_i)\right) \land \structure_t\land \Lambda\\
    \recursivesentence' &= \sentence_\forall \land \bigwedge_{i\in[n]\setminus\{t\}}\left(\tau_i(e_i) \land \relaxcelltypeb{\nu_i}{\pi_{t,i}}(e_i)\right) \land \Lambda.
  \end{align*}
  
  \paragraph*{($\Rightarrow$)}
  For any model $\mu'$ in $\fomodels{\recursivesentence'}{\domain'}$, we prove that $f(\mu') = \mu'\cup \structure_t \cup \tau_t(e_t)$ is a model in $\fomodels{\recursivesentence}{\domain}$.
  First, one can easily check that $f(\mu')$ satisfies $\sentence_\forall \land \bigwedge_{i\in[n]}\tau_i(e_i)$ by the satisfiability of $\recursivesentence$.
  Next, we demonstrate that $f(\mu')$ also satisfies $\bigwedge_{i\in[n]}\nu(e_i) \land \Lambda$.
  For any index $i\in[n]\setminus\{t\}$, any $j\in[n], q\in[0, k_j]$ such that $Z_{j,q}^\exists(x)\in \nu_i$, it can be easily shown that there exist exact $q$ ground atoms $R_j(e_i, e)$ in $f(\mu')$ by the definition of $\relaxcelltype{\nu_i}{\pi_{t,i}}$.
  Thus, the counting quantified formula $\exists_{=q} y: R_j(e_i, y)$ is true in $f(\mu')$.
  Then let us consider the case of $\neg (\exists_{q}y: R_j(e_i, y))$.
  For any $j\in[n], q\in[0, k_j]$ such that $Z_{j,q}^\nexists(x)\in \nu_i$, if $q>0$ or $R_j(e_i, e_t)\notin \structure_t$, there must not exist $q$ ground atoms $R_j(e_i, e)$ in $f(\mu')$ following the similar argument as above.
  If $q=0$ and $R_j(e_i, e_t)\in \structure_t$, the counting quantified formula $\neg (\exists_{=0} R_j(e_i, y))$ is true in $\structure_t$, and thus also satisfied in $f(\mu')$.
  Finally, by the satisfability of $\recursivesentence$, we have that $\nu_t(e_t)\land \Lambda$ is true in $\structure_t$.
  Therefore, we can conclude that $f(\mu')$ satisfies $\bigwedge_{i\in[n]}\nu(e_i) \land \Lambda$, which together with the satisfaction of $\sentence_\forall \land \bigwedge_{i\in[n]}\tau_i(e_i)$, implies that $f(\mu')$ is a model in $\fomodels{\recursivesentence}{\domain}$.
  
  \paragraph*{($\Leftarrow$)} 
  For any model $\mu$ in $\fomodels{\recursivesentence}{\domain}$, we prove that $\mu' = \mu \setminus (\structure_t \cup \tau_t(e_t))$ is the unique model in $\fomodels{\recursivesentence'}{\domain'}$ such that $f(\mu') = \mu$.
  The uniqueness of $\mu'$ is clear from the definition of $f$.
  Since $\mu$ is a model of $\recursivesentence$, we first have that $\mu'$ satisfies $\sentence_\forall \land \bigwedge_{i\in[n]\setminus\{t\}}\tau_i(e_i)$.
  Then we show that $\mu'$ also satisfies $\bigwedge_{i\in[n]\setminus\{t\}}\relaxcelltypeb{\nu_i}{\pi_{t, i}}(e_i)\land \Lambda$.
  For any $i\in[n]$, by the definition of the block type, it is easy to check that there are exact $q$ ground atoms $R_j(e_i, e)$ in $\mu$ for all $j\in[m], q\in[0, k_j]$ such that $Z^\exists_{j, q}(x)\in \nu_i$.
  Therefore, if $R_j(y,x)\in \pi_{t,i}$, we have that there exist exact $q-1$ ground atoms $R_j(e_i, e)$ in $\mu'$.
  Otherwise, the number of ground atoms $R_j(e_i, e)$ in $\mu'$ remains $q$.
  It follows that $\exists_{=s} y: R_j(e_i, y)$ is true in $\mu'$ for all $j\in[m], s\in[0, k_j]$ such that $Z_{j,s}^\exists(x)\in \relaxcelltype{\nu_i}{\pi_{t,i}}$.
  The argument for the case of $\neg (\exists_{=s} y: R_j(e_i, y))$ is similar.
  Hence, we have that $\mu'$ satisfies $\bigwedge_{i\in[n]\setminus\{t\}}\relaxcelltypeb{\nu_i}{\pi_{t, i}}(e_i)\land \Lambda$, which combined with the satisfaction of $\forall x\forall y: \fotwoformula(x,y)\land \bigwedge_{i\in[n]\setminus\{t\}}\tau_i(e_i)$ leads to conclusion.

  By the bijectivity of the mapping function $f$, it is easy to prove the consistency of sampling probability through the reduction.
  The argument is the same as in the proof of Lemma~\ref{lemma:modular}.
\end{proof}

\ccmodular*
\begin{proof}
  The proof follows the same argument for Lemma~\ref{lemma:ctwo-domain-recursion}.
  The only statement that needs to be argued is the bijectivity of the mapping function $f(\mu') = \mu' \cup \structure_t \cup \tau_t(e_t)$.
  For any model $\mu'$ in $\fomodels{\recursivesentence'}{\domain'}$, $\mu'$ satisfies both $\recursivesentence'$ and $\Upsilon'$.
  It follows that $f(\mu')$ satisfies both $\recursivesentence$ and $\Upsilon$.
  For any model $\mu$ in $\fomodels{\recursivesentence}{\domain}$, the reasoning is similar, and we have that $f^{-1}(\mu) = \mu \setminus (\structure_t \cup \tau_t(e_t))$ is a unique model in $\fomodels{\recursivesentence'}{\domain'}$.
  This establishes the bijectivity of $f$.
  The remainder of the proof, including the consistency of sampling probability, proceeds exactly the same as Lemma~\ref{lemma:modular}.
\end{proof}

\clearpage
 \bibliographystyle{plainnat} 
 \bibliography{cas-refs}

\begin{thebibliography}{58}
\providecommand{\natexlab}[1]{#1}
\providecommand{\url}[1]{\texttt{#1}}
\expandafter\ifx\csname urlstyle\endcsname\relax
  \providecommand{\doi}[1]{doi: #1}\else
  \providecommand{\doi}{doi: \begingroup \urlstyle{rm}\Url}\fi

\bibitem[Beame et~al.(2015)Beame, Van~den Broeck, Gribkoff, and
  Suciu]{beameSymmetricWeightedFirstOrder2015}
Paul Beame, Guy Van~den Broeck, Eric Gribkoff, and Dan Suciu.
\newblock Symmetric weighted first-order model counting.
\newblock In \emph{Proceedings of the 34th ACM SIGMOD-SIGACT-SIGAI Symposium on
  Principles of Database Systems}, pages 313--328, 2015.

\bibitem[Braun and M{\"{o}}ller(2016)]{DBLP:conf/ki/BraunM16}
Tanya Braun and Ralf M{\"{o}}ller.
\newblock Lifted junction tree algorithm.
\newblock In Gerhard Friedrich, Malte Helmert, and Franz Wotawa, editors,
  \emph{{KI} 2016: Advances in Artificial Intelligence - 39th Annual German
  Conference on AI, Klagenfurt, Austria, September 26-30, 2016, Proceedings},
  volume 9904 of \emph{Lecture Notes in Computer Science}, pages 30--42.
  Springer, 2016.
\newblock \doi{10.1007/978-3-319-46073-4\_3}.

\bibitem[Chakraborty et~al.(2013)Chakraborty, Meel, and
  Vardi]{chakrabortyScalableNearlyUniform2013}
Supratik Chakraborty, Kuldeep~S Meel, and Moshe~Y Vardi.
\newblock A scalable and nearly uniform generator of sat witnesses.
\newblock In \emph{International Conference on Computer Aided Verification},
  pages 608--623. Springer, 2013.

\bibitem[Chakraborty et~al.(2015{\natexlab{a}})Chakraborty, Fremont, Meel,
  Seshia, and Vardi]{chakrabortyParallelScalableUniform2015}
Supratik Chakraborty, Daniel~J Fremont, Kuldeep~S Meel, Sanjit~A Seshia, and
  Moshe~Y Vardi.
\newblock On parallel scalable uniform sat witness generation.
\newblock In \emph{International Conference on Tools and Algorithms for the
  Construction and Analysis of Systems}, pages 304--319. Springer,
  2015{\natexlab{a}}.

\bibitem[Chakraborty et~al.(2015{\natexlab{b}})Chakraborty, Fried, Meel, and
  Vardi]{chakrabortyWeightedUnweightedModel2015}
Supratik Chakraborty, Dror Fried, Kuldeep~S Meel, and Moshe~Y Vardi.
\newblock From weighted to unweighted model counting.
\newblock In \emph{Twenty-Fourth International Joint Conference on Artificial
  Intelligence}, 2015{\natexlab{b}}.

\bibitem[Chavira and Darwiche(2008)]{chavira2008probabilistic}
Mark Chavira and Adnan Darwiche.
\newblock On probabilistic inference by weighted model counting.
\newblock \emph{Artificial Intelligence}, 172\penalty0 (6-7):\penalty0
  772--799, 2008.

\bibitem[Cooper et~al.(2007)Cooper, Dyer, and
  Greenhill]{cooperSamplingRegularGraphs2007}
Colin Cooper, Martin Dyer, and Catherine Greenhill.
\newblock Sampling regular graphs and a peer-to-peer network.
\newblock \emph{Combinatorics, Probability and Computing}, 16\penalty0
  (4):\penalty0 557--593, 2007.

\bibitem[de~Salvo~Braz et~al.(2005)de~Salvo~Braz, Amir, and
  Roth]{DBLP:conf/ijcai/BrazAR05}
Rodrigo de~Salvo~Braz, Eyal Amir, and Dan Roth.
\newblock Lifted first-order probabilistic inference.
\newblock In \emph{{IJCAI}}, pages 1319--1325. Professional Book Center, 2005.

\bibitem[Domshlak and Hoffmann(2007)]{domshlak2007probabilistic}
Carmel Domshlak and J{\"o}rg Hoffmann.
\newblock Probabilistic planning via heuristic forward search and weighted
  model counting.
\newblock \emph{Journal of Artificial Intelligence Research}, 30:\penalty0
  565--620, 2007.

\bibitem[Dyer et~al.(2002)Dyer, Frieze, and Jerrum]{dyer2002counting}
Martin Dyer, Alan Frieze, and Mark Jerrum.
\newblock On counting independent sets in sparse graphs.
\newblock \emph{SIAM Journal on Computing}, 31\penalty0 (5):\penalty0
  1527--1541, 2002.

\bibitem[Erd{\H{o}}s et~al.(1960)Erd{\H{o}}s, R{\'e}nyi,
  et~al.]{erdHos1960evolution}
Paul Erd{\H{o}}s, Alfr{\'e}d R{\'e}nyi, et~al.
\newblock On the evolution of random graphs.
\newblock \emph{Publ. Math. Inst. Hung. Acad. Sci}, 5\penalty0 (1):\penalty0
  17--60, 1960.

\bibitem[Feng et~al.(2021)Feng, He, and Yin]{feng2021sampling}
Weiming Feng, Kun He, and Yitong Yin.
\newblock Sampling constraint satisfaction solutions in the local lemma regime.
\newblock In \emph{Proceedings of the 53rd Annual ACM SIGACT Symposium on
  Theory of Computing}, pages 1565--1578, 2021.

\bibitem[Gao and Wormald(2015)]{gaoUniformGenerationRandom2015}
Pu~Gao and Nicholas Wormald.
\newblock Uniform generation of random regular graphs.
\newblock In \emph{2015 IEEE 56th Annual Symposium on Foundations of Computer
  Science}, pages 1218--1230. IEEE, 2015.

\bibitem[Gogate and Domingos(2011)]{DBLP:conf/uai/GogateD11a}
Vibhav Gogate and Pedro~M. Domingos.
\newblock Probabilistic theorem proving.
\newblock In \emph{{UAI}}, pages 256--265. {AUAI} Press, 2011.

\bibitem[Gomes et~al.(2021)Gomes, Sabharwal, and
  Selman]{gomesModelCounting2009}
Carla~P Gomes, Ashish Sabharwal, and Bart Selman.
\newblock Model counting.
\newblock In \emph{Handbook of satisfiability}, pages 993--1014. IOS press,
  2021.

\bibitem[Gr{\"a}del et~al.(1997)Gr{\"a}del, Kolaitis, and
  Vardi]{gradel1997Bull.Symb.Log.}
Erich Gr{\"a}del, Phokion~G Kolaitis, and Moshe~Y Vardi.
\newblock On the decision problem for two-variable first-order logic.
\newblock \emph{Bulletin of symbolic logic}, 3\penalty0 (1):\penalty0 53--69,
  1997.

\bibitem[Gradel et~al.(1997)Gradel, Otto, and
  Rosen]{gradelTwovariableLogicCounting1997}
Erich Gradel, Martin Otto, and Eric Rosen.
\newblock Two-variable logic with counting is decidable.
\newblock In \emph{Proceedings of Twelfth Annual IEEE Symposium on Logic in
  Computer Science}, pages 306--317. IEEE, 1997.

\bibitem[Gribkoff et~al.()Gribkoff, Van~den Broeck, and
  Suciu]{gribkoffUnderstandingComplexityLifted2014}
Eric Gribkoff, Guy Van~den Broeck, and Dan Suciu.
\newblock Understanding the complexity of lifted inference and asymmetric
  weighted model counting.
\newblock In \emph{Workshops at the {{Twenty-Eighth AAAI Conference}} on
  {{Artificial Intelligence}}}.

\bibitem[Guo et~al.(2019)Guo, Jerrum, and Liu]{guo2019uniform}
Heng Guo, Mark Jerrum, and Jingcheng Liu.
\newblock Uniform sampling through the lov{\'a}sz local lemma.
\newblock \emph{Journal of the ACM (JACM)}, 66\penalty0 (3):\penalty0 1--31,
  2019.

\bibitem[He et~al.(2021)He, Sun, and Wu]{he2021perfect}
Kun He, Xiaoming Sun, and Kewen Wu.
\newblock Perfect sampling for (atomic) lov$\backslash$'asz local lemma.
\newblock \emph{arXiv preprint arXiv:2107.03932}, 2021.

\bibitem[Jaeger(2015)]{jaegerLowerComplexityBounds2015}
Manfred Jaeger.
\newblock Lower complexity bounds for lifted inference.
\newblock \emph{Theory and Practice of Logic Programming}, 15\penalty0
  (2):\penalty0 246--263, 2015.

\bibitem[Jerrum et~al.(1986)Jerrum, Valiant, and
  Vazirani]{jerrumRandomGenerationCombinatorial1986}
Mark Jerrum, Leslie~G. Valiant, and Vijay~V. Vazirani.
\newblock Random generation of combinatorial structures from a uniform
  distribution.
\newblock \emph{Theor. Comput. Sci.}, 43:\penalty0 169--188, 1986.

\bibitem[Kazemi et~al.(2014)Kazemi, Buchman, Kersting, Natarajan, and
  Poole]{DBLP:conf/kr/KazemiBKNP14}
Seyed~Mehran Kazemi, David Buchman, Kristian Kersting, Sriraam Natarajan, and
  David Poole.
\newblock Relational logistic regression.
\newblock In Chitta Baral, Giuseppe~De Giacomo, and Thomas Eiter, editors,
  \emph{Principles of Knowledge Representation and Reasoning: Proceedings of
  the Fourteenth International Conference, {KR} 2014, Vienna, Austria, July
  20-24, 2014}. {AAAI} Press, 2014.

\bibitem[Kazemi et~al.(2016)Kazemi, Kimmig, Van~den Broeck, and
  Poole]{kazemi2016new}
Seyed~Mehran Kazemi, Angelika Kimmig, Guy Van~den Broeck, and David Poole.
\newblock New liftable classes for first-order probabilistic inference.
\newblock \emph{Advances in Neural Information Processing Systems}, 29, 2016.

\bibitem[Kazemi et~al.(2017)Kazemi, Kimmig, Broeck, and
  Poole]{kazemiDomainRecursionLifted2017}
Seyed~Mehran Kazemi, Angelika Kimmig, Guy Van~den Broeck, and David Poole.
\newblock Domain recursion for lifted inference with existential quantifiers.
\newblock \emph{arXiv preprint arXiv:1707.07763}, 2017.

\bibitem[Kersting et~al.()Kersting, Ahmadi, and
  Natarajan]{kerstingCountingBeliefPropagation2009}
Kristian Kersting, Babak Ahmadi, and Sriraam Natarajan.
\newblock Counting belief propagation.
\newblock In \emph{Proceedings of the {{Twenty-Fifth Conference}} on
  {{Uncertainty}} in {{Artificial Intelligence}}}, pages 277--284.

\bibitem[Kuusisto and Lutz(2018)]{kuusistoWeightedModelCounting2018}
Antti Kuusisto and Carsten Lutz.
\newblock Weighted model counting beyond two-variable logic.
\newblock In \emph{Proceedings of the 33rd Annual ACM/IEEE Symposium on Logic
  in Computer Science}, pages 619--628, 2018.

\bibitem[Kuzelka(2021)]{kuzelkaWeightedFirstorderModel2021}
Ondrej Kuzelka.
\newblock Weighted first-order model counting in the two-variable fragment with
  counting quantifiers.
\newblock \emph{Journal of Artificial Intelligence Research}, 70:\penalty0
  1281--1307, 2021.

\bibitem[Massey~Jr(1951)]{massey1951kolmogorov}
Frank~J Massey~Jr.
\newblock The kolmogorov-smirnov test for goodness of fit.
\newblock \emph{Journal of the American statistical Association}, 46\penalty0
  (253):\penalty0 68--78, 1951.

\bibitem[Milch et~al.(2008)Milch, Zettlemoyer, Kersting, Haimes, and
  Kaelbling]{DBLP:conf/aaai/MilchZKHK08}
Brian Milch, Luke~S. Zettlemoyer, Kristian Kersting, Michael Haimes, and
  Leslie~Pack Kaelbling.
\newblock Lifted probabilistic inference with counting formulas.
\newblock In Dieter Fox and Carla~P. Gomes, editors, \emph{Proceedings of the
  Twenty-Third {AAAI} Conference on Artificial Intelligence, {AAAI} 2008,
  Chicago, Illinois, USA, July 13-17, 2008}, pages 1062--1068. {AAAI} Press,
  2008.

\bibitem[Mladenov et~al.(2012)Mladenov, Ahmadi, and
  Kersting]{DBLP:journals/jmlr/MladenovAK12}
Martin Mladenov, Babak Ahmadi, and Kristian Kersting.
\newblock Lifted linear programming.
\newblock In Neil~D. Lawrence and Mark~A. Girolami, editors, \emph{Proceedings
  of the Fifteenth International Conference on Artificial Intelligence and
  Statistics, {AISTATS} 2012, La Palma, Canary Islands, Spain, April 21-23,
  2012}, volume~22 of \emph{{JMLR} Proceedings}, pages 788--797. JMLR.org,
  2012.

\bibitem[Mladenov et~al.(2017)Mladenov, Kleinhans, and
  Kersting]{DBLP:conf/aaai/MladenovKK17}
Martin Mladenov, Leonard Kleinhans, and Kristian Kersting.
\newblock Lifted inference for convex quadratic programs.
\newblock In Satinder Singh and Shaul Markovitch, editors, \emph{Proceedings of
  the Thirty-First {AAAI} Conference on Artificial Intelligence, February 4-9,
  2017, San Francisco, California, {USA}}, pages 2350--2356. {AAAI} Press,
  2017.
\newblock \doi{10.1609/AAAI.V31I1.10841}.

\bibitem[Naaman(2021)]{naaman2021tight}
Michael Naaman.
\newblock On the tight constant in the multivariate
  dvoretzky--kiefer--wolfowitz inequality.
\newblock \emph{Statistics \& Probability Letters}, 173:\penalty0 109088, 2021.

\bibitem[Poole(2003)]{poole2003first}
David Poole.
\newblock First-order probabilistic inference.
\newblock In \emph{IJCAI}, volume~3, pages 985--991, 2003.

\bibitem[Raedt et~al.(2007)Raedt, Kimmig, and
  Toivonen]{deraedtProbLogProbabilisticProlog2007}
Luc~De Raedt, Angelika Kimmig, and Hannu Toivonen.
\newblock Problog: {A} probabilistic prolog and its application in link
  discovery.
\newblock In Manuela~M. Veloso, editor, \emph{{IJCAI} 2007, Proceedings of the
  20th International Joint Conference on Artificial Intelligence, Hyderabad,
  India, January 6-12, 2007}, pages 2462--2467, 2007.

\bibitem[Richardson and Domingos(2006)]{richardsonMarkovLogicNetworks2006}
Matthew Richardson and Pedro Domingos.
\newblock Markov logic networks.
\newblock \emph{Machine learning}, 62\penalty0 (1):\penalty0 107--136, 2006.

\bibitem[Sang et~al.(2005)Sang, Beame, and Kautz]{sang2005solving}
Tian Sang, Paul Beame, and Henry Kautz.
\newblock Solving bayesian networks by weighted model counting.
\newblock In \emph{Proceedings of the Twentieth National Conference on
  Artificial Intelligence (AAAI-05)}, volume~1, pages 475--482. AAAI Press,
  2005.

\bibitem[Scott(1962)]{scott1962decision}
Dana Scott.
\newblock A decision method for validity of sentences in two variables.
\newblock \emph{Journal of Symbolic Logic}, 27\penalty0 (377):\penalty0 74,
  1962.

\bibitem[Singla and Domingos()]{singla2008AAAI}
Parag Singla and Pedro~M Domingos.
\newblock Lifted {{First-Order Belief Propagation}}.
\newblock In \emph{{{AAAI}}}, volume~8, pages 1094--1099.

\bibitem[Soos et~al.(2020)Soos, Gocht, and Meel]{soosTintedDetachedLazy2020}
Mate Soos, Stephan Gocht, and Kuldeep~S Meel.
\newblock Tinted, detached, and lazy cnf-xor solving and its applications to
  counting and sampling.
\newblock In \emph{International Conference on Computer Aided Verification},
  pages 463--484. Springer, 2020.

\bibitem[Taghipour et~al.(2013)Taghipour, Fierens, Davis, and
  Blockeel]{DBLP:journals/jair/TaghipourFDB13}
Nima Taghipour, Daan Fierens, Jesse Davis, and Hendrik Blockeel.
\newblock Lifted variable elimination: Decoupling the operators from the
  constraint language.
\newblock \emph{J. Artif. Intell. Res.}, 47:\penalty0 393--439, 2013.

\bibitem[T{\'o}th and Ku{\v{z}}elka(2023)]{toth2022ArXivPrepr.ArXiv221101164}
Jan T{\'o}th and Ond{\v{r}}ej Ku{\v{z}}elka.
\newblock Lifted inference with linear order axiom.
\newblock In \emph{Proceedings of the AAAI Conference on Artificial
  Intelligence}, 2023.
\newblock to appear.

\bibitem[Valiant(1979{\natexlab{a}})]{valiant1979SIAMJ.Comput.}
Leslie~G. Valiant.
\newblock The complexity of enumeration and reliability problems.
\newblock \emph{{SIAM} J. Comput.}, 8\penalty0 (3):\penalty0 410--421,
  1979{\natexlab{a}}.

\bibitem[Valiant(1979{\natexlab{b}})]{valiant1979Theor.Comput.Sci.}
Leslie~G. Valiant.
\newblock The complexity of computing the permanent.
\newblock \emph{Theor. Comput. Sci.}, 8:\penalty0 189--201, 1979{\natexlab{b}}.

\bibitem[van Bremen and Kuzelka(2020)]{DBLP:conf/ijcai/BremenK20}
Timothy van Bremen and Ondrej Kuzelka.
\newblock Approximate weighted first-order model counting: Exploiting fast
  approximate model counters and symmetry.
\newblock In Christian Bessiere, editor, \emph{Proceedings of the Twenty-Ninth
  International Joint Conference on Artificial Intelligence, {IJCAI} 2020},
  pages 4252--4258. ijcai.org, 2020.
\newblock \doi{10.24963/IJCAI.2020/587}.

\bibitem[van Bremen and Kuzelka(2021)]{vanbremenLiftedInferenceTree2021}
Timothy van Bremen and Ondrej Kuzelka.
\newblock Lifted inference with tree axioms.
\newblock In Meghyn Bienvenu, Gerhard Lakemeyer, and Esra Erdem, editors,
  \emph{Proceedings of the 18th International Conference on Principles of
  Knowledge Representation and Reasoning, {KR} 2021, Online event, November
  3-12, 2021}, pages 599--608, 2021.

\bibitem[van Bremen et~al.(2021)van Bremen, Derkinderen, Sharma, Roy, and
  Meel]{van2021symmetric}
Timothy van Bremen, Vincent Derkinderen, Shubham Sharma, Subhajit Roy, and
  Kuldeep~S. Meel.
\newblock Symmetric component caching for model counting on combinatorial
  instances.
\newblock In \emph{Thirty-Fifth {AAAI} Conference on Artificial Intelligence,
  {AAAI} 2021}, pages 3922--3930. {AAAI} Press, 2021.

\bibitem[Van~den Broeck()]{vandenbroeckCompletenessFirstorderKnowledge2011}
Guy Van~den Broeck.
\newblock On the completeness of first-order knowledge compilation for lifted
  probabilistic inference.
\newblock pages 1--9.

\bibitem[Van~den Broeck(2011)]{broeckCompletenessFirstOrderKnowledge2011}
Guy Van~den Broeck.
\newblock On the completeness of first-order knowledge compilation for lifted
  probabilistic inference.
\newblock \emph{Advances in Neural Information Processing Systems}, 24, 2011.

\bibitem[Van~den Broeck and
  Darwiche(2013)]{vandenbroeckComplexityApproximationBinary2013}
Guy Van~den Broeck and Adnan Darwiche.
\newblock On the complexity and approximation of binary evidence in lifted
  inference.
\newblock \emph{Advances in Neural Information Processing Systems}, 26, 2013.

\bibitem[Van~den Broeck and
  Davis(2012)]{vandenbroeckConditioningFirstorderKnowledge2012}
Guy Van~den Broeck and Jesse Davis.
\newblock Conditioning in first-order knowledge compilation and lifted
  probabilistic inference.
\newblock In \emph{Twenty-Sixth AAAI Conference on Artificial Intelligence},
  2012.

\bibitem[Van~den Broeck et~al.()Van~den Broeck, Kersting, Natarajan, and
  Poole]{10.7551/mitpress/10548.001.0001}
Guy Van~den Broeck, Kristian Kersting, Sriraam Natarajan, and David Poole.
\newblock \emph{An introduction to lifted probabilistic inference}.
\newblock The {MIT} Press.
\newblock ISBN 978-0-262-36559-8.
\newblock \doi{10.7551/mitpress/10548.001.0001}.

\bibitem[Van~den Broeck et~al.(2011)Van~den Broeck, Taghipour, Meert, Davis,
  and De~Raedt]{vandenbroeckLiftedProbabilisticInference2011}
Guy Van~den Broeck, Nima Taghipour, Wannes Meert, Jesse Davis, and Luc
  De~Raedt.
\newblock Lifted probabilistic inference by first-order knowledge compilation.
\newblock In \emph{Proceedings of the Twenty-Second international joint
  conference on Artificial Intelligence}, pages 2178--2185, 2011.

\bibitem[Van~den Broeck et~al.(2014)Van~den Broeck, Meert, and
  Darwiche]{vandenbroeck2014Proc.FourteenthInt.Conf.Princ.Knowl.Represent.Reason.}
Guy Van~den Broeck, Wannes Meert, and Adnan Darwiche.
\newblock Skolemization for weighted first-order model counting.
\newblock In \emph{Fourteenth International Conference on the Principles of
  Knowledge Representation and Reasoning}, 2014.

\bibitem[{Van den Broeck} et~al.(2021){Van den Broeck}, Kersting, Natarajan,
  and Poole]{liftedinferencebook}
Guy {Van den Broeck}, Kristian Kersting, Sriraam Natarajan, and David Poole,
  editors.
\newblock \emph{An Introduction to Lifted Probabilistic Inference}.
\newblock Neural Information Processing series. The MIT Press, 2021.

\bibitem[Van~Haaren et~al.(2016)Van~Haaren, Van~den Broeck, Meert, and
  Davis]{van2016lifted}
Jan Van~Haaren, Guy Van~den Broeck, Wannes Meert, and Jesse Davis.
\newblock Lifted generative learning of markov logic networks.
\newblock \emph{Machine Learning}, 103:\penalty0 27--55, 2016.

\bibitem[Wang et~al.(2022)Wang, van Bremen, Wang, and
  Kuzelka]{wangDomainLiftedSamplingUniversal2022}
Yuanhong Wang, Timothy van Bremen, Yuyi Wang, and Ondrej Kuzelka.
\newblock Domain-lifted sampling for universal two-variable logic and
  extensions.
\newblock In \emph{Thirty-Sixth {AAAI} Conference on Artificial Intelligence,
  {AAAI} 2022,}, pages 10070--10079. {AAAI} Press, 2022.

\bibitem[Wang et~al.(2023)Wang, Pu, Wang, and
  Kuzelka]{DBLP:conf/lics/WangP0K23}
Yuanhong Wang, Juhua Pu, Yuyi Wang, and Ondrej Kuzelka.
\newblock On exact sampling in the two-variable fragment of first-order logic.
\newblock In \emph{{LICS}}, pages 1--13, 2023.

\end{thebibliography}





\end{document}